\documentclass{article}
\PassOptionsToPackage{numbers,sort&compress}{natbib}

\usepackage[preprint]{neurips_2026}

\usepackage{microtype} 
\usepackage[utf8]{inputenc}				
\usepackage[T1]{fontenc}				
\usepackage{graphicx} 
\usepackage{hyperref}
\usepackage{amsmath, amssymb, amsthm, mathtools, bm}
\usepackage{thmtools}
\usepackage{thm-restate}
\usepackage[dvipsnames]{xcolor}
\usepackage{caption}
\usepackage{url}
\usepackage{booktabs} 
\usepackage{float}
\usepackage[labelformat=simple]{subcaption}
\usepackage{tikz}
\usepackage{nicefrac}					
\usepackage{multirow}

\graphicspath{{./figures/}}
\DeclareGraphicsExtensions{.pdf}

\hypersetup{
	hidelinks,
    colorlinks=true,
    linkcolor=Blue,
    filecolor=Blue,
    citecolor=Blue,
    urlcolor=Blue,
    pdftitle={LoRA vs. Full Fine-Tuning: A Theoretical Perspective},
}

\theoremstyle{plain}

\newtheorem{proposition}{Proposition}
\newtheorem{lemma}{Lemma}
\newtheorem{corollary}{Corollary}
\theoremstyle{definition}

\theoremstyle{remark}

\newcommand{\norm}[1]{\left\lVert#1\right\rVert_2}
\newcommand{\fnorm}[1]{\left\lVert#1\right\rVert_{\mathrm{F}}}

\newcommand\ddd[0]{\mathcal{D}}

\newcommand\rb[1]{\left(#1\right)}
\renewcommand\sb[1]{\left[#1\right]}

\newcommand\qq[0]{\mathbf{q}}
\newcommand\vv[0]{\mathbf{v}}
\newcommand\xx[0]{\mathbf{x}}
\newcommand\yy[0]{\mathbf{y}}
\newcommand\zz[0]{\mathbf{z}}

\renewcommand\AA[0]{\mathbf{A}}
\newcommand\BB[0]{\mathbf{B}}
\newcommand\CC[0]{\mathbf{C}}
\newcommand\DD[0]{\mathbf{D}}
\newcommand\XX[0]{\mathbf{X}}
\newcommand\YY[0]{\mathbf{Y}}
\newcommand\II[0]{\mathbf{I}}
\newcommand\UU[0]{\mathbf{U}}
\newcommand\VV[0]{\mathbf{V}}
\newcommand\WW[0]{\mathbf{W}}
\newcommand\ZZ[0]{\mathbf{Z}}
\newcommand\MM[0]{\mathbf{M}}

\newcommand\QQ[0]{\mathbf{Q}}
\newcommand\PP[0]{\mathbf{P}}
\newcommand\GG[0]{\mathbf{G}}

\newcommand\RR[0]{\mathbf{R}}

\newcommand{\vecnoise}{\boldsymbol{\varepsilon}} 
\newcommand{\matnoise}{\bm{\xi}} 
\newcommand{\risk}{\mathcal{R}}
\newcommand{\empsol}{\widehat{\AA}}

\newcommand{\lorasol}{\widehat{\AA}_{\mathrm{LoRA}}}
\newcommand{\gdsol}{\widehat{\AA}_{\mathrm{FFT}}}
\newcommand{\initmodel}{\AA_0}
\newcommand{\empcov}{\boldsymbol{\widehat{\Sigma}}}
\newcommand{\trunk}[1]{\operatorname{trunc}_r \! \left\{ #1 \right\} }
\newcommand{\vectorization}{\operatorname{vec}}
\newcommand{\Cov}{\operatorname{Cov}}

\DeclareMathOperator*{\argmin}{arg\,min}

\newcommand{\zeroVec}{\boldsymbol{0}}

\newcommand{\bDelta}[0]{\boldsymbol{\Delta}}
\newcommand{\rvx}{{\bf x}}

\newcommand{\circled}[1]{\tikz[baseline=(C.base)]{\node[draw,circle,inner sep=1.5pt](C){#1};}}
\newcommand{\R}{\mathbb{R}}
\newcommand{\prob}{\mathbb{P}}
\newcommand{\E}{ \mathbb{E} }

\newcommand{\CovMat}{\boldsymbol{\Sigma}}
\newcommand{\rank}{\operatorname{rank}}
\newcommand{\trace}{\operatorname{Tr}}
\newcommand{\Var}[0]{\operatorname{\mathbb{V}ar}}
\newcommand{\kernel}{\operatorname{ker}}
\newcommand{\image}{\operatorname{im}}

\title{LoRA vs. Full Fine-Tuning: A Theoretical Perspective}

\author{%
  Ali Zindari\thanks{Equal contribution} $~^{1,2}$ ~ Rotem Mulayoff$^{*\,2}$ ~ Sebastian U.\ Stich$^{2}$\\%
  $^{1}$Universit{\"a}t des Saarlandes ~ $^{2}$CISPA Helmholtz Center for Information Security\\%
  Saarbr{\"u}cken, Germany\\%
  \texttt{\{ali.zindari,rotem.mulayoff,stich\}@cispa.de}%
}

\begin{document}

\maketitle

\begin{abstract}
    Fine-tuning adapts a pre-trained model to downstream tasks using a small amount of labeled data. Low-Rank Adaptation (LoRA) is an efficient fine-tuning method that reduces memory and computation costs while often achieving performance close to full fine-tuning. Despite its widespread use, the theoretical behavior of LoRA is not yet well understood. In this paper, we study LoRA in a simple linear regression setting and compare its excess risk with that of full fine-tuning. Our analysis identifies regimes in which LoRA achieves lower excess risk than full fine-tuning in both overdetermined and underdetermined settings. Specifically, our theory predicts that LoRA can outperform full fine-tuning when the difference between the pretraining and the downstream tasks is effectively low-rank. We further show how the choice of LoRA rank affects generalization performance, explaining why using a very small rank can improve test accuracy in certain settings, even though it limits model expressivity. Finally, we support our theoretical results with experiments on practical tasks, suggesting that the identified tradeoffs and insights extend beyond linear regression.
\end{abstract}

\section{Introduction}\label{sec:intro}

Large Language Models (LLMs), such as GPT-3 \citep{brown2020language}, PaLM \citep{chowdhery2023palm}, and LLaMA \citep{touvron2023llama}, have achieved strong performance across a wide range of tasks. However, training these models from scratch requires optimizing billions of parameters over massive datasets, resulting in substantial computational,  financial, and environmental costs. As model scales continue to grow, full pre-training has become increasingly impractical and often yields diminishing returns relative to its cost.

\emph{Fine-tuning (FT)} offers an efficient alternative by adapting a pre-trained LLM to a downstream task using a smaller, task-specific dataset \citep{ziegler2019fine, devlin2019bert, dodge2020fine}. The most straightforward approach, \emph{full fine-tuning (FFT)}, updates all model parameters but incurs substantial memory and computational costs, as gradients must be computed and stored for every trainable weight.

To mitigate these overheads, \citet{hu2022lora} proposed using \emph{Low-Rank Adaptation (LoRA)}. Instead of updating the entire weight matrix of each trained layer, LoRA constrains updates to lie in a low-dimensional subspace. This is done by expressing the update as a low-rank matrix factorization. This approach significantly reduces the number of trainable parameters while maintaining performance close to full fine-tuning \citep{biderman2024lora,shuttleworth2024lora} with an appropriate choice of rank. Due to its good performance and efficiency, LoRA has since become one of the most widely adopted fine-tuning methods for LLMs.

Despite LoRA’s widespread adoption and strong empirical performance, its statistical generalization behavior remains only partially understood. A limited number of works have begun to study LoRA from a theoretical perspective. For example, \citet{zeng2023expressive} analyze the expressive power of LoRA, while focusing on the rank required for fine-tuning a model to match a given target model. In a complementary direction, \citet{zhang2025lora} investigate the behavior of LoRA parameters after a single step of gradient descent (GD), and show that even one update can align LoRA with some singular subspaces of the corresponding full fine-tuning solution. In addition, an extensive empirical study by \citet{shuttleworth2024lora} demonstrates that LoRA with low ranks can significantly underperform FFT on a coding dataset, while it can achieve almost the same accuracy on a math dataset using moderate ranks. Overall, in spite of this recent progress, the performance of LoRA in terms of \emph{generalization} remains largely unclear, particularly compared to other leading fine-tuning methods.

In this work, we address this gap by studying the generalization of LoRA in the simplified setting of linear regression, focusing on its expected excess risk over the population distribution. We compare its generalization performance to that of FFT in both overdetermined and underdetermined regimes. The goal is to understand how the imposed low-rank constraint influences generalization and to identify regimes under which one method generalizes better than the other. We begin in Sec.~\ref{sec:setting} by formulating FFT and LoRA within a linear regression framework and presenting their solutions in closed form. In Sec.~\ref{sec:risk_bounds}, we review existing results on the excess risk of FFT, which correspond to the least-squares and minimum-norm solution in the overdetermined and underdetermined settings, respectively. We then present our main theoretical contributions: novel excess risk bounds for LoRA in both regimes, together with an analysis of their asymptotic behavior in high dimensions. Finally, in Sec.~\ref{sec:experiments}, we validate our theoretical findings through controlled linear regression experiments and extend our study to LLM fine-tuning to examine how these insights carry over beyond the linear~setting.

\section{Problem setting and preliminaries} \label{sec:setting}

In this paper, we study fine-tuning in a linear regression setting and compare FFT and LoRA in terms of excess risk. Let $\ddd = \{(\xx_i, \yy_i)\}_{i=1}^n$ be a fine-tuning dataset of~$n$ i.i.d.\ random samples. For each sample $(\xx_i,\yy_i) \sim \prob_{\xx,\yy}$, the feature vector $\xx_i \in \R^{d_x}$ and its label $\yy_i \in \R^{d_y}$ are related through an underlying linear map $\AA^\star \in \R^{d_y \times d_x}$ as
\begin{equation}\label{eq:data generation model}
    \yy_i = \AA^\star \xx_i + \vecnoise_i,
\end{equation}
where $\vecnoise_i \in \R^{d_y}$ is a Gaussian noise, \emph{i.e.}, $\vecnoise_i \sim \mathcal{N} (\zeroVec, \CovMat_{\vecnoise\vecnoise} ) $ i.i.d. Let $\XX = [\xx_1, \ldots, \xx_n] \in\R^{d_x\times n}$ and $\YY = [\yy_1, \dots, \yy_n] \in\R^{d_y\times n}$ be the features and labels matrices, respectively, where the samples are placed in columns. Let $\CovMat_{\rvx\rvx}$ and $\empcov_{\xx\xx}$ be the population and empirical covariances of the feature vectors $ \{ \xx_i\} $, defined as
\begin{equation}
    \CovMat_{\rvx\rvx} := \E\left[\rvx \rvx^\top\right] \qquad \text{and} \qquad 
    \empcov_{\xx\xx} := \frac{1}{n} \sum_{i=1}^n \xx_i \xx_i^{\top} = \frac{1}{n} \XX \XX^\top.
\end{equation}
Throughout the paper, we assume that $\CovMat_{\rvx\rvx}$ has full rank. Given an initial (pre-trained) model $\initmodel \in \R^{d_y \times d_x}$, we define the ground-truth adaptation as
\begin{equation}\label{eq:delta_star}
   \bDelta^\star =  \AA^\star-\initmodel,
\end{equation}
where $\smash{\bDelta^\star \in \R^{d_y \times d_x}}$ captures the difference between the given pre-trained model~$\initmodel$ and the underlying ground-truth map~$\AA^\star$. This formulation allows us to consider a wide range of fine-tuning scenarios: when $\bDelta^\star = \zeroVec$, the pre-training and FT tasks are identical, while a nonzero $\bDelta^\star$ quantifies the degree of task shift. In general, we do not impose any constraint on $\bDelta^\star$, and it can be arbitrary. However, in the following sections, we will analyze how specific structural assumptions on $\bDelta^\star$, such as low or full rank, affect the generalization performance of LoRA and FFT.

The goal of fine-tuning is to find a good estimation $ \empsol $ for~$\AA^\star$. To this end, here we consider the minimization of the empirical risk $ \widehat{\risk} $ with the square loss, defined as
\begin{equation}
     \widehat{\risk} (\AA) := \frac{1}{n}\sum_{i=1}^n \left\| \AA \xx_i - \yy_i  \right\|^2.
\end{equation}
Thus, using the Frobenius norm $\smash{ \fnorm{ \, \cdot \,}} $, the estimator is given by
\begin{equation}
    \empsol 
      =  \argmin_{\AA}  \widehat{\risk}(\AA)  
      = \argmin_{\AA} \fnorm{ \AA \XX  - \YY}^2 .
\end{equation}
Below, we describe the two fine-tuning methods considered in this paper: FFT and LoRA.

\subsection{Full fine-tuning}
In FFT, we minimize the empirical risk without any constraints or regularization, while initializing the weights at the pre-trained model~$\initmodel$. In certain settings, such as the overparameterized regime, the empirical risk may admit multiple global minima. In this work, we opt to consider the solution that minimizes the Euclidean distance to the pre-trained model (initialization). Formally,
\begin{equation*}\label{eq:full_FT}
    \min_{\AA} \fnorm{\AA - \initmodel}
    \qquad \text{s.t.} \qquad
    \AA \in  \argmin_{\AA \in \R^{d_y \times d_x} } \fnorm{ \AA \XX - \YY}^2.
    \tag{FFT}
\end{equation*}
We focus on the minimum-norm solution for four reasons:
(1)~it remains closest to the pre-trained model, which is consistent with the intuition behind fine-tuning;
(2)~in practice, FFT often converges to the minimum-norm solution in certain settings, for example, when optimized using GD \citep{bach2024learning,hastie2022surprises,bartlett2020benign} (see App.~\ref{app_FFT_sol});
(3)~it provides a unified criterion for a fair comparison between FFT and LoRA; and
(4)~minimum-norm solutions often exhibit better generalization and are therefore of greater practical interest.
The optimization problem of \ref{eq:full_FT} has been fully studied, and its solution is as follows.
\begin{restatable}[FFT solution]{proposition}{lemmaFFTSol}\label{lemma:FFT_sol}
    The solution of \ref{eq:full_FT} is given by
    \begin{equation}\label{eq:GD_unified_solution}
        \gdsol = \initmodel + \widehat
         {\bDelta}_{\operatorname{FFT}} \qquad
        \text{s.t.} \qquad \widehat{\bDelta}_{\operatorname{FFT}} := \big(\YY - \initmodel \XX\big) \XX^{\dagger}.
    \end{equation}
    where the superscript $^\dagger$ denotes the Moore--Penrose inverse.
\end{restatable}
This expression unifies the FFT solution for both overdetermined and underdetermined regimes. In the overdetermined regime $ (n \geq d_x) $, where the empirical covariance $ \empcov_{\xx\xx} $ is invertible with probability one, we have $ \XX^\dagger =  \XX^{\top} (\XX \XX^\top)^{-1} $. In this case, the empirical risk admits a unique minimizer that coincides with the ordinary least squares (OLS) estimator. In contrast, in the underdetermined regime $ (n < d_x) $, where  $ \XX^\top \XX $ is invertible with probability one, we have $ \XX^\dagger = (\XX^\top \XX)^{-1}\XX^\top $. In this case, there exist infinitely many solutions, and $ \gdsol $ is the closest to~$ \initmodel $.

\subsection{LoRA fine-tuning}
In LoRA, we fix the pre-trained model and optimize only an additive adaptation. Specifically, this adaptation takes the form $ \AA = \initmodel + \bDelta $, where~$ \bDelta \in \R^{d_y \times d_x} $ is constrained to have rank at most $ r $. The objective is therefore to find the optimal rank-$ r $ matrix $ \bDelta $ that minimizes the empirical risk on the fine-tuning dataset. In the case of linear regression, this objective can be formulated as
\begin{equation*}\label{eq:lora_update}
    \min_{\AA} \fnorm{\AA\XX - \YY}^2
    \qquad \text{s.t.} \qquad
    \AA  = \initmodel + \bDelta \quad   \text{with} \quad \rank(\bDelta)\leq r.
    \tag{LoRA} 
\end{equation*}
Even with the rank constraint, the above optimization problem may still have multiple minimizers. As before, we opt to study the min-norm solution, \emph{i.e.}, the one that is closest to the pre-trained model.
\begin{restatable}[LoRA solution, \citet{sondermann1986best}]{proposition}{lemmaLoraSol}\label{lemma:lora_sol}
    Assume $ \smash{r \leq \rank(\widehat{\bDelta}_{\operatorname{FFT}}\empcov_{\xx\xx}^{\frac{1}{2}}) \leq n } $ and let
    \begin{equation}\label{eq:LoRA solution}
        \widehat{\bDelta}_{\operatorname{LoRA}}  =
        \trunk{ \widehat\bDelta_{\operatorname{FFT}}\empcov_{\xx\xx}^{\frac{1}{2}}}
        \Big(\empcov_{\xx\xx}^{\frac{1}{2}}\Big)^{\dagger},
    \end{equation}
    where $\trunk{\cdot}$ denotes the top $r$ singular value truncation. Then $ \smash{\lorasol = \initmodel +  \widehat{\bDelta}_{\operatorname{LoRA}} }$ is a minimizer of \ref{eq:lora_update} that achieves the minimal Euclidean distance to $ \initmodel $ among all minimizers. 
    Moreover, this min-norm solution is unique if and only if the truncation is unique\footnote{A rank-$r$ truncation $ \trunk{\MM} $ is unique if and only if $  \sigma_{r}(\MM)>\sigma_{r+1}(\MM) $.}.
\end{restatable}
Proof of this result can also be found in \citep{friedland2007generalized} and in App.~\ref{app:low_rank_sol}. The expression above unifies LoRA's solution in both overdetermined and underdetermined regimes.  Importantly, this result reveals a relationship between the minimum-norm solutions of LoRA and FFT. The LoRA solution is obtained by taking the FFT update ${\widehat{\bDelta}_{\operatorname{FFT}}}$, transforming it by the square root of the empirical covariance matrix~$ {\empcov_{\xx\xx}^{\frac{1}{2}}} $, applying a rank~$r$ truncation, and mapping it back via an inverse transformation.

In our setting, $ \widehat{\bDelta}_{\operatorname{FFT}}\empcov_{\xx\xx}^{\frac{1}{2}} $ is a random matrix of rank $ n $ that has distinct singular values with probability one. Consequently, we consider this minimum-norm solution to be unique throughout the paper. In practical training, this min-norm solution is approximately recovered with small weight decay or GD with small initialization \citep{kim2025lora}.

\subsection{Expected excess risk}\label{subsection:risk definition}
We compare the solutions found by LoRA and FFT in terms of \textit{expected excess risk}, a metric that measures the generalization performance. Given an estimator $ \empsol $, we evaluate its risk (error) over the population distribution. Mathematically, let $(\xx,\yy) \sim \prob_{\xx,\yy}$ be a new unseen data pair drawn independently of~$ \empsol $. Then the population risk of $ \empsol $ is defined as
\begin{equation}
    \risk \big( \empsol \big)
    := \E\left[ \| \empsol \xx - \yy  \|^2 \; \middle| \; \empsol \right],
\end{equation}
where the expectation is taken over $ \xx $ and $ \yy $. It is often more informative to consider the \emph{excess risk}, which measures the error relative to that of the ground-truth parameter~$ \AA^{\star} $, \emph{i.e.}, $ \risk ( \empsol ) -\risk ( \AA^{\star} ) $. In our setting, the excess risk can be reduced to (see Lemma~\ref{lemma:general_risk} in App.~\ref{app:Useful Lemmas})
\begin{equation}
    \risk \big( \empsol \big) - \risk \big( \AA^{\star} \big)
    =  \fnorm{ \big. \smash{ \big( \empsol  - \AA^{\star} \big) \CovMat^{\frac{1}{2}}_{\xx \xx}}}^2. 
\end{equation}
Note that $ \empsol $ depends on the particular realization of its training data. While the resulting error may be small for some training sets, it can be large for others. To capture the average behavior, we take the expectation with respect to the training set~$\ddd$. Our performance metric is thus the expected excess risk, $\E[\risk(\empsol)] - \risk(\AA^\star)$, where the expectation is taken over the randomness in the training set.

\section{Excess risk analysis}\label{sec:risk_bounds}
In this section, we analyze the excess risk of FFT and LoRA in both overdetermined and underdetermined regimes and compare their performance. Since the behavior of FFT is well understood, we include it only for comparison; our main contribution lies in deriving excess risk bounds for LoRA.

\subsection{Overdetermined (underparameterized) regime}
We start with the overdetermined regime, where the number of training points $n$ exceeds the feature vector dimension $d_x$. In this case, the FFT solution is independent of the pre-trained model~$ \initmodel $, and is determined solely by the training data, taking the form $ \gdsol~=~\YY\XX^{\top} (\XX \XX^{\top})^{-1} $. A well-known result, typically stated in the multiple linear regression setting, provides an exact analytic expression for the excess risk (see standard proof in \cite[Prop.~3.10]{bach2024learning}, and App.~\ref{subsection:app_FFT_over} for our setting).
\begin{restatable}[FFT excess risk]{theorem}{GDExcessOver}\label{theorem:Gd_excessrisk_overdetermined}
    Let~$ \gdsol $ be the FFT solution \eqref{eq:GD_unified_solution}.  If $ n \geq d_x $, then
    \begin{equation} 
        \E \sb{ \risk\big(\gdsol\big)} - \risk \big( \AA^\star \big) = \frac{1}{n} \trace \big(\CovMat_{\vecnoise\vecnoise}\big)
        \E \sb{\trace \big( \CovMat_{\xx\xx} \empcov_{\xx\xx}^{-1} \big) } \geq \bar{\sigma}^2_{\vecnoise} \frac{d_x d_y}{n},
    \end{equation}
    where $ \bar{\sigma}^2_{\vecnoise} := \E[ \| \vecnoise \|^2 ]/d_y $ (average noise variance per coordinate).
\end{restatable}
From the perspective of the classical bias-variance decomposition, the FFT solution has only variance. This is because, in the overdetermined regime, $ \gdsol $ is an unbiased estimator, \emph{i.e.}, $ \E[ \gdsol ] = \AA^{\star} $. Although the resulting expression for the excess risk is generally intractable, in certain cases we can derive a closed form (see proof in {\cite[Sec.~12.2.3]{bach2024learning}} for the standard setting, and App.~\ref{subsection:app_FFT_over} for ours).
\begin{restatable}[Gaussian case]{proposition}{corGaussOver}\label{proposition:FFT excess risk gaussian overdetermined}
    Assume $n>d_x + 1, \ \xx_i \sim \mathcal{N} (\zeroVec, \CovMat_{\xx \xx})$ i.i.d., and $\CovMat_{\vecnoise\vecnoise} \succeq \zeroVec$, then
    \begin{equation}
        \E\sb{\risk \big( \gdsol \big)}-\risk \big(\AA^\star \big) = \trace \left(\CovMat_{\vecnoise\vecnoise}\right) \frac{d_x}{n-d_x-1}.
    \end{equation}
\end{restatable}
This result shows that, in the overdetermined regime, increasing the number of samples reduces excess risk by lowering variance. As $n$ decreases towards $d_x+1$, the excess risk blows up, reflecting the instability of the min-norm estimator near the interpolation threshold. From Thm.~\ref{theorem:Gd_excessrisk_overdetermined} we see that as long as the average noise variance per coordinate is non-vanishing, the excess risk scales as~$ \Omega \big({d_xd_y}/{n}\big)$. This indicates that the generalization error of FFT can be quite large in high-dimensional settings. Let us compare this behavior with that of LoRA (see proof in App.~\ref{app:Overdetermined Regime}).
\begin{restatable}[LoRA excess risk]{theorem}{LoRAExcessOver}\label{theorem:lora_excessrisk_overdetermined}
    Let~$ \lorasol $ be the LoRA solution \eqref{eq:LoRA solution}. If $ n \geq d_x > r $, then 
    \begin{align}\label{eq:theorem:lora_excessrisk_overdetermined}
       \E\sb{\risk \big( \lorasol \big)} - \risk \big( \AA^\star \big) & \lesssim  \frac{r\max\{d_x,d_y\}}{n} \lambda_{\max} \big( \CovMat_{\vecnoise\vecnoise} \big)
       \E\sb{\lambda_{\max}\big(\CovMat_{\xx \xx} \empcov_{\xx \xx}^{-1} \big) }  \nonumber \\
       & \qquad + \lambda_{\max} \big(\CovMat_{\xx \xx}\big) \bigg( r\sigma^2_{r+1} \big( \bDelta^\star \big) + \sum_{i=r+1}^{\min\{d_x,d_y\}} \sigma^2_{i} \big(\bDelta^\star\big) \bigg) \nonumber \\
       & \qquad + r \sigma^2_{r+1} \big(\bDelta^\star\big) \E\sb{ \lambda_{\max}\big(\empcov_{\xx\xx}\big) \lambda_{\max}\big(\CovMat_{\xx \xx} \empcov_{\xx \xx}^{-1}\big)}.
    \end{align}
\end{restatable}
The rank constraint imposed by LoRA introduces bias into its estimator. The excess risk due to this bias is captured by the second and third terms of the bound. Our result shows that this risk component is controlled by the spectral tail of \( \bDelta^* \), \emph{i.e.}, $ \sum_{i>r} \sigma_{i}^2 (\bDelta^*) $. This error is inevitable, as it is incurred by any rank-$r$ adaptation (see App.~\ref{app:Inevitability error}). We therefore focus on the first term in the bound, which corresponds to the variance. To build intuition for how this term behaves, we derive its asymptotic behavior in high dimensions under the same setting of Prop.~\ref{proposition:FFT excess risk gaussian overdetermined} (see proof in App.~\ref{app:Overdetermined Regime}).
\begin{restatable}{proposition}{LoRAvarGaussOver}\label{proposition:LoRA variance gaussian overdetermined}
    Assume $n>d_x + 1, \ \xx_i \sim \mathcal{N} (\zeroVec, \CovMat_{\xx \xx})$ i.i.d., and $\CovMat_{\vecnoise\vecnoise} \succeq \zeroVec$. The variance (first) term in \eqref{eq:theorem:lora_excessrisk_overdetermined} is asymptotically equivalent in large dimensions to $    r \max\{d_x,d_y\} \big(\sqrt{n}-\sqrt{d_x}\big)^{-2} \lambda_{\max}( \CovMat_{\vecnoise\vecnoise}) $.
\end{restatable}
We observe that LoRA may also exhibit instabilities near the interpolation threshold ($n = d_x$). Yet, far from this threshold ($n \gg d_x$), its variance term can be significantly smaller than that of FFT. In particular, when $\lambda_{\max}(\CovMat_{\vecnoise\vecnoise})$ remains bounded with respect to the dimension $d_y$, while $\bar{\sigma}_{\vecnoise}^2~=~\trace(\CovMat_{\vecnoise\vecnoise}) / d_y$ does not vanish, \emph{e.g.}, isotropic noise, LoRA’s variance scales as $\mathcal{O}\big(r \max\{d_x, d_y\} / n\big)$, compared to the $\Omega\big(d_x d_y / n\big)$ scaling of FFT. Since the rank~$r$ used in practice is typically much smaller than the ambient dimensions~$ d_x $ and~$ d_y $, LoRA achieves a substantial reduction in variance. In contrast, when the noise distribution is concentrated in a low-dimensional subspace around the ground-truth labels, LoRA is not expected to have a meaningful advantage. As the rank~$ r $ increases, the variance grows while the bias decreases, and vice versa, demonstrating the classical bias–variance tradeoff.

Often in practical settings, especially in machine learning, most of the energy of a matrix is concentrated in a few leading singular values, while the remaining singular values are small and contribute little \citep{udell2019big}. Thus, when the singular values of $\bDelta^\star$ decay fast enough, the spectral tail becomes negligible even for small $ r $. In such cases, LoRA's bias is expected to be small, allowing it to reduce overall risk relative to FFT substantially. In contrast, when the spectrum of $\bDelta^\star$ is flat, \emph{i.e.}, the singular values are approximately uniform, the bias can be significant. In this setting, there is no low-rank structure to exploit, and LoRA would perform poorly. Note that when the ground-truth update~$ \bDelta^{\star} $ is truly low-rank, Thm.~\ref{theorem:lora_excessrisk_overdetermined} simplifies to the following result.
\begin{corollary}\label{cor:lora_excessrisk_overdetermined}
    Suppose $\rank(\bDelta^\star) \leq r < \min\{d_x, d_y\} < n$, then
    \begin{align}
        &\E\sb{\risk \big( \lorasol \big)} - \risk \big( \AA^\star \big)  \lesssim  \frac{r\max\{d_x,d_y\}}{n} \lambda_{\max} \big( \CovMat_{\vecnoise\vecnoise} \big)
        \E\sb{ \lambda_{\max}\big(\CovMat_{\xx \xx} \empcov_{\xx \xx}^{-1} \big) } .  
    \end{align}
\end{corollary}
This corollary shows that when $\operatorname{rank}(\bDelta^\star) \leq r$, all bias terms vanish and the excess risk reduces to the variance term alone. As discussed above, under a mild assumption on the noise covariance, this term scales better than the excess risk of FFT. This formally demonstrates the advantage of LoRA over FFT whenever the underlying update has low rank.

\subsection{Underdetermined (overparameterized) regime}
We now turn to the underdetermined regime, where the number of training samples $n$ is smaller than the ambient dimension $d_x$. This setting is better aligned with real-world fine-tuning tasks, where one typically adapts a large pre-trained model using a relatively small, domain-specific dataset. We start by presenting the exact excess risk of the solution obtained by FFT  \eqref{eq:GD_unified_solution}. Here, $\gdsol$ depends on the pre-trained model $ \initmodel $, and is given by
\begin{equation}
    \gdsol = \YY  (\XX^{\top} \XX)^{-1} \XX^{\top} + \initmodel \PP^{\perp}_{\XX},
\end{equation}
where $ \PP^{\perp}_{\XX} $ is the projection matrix onto the orthogonal complement of the column space of $ \XX $. The excess risk in this case is as follows (see proof in {\citep[Sec. 12.2.3]{bach2024learning}} and App.~\ref{app:FFT_under_proof} for our setting).
\begin{restatable}[FFT excess risk]{theorem}{FullFTLowerUnder}\label{theorem:FullFT_lower_underdetermined}
    Let~$ \gdsol $ be the FFT solution \eqref{eq:GD_unified_solution}. If $ n  < d_x $, 
    then
    \begin{equation}
        \E \left[\risk\big(\gdsol\big) \right] - \risk\big(\AA^\star\big)
        = \underbrace{\E\sb{ \fnorm{ \big. \smash{ \bDelta^{\star} \PP^{\perp}_{\XX} \CovMat_{\xx\xx}^{\frac{1}{2}}} }^2}}_{\text{FFT Bias}} 
        +  \underbrace{\frac{1}{n} \trace\big(\CovMat_{\vecnoise\vecnoise}\big) \E \sb{\trace\big(  \CovMat_{\xx\xx} \empcov_{\xx\xx}^\dagger \big)}}_{\text{FFT Variance}}.
    \end{equation}
\end{restatable}
This result generalizes Thm.~\ref{theorem:Gd_excessrisk_overdetermined} to the underdetermined regime. Unlike the overdetermined case, the FFT estimator this time is biased, as reflected in the first term of this expression. To gain insight into these terms, we return to the Gaussian case. Specifically, we can obtain a closed-form expression when the features are distributed isotropically (see proof {\citep[Sec.~12.2.3]{bach2024learning}} and App.~\ref{app:FFT_under_proof}).
\begin{restatable}[Gaussian case]{proposition}{LoRAExcessUnderProp}\label{proposition:FFT excess risk gaussian underdetermined}
    Assume $n < d_x - 1, \ \xx_i \sim \mathcal{N} (\zeroVec, \sigma_{\xx}^2 \II )$ i.i.d., and $\CovMat_{\vecnoise\vecnoise} \succeq \zeroVec$. Then
    \begin{align}\label{eq:var_term}
        \text{FFT Bias}
        = \fnorm{\bDelta^{\star}}^2 \frac{d_x - n}{d_x}, \qquad \text{and} \qquad
        \text{FFT Variance}
        = \trace\big(\CovMat_{\vecnoise\vecnoise}\big) \frac{n}{d_x-n-1}.
    \end{align}
\end{restatable}
We see that the bias term (left) decreases linearly with the number of samples $n$, whereas the variance term (right) increases with $n$. This behavior arises because the estimator is constructed from observed samples: each additional sample effectively adds a noisy dimension to the estimation. As a result, in the underdetermined regime, increasing the number of samples does not necessarily improve performance. When $ \trace(\CovMat_{\vecnoise\vecnoise}) \gg \fnorm{\bDelta^{\star}}^2 $, the problem lies in the strong-noise regime\footnote{Note that $ \trace(\CovMat_{\vecnoise\vecnoise}) = \E[\| \vecnoise \|^2] $ is the expected noise magnitude.}, and the excess risk is dominated by the variance term. In this regime, adding more samples is expected to increase the risk. Conversely, when $  \trace(\CovMat_{\vecnoise\vecnoise}) \ll \fnorm{\bDelta^{\star}}^2  $, the problem is in the weak-noise regime, and if $ n $ is not too close to $ d_x $, the excess risk is governed by the bias term and the total risk decays with $n$. We now compare this behavior with that of LoRA (see proof in App.~\ref{subsection:app_lora_under})
\begin{restatable}[LoRA excess risk]{theorem}{LoRAExcessUnder}\label{theorem:LoRA_excessrisk_underdetermined}
    Let~$ \lorasol $ be the LoRA solution \eqref{eq:LoRA solution}. If $ r < n < d_x $, then 
    \begin{align}\label{eq:LoRA_excessrisk_underdetermined}
        \E\sb{\risk \big( \lorasol \big)} \! - \! \risk \big( \AA^\star \big)
        \lesssim & \E\sb{ \fnorm{ \big. \smash{ \bDelta^{\star} \PP^{\perp}_{\XX} \CovMat_{\xx\xx}^{\frac{1}{2}}} }^2}
        \! + \! \frac{r\max\{n,d_y\}}{n} \lambda_{\max} \big( \CovMat_{\vecnoise\vecnoise} \big)
        \E\sb{\lambda_{\max}\big(\CovMat_{\xx \xx} \empcov_{\xx \xx}^{\dagger} \big) }\nonumber \\
        & \qquad + \lambda_{\max} \big(\CovMat_{\xx \xx}\big) \bigg( r\sigma^2_{r+1} \big( \bDelta^\star \big)   + \sum_{i=r+1}^{\min\{n,d_y\}} \sigma^2_{i} \big(\bDelta^\star\big) \bigg) 
        \nonumber \\ & 
        \qquad + r \sigma^2_{r+1} \big(\bDelta^\star\big) \E\sb{ \lambda_{\max}\big(\empcov_{\xx\xx}\big) \lambda_{\max}\big(\CovMat_{\xx \xx} \empcov_{\xx \xx}^{\dagger}\big)}.
    \end{align}
\end{restatable}
The excess risk of LoRA contains three distinct contributions: (1) the bias induced by the lack of sufficient data, (2) the variance term, and (3) the approximation error from predicting only rank~$r$ matrices, given by the last two terms. Let us compare these terms with their counterparts for FFT. The bias term $ \Big. \E[ \fnorm{ \smash{\bDelta^{\star} \PP^{\perp}_{\XX} \CovMat_{\xx\xx}^{\frac{1}{2}}} }^2 ] $ is identical to that appearing in the excess risk of FFT. Like in the overdetermined regime, the main advantage of LoRA appears in the variance term. Let us derive its asymptotic behavior in high dimensions in the same setting of Prop.~\ref{proposition:FFT excess risk gaussian underdetermined} (see proof in App. \ref{subsection:app_lora_under}).
\begin{restatable}{proposition}{LoRAvarGaussUnder}\label{proposition:LoRA variance gaussian underdetermined}
    Assume $d_x>n + 1, \  \xx_i \sim \mathcal{N} (\zeroVec, \sigma^2_{\xx} \II )$ i.i.d., and $\CovMat_{\vecnoise\vecnoise} \succeq \zeroVec$. The variance (second) term in \eqref{eq:LoRA_excessrisk_underdetermined} is asymptotically equivalent in large dimensions to $  r \max\{n,d_y\} ( \sqrt{d_x} - \sqrt{n})^{-2} \lambda_{\max}(\CovMat_{\vecnoise\vecnoise}) $.
\end{restatable}
The comparison between the bias terms in the underdetermined regime is similar to that in the overdetermined regime. 
Again, away from the interpolation threshold ($ d_x \gg n $), under mild assumptions on the noise covariance\footnote{The same condition from the overdetermined case, see discussion below Prop.~\ref{proposition:LoRA variance gaussian overdetermined}}, the variance term for LoRA scales as $\mathcal{O} \big( {r\max\{n,d_y\}}/{d_x} \big) $, compared to the $ \Theta \big( {n d_y}/{d_x} \big) $ scaling of the corresponding variance term in FFT. Since the rank~$r$ used in practice is typically much smaller than the ambient dimension~$d_y$ and the number of samples~$ n $, LoRA achieves a substantial reduction in variance. Therefore, just like in the overdetermined regime, LoRA can outperform FFT if the rank-constraint component, given by the last two terms in \eqref{eq:LoRA_excessrisk_underdetermined}, is negligible. As in the overdetermined regime, it is small whenever the singular values of $\bDelta^\star$ decay rapidly. In particular, when~$\bDelta^\star$ has rank at most~$r$, this component vanishes entirely.
\begin{corollary}\label{cor:lora_excessrisk_underdetermined}
    Suppose $\rank(\bDelta^\star) \leq r < n < d_x$, then
    \begin{equation}
        \E\sb{\risk \big( \lorasol \big)} \! - \! \risk \big( \AA^\star \big) 
        \lesssim \E\sb{ \fnorm{ \big. \smash{ \bDelta^{\star} \PP^{\perp}_{\XX} \CovMat_{\xx\xx}^{\frac{1}{2}}} }^2}
        \! + \! \frac{r \max\{ n,d_y \}}{n} \lambda_{\max} \big( \CovMat_{\vecnoise\vecnoise} \big)
        \E\sb{\lambda_{\max} \big(\CovMat_{\xx \xx} \empcov_{\xx \xx}^{\dagger} \big) } \! .
    \end{equation}
\end{corollary}
Overall, our results show that when $\bDelta^\star$ is effectively low rank, \emph{i.e.}, rapidly decaying spectrum, LoRA can achieve lower generalization error than FFT. This advantage is pronounced in the strong-noise regime, where the excess risk is dominated by variance, amplifying the benefits of variance reduction.

\begin{figure}[t]
    \centering
    \begin{subfigure}[t]{0.5\linewidth}
        \centering
        \includegraphics[width=\linewidth]{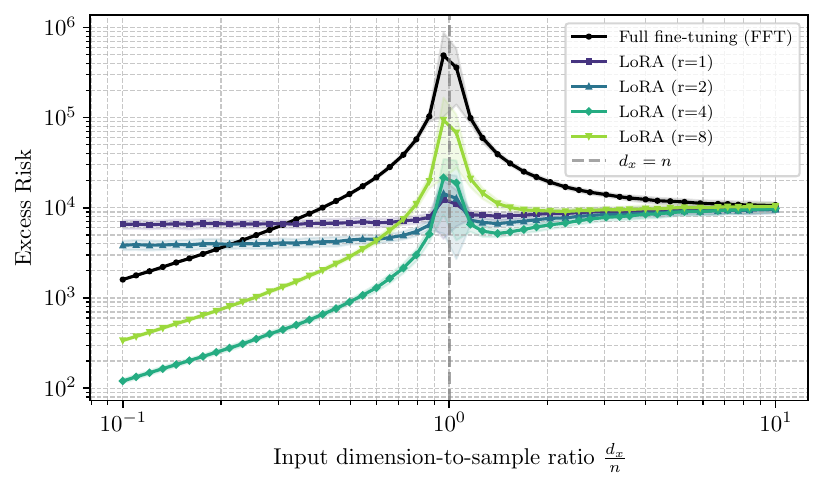}
        \caption{Excess risk vs. the dimension-to-sample ratio}
        \label{fig:samples_sweep}
    \end{subfigure}%
    \begin{subfigure}[t]{0.5\linewidth}
        \centering
        \includegraphics[width=\linewidth]{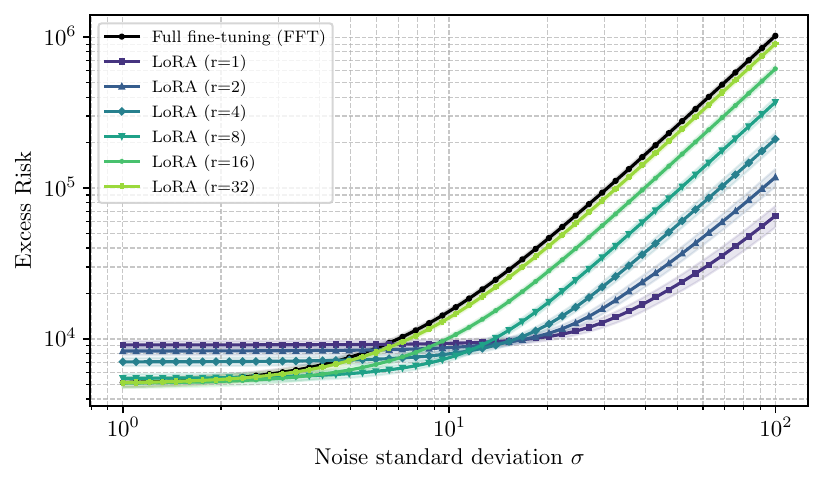}
        \caption{Excess risk vs. noise level}
        \label{fig:noise_sweep}
    \end{subfigure}
    \caption{{\bf Linear regression experiments.}
    Panel~\protect\subref{fig:samples_sweep} presents the excess risk of FFT and LoRA under varying sample size $n$ (decreasing from left to right), fixed dimensions $d_x = d_y = 100$, noise magnitude $\sigma=12$, and true task-shift rank $\mathrm{rank}(\bDelta^\star) = 4$.
    Panel~\protect\subref{fig:noise_sweep} plots the excess risk as a function of noise level $\sigma\in[1,100]$ with $d_x=d_y=100$, $n=50$, and $\rank(\bDelta^\star)=10$.
    Results are averaged over $100$ random seeds; shaded regions indicate $\pm1$ standard deviation.}
    \label{fig:lora_sweeps}
\end{figure}

\section{Experiments}\label{sec:experiments}
In this section, we empirically validate our theory. We first examine a controlled linear regression setting to isolate the effects of rank, sample size, and noise, and then conduct large-scale fine-tuning experiments with LLMs to determine whether similar behaviors arise in real-world fine-tuning.

{\bf Linear regression experiment.}
For this experiment, we use isotropic Gaussian feature vectors, $\xx \sim \mathcal{N}(0,\II_{d_x})$, where the labels are generated as in \eqref{eq:data generation model} with $\vecnoise \sim \mathcal{N}(0,\sigma^2_{\vecnoise} \II_{d_y})$. The underlying linear map is created as $\AA^\star = \initmodel + \bDelta^\star$, where the pre-trained model $\initmodel$ has i.i.d.\ standard Gaussian entries, and $\bDelta^\star = \frac{1}{\sqrt{r^\star}} \UU \VV $ is a low-rank adaptation constructed by sampling $\UU \in \mathbb{R}^{d_y \times r^\star}$ and $\VV \in \mathbb{R}^{r^\star \times d_x}$ with i.i.d.\ normal Gaussian entries; the scaling ensures bounded perturbation magnitude. We use the excess risk $\smash{\risk(\empsol) }$ from Sec.~\ref{subsection:risk definition} to evaluate the performance of the obtained estimators. Figure~\ref{fig:lora_sweeps} plots the excess risk of FFT and LoRA against sample size and the noise level.

The gray line at $d_x=n$ in Fig.~\ref{fig:samples_sweep} marks the interpolation threshold, separating the overdetermined and underdetermined regimes. In the overdetermined regime ($n>d_x$), smaller sample sizes lead to higher excess risk for all methods, consistent with our theoretical predictions. Moreover, when~$r$ matches or exceeds the rank of the true shift $r^{\star}$, LoRA outperforms FFT, with the best performance attained at $r=r^{\star}=4$. Using a larger rank than the ground-truth only slightly degrades performance, as it introduces unnecessary degrees of freedom. For smaller ranks, \emph{e.g.}, $r\in \{1,2\}$, the excess risk remains almost constant, reflecting the dominance of the bias induced by the rank constraint. Since in this case, the spectral tail of~$ \bDelta^\star $ is significant, and as predicted by our theory, LoRA performs poorly.

In the underdetermined regime, most configurations exhibit a decrease in excess risk as the number of training samples gets smaller. As explained above (see discussion below Prop.~\ref{proposition:FFT excess risk gaussian underdetermined}), this behavior arises because the estimators are constructed from the observed samples, while each additional sample effectively introduces a new noisy dimension into the estimation. This affects both LoRA and FFT, where the effect's strength on LoRA depends on the rank. Specifically, just after the interpolation threshold $d_x=n$, the variance term dominates the excess risk, and decreasing $n$ reduces the risk until the bias term becomes dominant and the curves stabilize.

Figure~\ref{fig:noise_sweep} shows the effect of noise magnitude in the underdetermined regime. When the noise level is low, FFT performs best and is comparable to high-rank LoRA. As the noise increases, FFT degrades rapidly and becomes much less robust than low-rank LoRA. This behavior is expected, since updating many unnecessary directions amplifies noise, whereas LoRA reduces this effect.

\begin{figure}[t]
    \centering
    \begin{subfigure}[t]{0.5\linewidth}
        \centering
        \includegraphics[width=\linewidth]{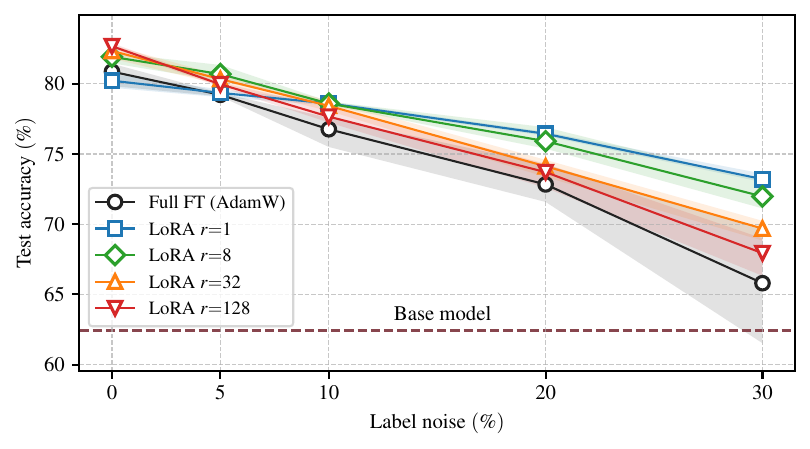}
        \caption{Qwen$2.5$-$0.5$B on BoolQ}
        \label{fig:Boolq_noise_sweep_0.5B_main}
    \end{subfigure}%
    \begin{subfigure}[t]{0.5\linewidth}
        \centering
        \includegraphics[width=\linewidth]{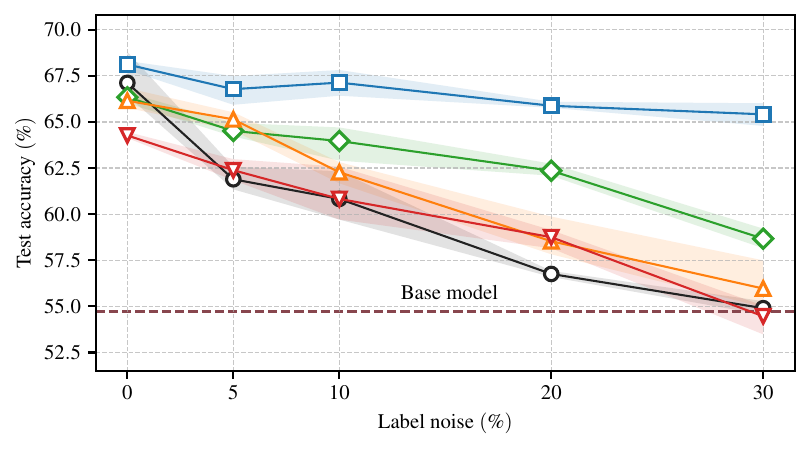}
        \caption{Qwen$2.5$-$0.5$B on CommonsenseQA}
        \label{fig:commonsense_noise_sweep_0.5B_main}
    \end{subfigure}
    \caption{{\bf Effect of label noise in LLMs fine-tuning.}
    We fine-tuned Qwen$2.5$ models using LoRA with various ranks across different levels of label noise. For each configuration, we trained $3$ times using random seeds. Panels~\protect\subref{fig:Boolq_noise_sweep_0.5B_main} and~\protect\subref{fig:commonsense_noise_sweep_0.5B_main} show the mean and the range of the results for the $0.5$B model fine-tuned on BoolQ and CommonsenseQA, respectively. Here, in the strong-noise regime, LoRA outperforms FFT as predicted by our theory.
    For other model sizes, see App.~\ref{subsection:app_additional_llm_experiments}.}
    \label{fig:LLM_noise_sweeps_0.5B_main}
\end{figure}

\paragraph{LLM fine-tuning setup.}
We next examine whether similar qualitative trends arise in practical LLM fine-tuning. We fine-tune Qwen$2.5$ models~\cite{qwen2025qwen25technicalreport} ($0.5$B, $1.5$B, and $3$B) using FFT and LoRA with various ranks across different levels of label noise/sample sizes. For each configuration, we run $3$ random seeds on BoolQ~\cite{clark2019boolq} and CommonsenseQA~\cite{talmor2019commonsenseqa} datasets. 

Figure~\ref{fig:LLM_noise_sweeps_0.5B_main} depicts the test accuracy of the trained $0.5$B model under increasing percentage of label noise. In the weak noise regime, no consistent ordering across ranks is observed. Here, the bias–variance tradeoff suggests that an intermediate rank is optimal. In contrast, in the strong-noise regime, performance decreases monotonically with increasing rank. This trend is consistent with the variance-reduction mechanism predicted by our theory: when noise dominates, restricting the number of trainable directions can improve generalization (see discussion below Corollary~\ref{cor:lora_excessrisk_overdetermined}). 

Finally, Fig.~\ref{fig:LLM_samples_sweeps_1.5B_main} shows the test accuracy of the trained $1.5$B model on varying training set sizes. To estimate the spectral structure of the ground-truth adaptation, we train a high-performing reference model using additional training samples and model averaging across multiple runs; details are provided in App.~\ref{app:Experiments}. We then compute the per-layer normalized cumulative spectral energy of the corresponding weight difference~$\bDelta^\star$, given by $ s(r) = \sum_{i=1}^r \sigma^2_i(\WW)/\fnorm{\WW}^2 $. The lower panels visualize the distribution of this quantity across layers. For CommonsenseQA, the reference update exhibits a rapidly decaying spectrum, suggesting that an optimal adaptation is well captured by a low-rank update. Correspondingly, LoRA significantly outperforms FFT, consistent with our theoretical prediction that low-rank adaptations favor LoRA. In contrast, for BoolQ, the spectral tail is heavier, suggesting that the underlying adaptation is less well approximated by a low-rank structure. In this case, LoRA provides no significant advantage over FFT, aligning with our theoretical expectations.

\section{Related work}\label{sec:related work}
The effectiveness of LoRA and its comparison with FFT have been the focus of several recent papers. On the empirical side, \citet{shuttleworth2024lora} examined the spectral properties of weight matrices learned by LoRA and FFT. They found that these properties are very different, and LoRA training creates new high-ranking singular vectors called \textit{intruder dimensions}. Experiments by \citet{biderman2024lora} showed that LoRA often underperforms FFT on tasks such as coding, math, and instruction fine-tuning, but it suffers less from catastrophic forgetting and performs better than common regularization methods for reducing forgetting. Similarly, \citet{ivison2023camels} found that QLoRA \citep{dettmers2023qlora} falls short of FFT on long-form generation tasks. Additionally, \citet{zhuo2024astraios} empirically studied fine-tuning across various model sizes and methods and found that FFT typically achieves the best performance, while LoRA offers the best balance between performance and efficiency. 

On the theoretical side, \citet{zeng2023expressive} examined the expressive power of LoRA, focusing on its ability to represent smaller target models. They considered fully connected and transformer architectures and showed that, with sufficiently large rank, LoRA can adapt to these small models. Likewise, \citet{zhang2025lora} demonstrated that LoRA adapters align with specific singular subspaces of the one-step full fine-tuning gradient under GD, leading to the LoRA-One algorithm that achieves linear convergence and better handles ill-conditioning through preconditioners. Moreover, \citet{kim2025lora} showed that LoRA training either converges to a low-rank, small-magnitude global minimum or results in a high-rank, large-magnitude solution, with zero initialization and weight decay biasing toward the desirable low-rank outcomes. Finally, \citet{kratsios2025sharp} derived sharp generalization bounds for asymmetric randomized LoRA, showing a sample complexity of $\tilde{\mathcal{O}}({\sqrt{r/N}})$ for rank~$r$ and $N$ samples, along with matching lower bounds that highlight its efficiency.

Our excess risk analysis for LoRA fine-tuning is closely related to the literature on multivariate linear regression with rank constraints. However, our setting differs from the existing literature in three main ways. First, we study generalization through the population risk, \emph{i.e.}, performance on unseen test data, rather than the empirical risk on the training set used by prior work. This makes our setting substantially more challenging. Moreover, bounds on empirical risk do not directly translate into guarantees on test error and, therefore, are insufficient for assessing generalization. Second, our results do not require any rank assumption on the underlying linear map that generates the data. In contrast, much of the existing work assumes that the true regression operator is low-rank. Third, our analysis applies to both the underdetermined and overdetermined regimes. For example, \citet{giraud2021introduction} studied rank-constrained regression in the underdetermined regime and derived an upper bound on the excess risk over the \emph{training set}. Similarly, \citet{rigollet2023high} analyzed estimators based on singular value thresholding and least squares with rank penalization in the overdetermined low-rank regime, establishing training error bounds of the form ${r \max\{d_x,d_y}\}/{n}$.

\begin{figure}[t]%
    \begin{subfigure}[t]{0.5\linewidth}%
    \!\!\!
        \includegraphics[width=\linewidth]{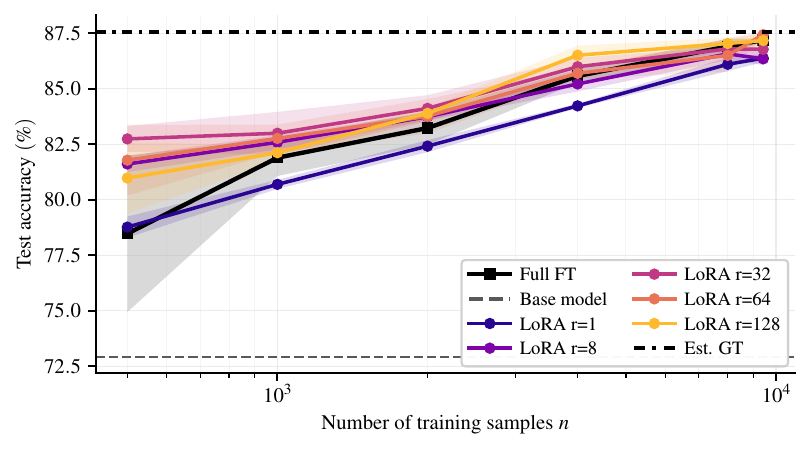}\\%
        \includegraphics[width=\linewidth]{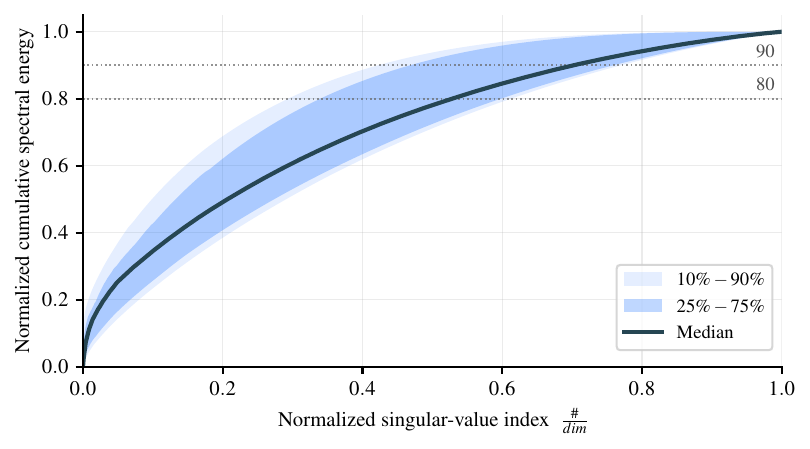}
        \caption{Qwen$2.5$-$1.5$B on BoolQ}
        \label{fig:Boolq_samples_sweep_1.5B_main}
    \end{subfigure}%
    \begin{subfigure}[t]{0.5\linewidth}%
        \includegraphics[width=\linewidth]{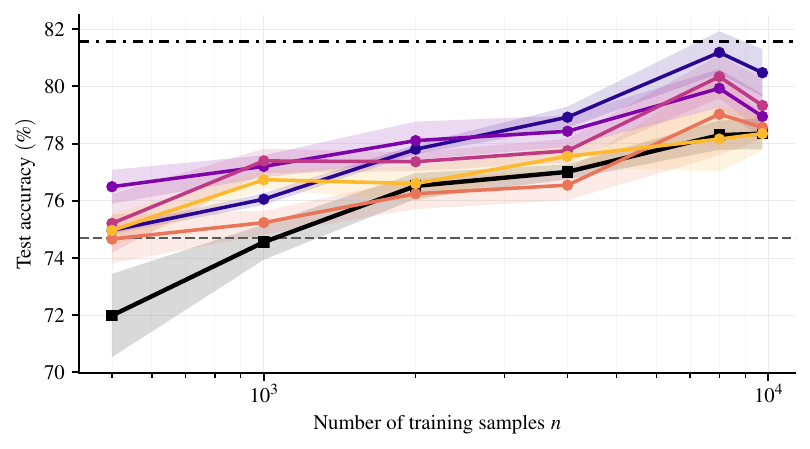}\\%
        \includegraphics[width=\linewidth]{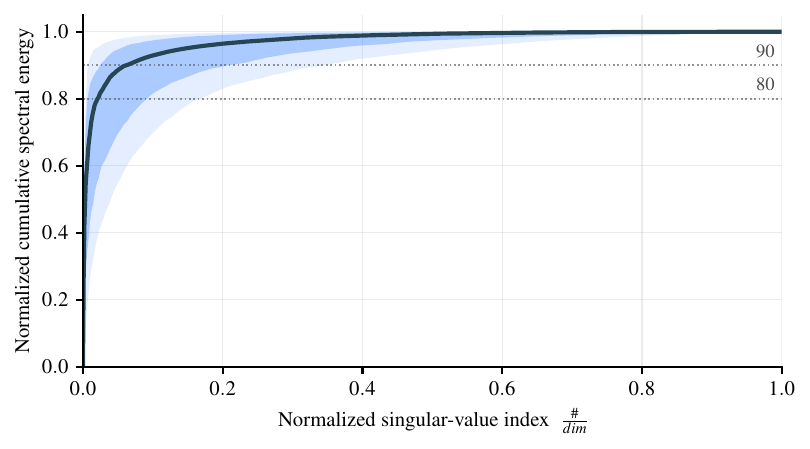}
        \caption{Qwen$2.5$-$1.5$B on CommonsenseQA}
        \label{fig:commonsense_samples_sweep_1.5B_main}
    \end{subfigure}
    \caption{{\bf Effect of sample size in LLMs fine-tuning.}
    We fine-tuned Qwen$2.5$ models using LoRA with various ranks across different sample sizes. For each configuration, we trained $3$ times using random seeds. The top of Panels~\protect\subref{fig:Boolq_samples_sweep_1.5B_main} and~\protect\subref{fig:commonsense_samples_sweep_1.5B_main} shows the mean and the range of the results for the $1.5$B model fine-tuned on BoolQ and CommonsenseQA, respectively. For each task, we estimated $\bDelta^\star$ and computed its singular values per layer.
    The bottom part shows the statistics of the spectral cumulative energy of $\bDelta^\star$ across layers. When the effective rank of the adaptation is low (CommonsenseQA), LoRA outperforms FFT, as predicted by our theory.
    For other model sizes, see App.~\ref{subsection:app_additional_llm_experiments}.}
    \label{fig:LLM_samples_sweeps_1.5B_main}
\end{figure}

\section{Conclusion}\label{sec:conclusion}
In this paper, we provided a rigorous theoretical analysis of the generalization behavior of LoRA compared to FFT by deriving excess risk bounds for linear models in both overdetermined and underdetermined regimes. Our analysis identifies the primary advantage of LoRA as a significant reduction in variance. This effect is particularly prominent when the task shift $\bDelta^\star$ is effectively low rank and in the strong-noise regime. Fine-tuning experiments on Qwen2.5 models across reasoning tasks reveal patterns consistent with these insights.

\textbf{Limitations and future work.}
Our theoretical results apply to the simple setting of linear models with the MSE loss. A natural way to extend our analysis to nonlinear models is to use the Neural Tangent Kernel (NTK) regime, in which nonlinear models behave as linear. However, showing that LLMs also exhibit an NTK regime is beyond the scope of this paper. Additionally, to fully reduce the comparison between LoRA and FFT from nonlinear models to linear regression, as in our setting, one needs the linearization to be the same for both methods. This is nontrivial, and we leave it for future work. Finally, the qualitative empirical observations from our LLM experiments are only suggestive of the patterns identified in our linear theory; further empirical investigation is required.

\section*{Acknowledgment}
This work was supported by the Helmholtz Association's Initiative and Networking Fund on the HAICORE@FZJ partition. We also gratefully acknowledge funding from the European Research Council (ERC) under the Horizon Europe Framework Programme (HORIZON) for proposal number 101170430 CollectiveMinds. Views and opinions expressed are however those of the authors only and do not necessarily reflect those of the European Union or the European Research Council. Neither the European Union nor the granting authority can be held responsible for them. The authors thank Nir Weinberger for the helpful discussion during the preparation of this work. 

\bibliographystyle{plainnat}
\bibliography{references}


\clearpage
\appendix

\begin{center}
{\Huge\bfseries Appendix}
\end{center}
\renewcommand{\thesection}{\Roman{section}}

\section{Summary of Notation}

\begin{table}[h]
\centering
\caption{Summary of notation used in the paper}
\renewcommand{\arraystretch}{1.45}
\setlength{\arrayrulewidth}{.5mm} 
\setlength{\tabcolsep}{8pt} 
\resizebox{\textwidth}{!}{
\begin{tabular}{|c|c|c|}
\hline
\bfseries{Symbol} & \bfseries{Definition} & \bfseries{Mathematical definition / short note} \\ \hline
$\xx_i, \xx$ & Fine-tuning input vector / test data & $\xx_i, \xx\in\R^{d_x}$ \\ \hline
$\yy_i, \yy$ & Fine-tuning output label / test label & $\yy_i, \yy\in\R^{d_y}$, $\yy_i = \AA^\star  \xx_i + \vecnoise_i$ \\ \hline
$\XX, \YY$ & Pre-training data matrices & $\XX=[\xx_1,\dots,\xx_n],\ \YY=[\yy_1,\dots,\yy_n]$ \\ \hline
$\vecnoise,\ \matnoise$ & Noise vector / matrix & $\E[\vecnoise]=0,\ \E[\vecnoise\vecnoise^\top]=\CovMat_{\vecnoise\vecnoise}$ \\ \hline
$\CovMat_{\xx\xx},\ \CovMat_{\yy\xx}$ & Population input covariances / cross-covariances & $\CovMat_{\xx\xx}=\E[\xx \xx^\top],\ \CovMat_{\yy\xx}=\E[\yy \xx^\top]$ \\ \hline
$\empcov_{\xx\xx},\ \empcov_{\yy\xx}$ & Empirical covariances / cross-covariances & $\empcov_{\xx\xx}=\tfrac{1}{n}\XX\XX^\top,\ \empcov_{\yy\xx}=\tfrac{1}{n}\YY\XX^\top$ \\ \hline
$d_x, d_y$ & Input and output dimensions & positive integers (dimensions) \\ \hline
$\mathcal{D}$ & Fine-tuning dataset & collection of $(\xx_i,\yy_i)$ \\ \hline
$\AA^\star $ & True map for FT dataset & $\yy = \AA^\star  \xx + \vecnoise$ (true map) \\ \hline
$\initmodel$ & Pre-trained model (nonzero) & fixed initializer used before FT \\ \hline
$\gdsol$ & Solution found by Full Fine-tuning & $\gdsol=\initmodel + (\YY-\initmodel \XX)\XX^{\dagger}$ \\ \hline
$\lorasol$ & Solution found by LoRA & $\lorasol \! = \!\initmodel \!+\!\trunk{\widehat\bDelta_{\operatorname{FFT}}\empcov_{\xx\xx}^{\frac{1}{2}}\!}\!\big(\empcov_{\xx\xx}^{\frac{1}{2}}\big)^{\dagger}\!\!\!$ \\ \hline
$\bDelta^\star$ & True fine-tuning update & $\bDelta^\star :=  \AA^\star - \initmodel $ \\ \hline
$ \widehat \bDelta_{\operatorname{FFT}},\widehat \bDelta_{\operatorname{LoRA}} \in\R^{d_y\times d_x}$ & FFT / LoRA update matrices &  $\rank(\widehat \bDelta_{\operatorname{LoRA}}) = r$ \\ \hline
$r$ & LoRA (adaptation) rank & chosen rank for low-rank adaptation \\ \hline
$\XX^{\dagger}$ & Moore–Penrose pseudoinverse of $\XX$ & $\XX^{\dagger}=(\XX^\top \XX)^{-1}\XX^\top$ when full column rank \\ \hline
$\mathcal{R}(\AA)$ & Population risk of estimator $\AA$ & $\mathcal{R}(\AA)=\E\big[\|\yy-\AA\xx\|_2^2\big]$ \\ \hline
$\sigma_i(\MM)$ & $i$-th singular value of $\MM$ & ordered $\sigma_1\ge\cdots\ge\sigma_{\min}$ \\ \hline
$\fnorm{\,\cdot\,},\ \norm{\,\cdot\,}$ & Frobenius and operator norms & $\fnorm{\MM}=\sqrt{\trace(\MM^\top\MM)},\,\, \norm{\MM} = \sigma_{\max}(\MM)$ \\ \hline
$\mathrm{Tr}(\,\cdot\,)$ & Trace operator & sum of diagonal entries \\ \hline
$\operatorname{rank}(\,\cdot\,)$ & Matrix rank & dimension of column space \\ \hline
$\lambda_{\min}(\,\cdot\,),\ \lambda_{\max}(\,\cdot\,)$ & Smallest / largest eigenvalue & for symmetric PSD matrices \\ \hline
\end{tabular}
}
\label{tab:notation_clean}
\end{table}

\clearpage
\section{Useful lemmas}\label{app:Useful Lemmas}
\begin{lemma}[Eckart-Young Theorem] \label{lemma:young_rank}
For any matrix $\MM \in \R^{m\times n}$ we have
\begin{equation}
    \trunk{\MM} = \argmin_{\rank(\BB)\leq r} \norm{\BB-\MM} = \argmin_{\rank(\BB)\leq r} \fnorm{\BB-\MM},
\end{equation}
where $\trunk{\cdot}$ is an operator that truncates the singular values of a matrix. Furthermore,
\begin{equation}
    \norm{\trunk{\MM} - \MM} = \sigma_{r+1}(\MM), \qquad \fnorm{\trunk{\MM} - \MM}^2 = \sum_{i=r+1}^{\min\{m,n\}} \sigma_{i}^2(\MM).
\end{equation}
\end{lemma}
\begin{lemma}\label{lemma:perturbation_bound}
    For any matrix $\AA \in \R^{m\times n}$ and a perturbation matrix $\bDelta \in \R^{m\times n}$ we have   
    \begin{equation}
        \norm{\trunk{\AA+\bDelta}-\AA} \leq \sigma_{r+1}(\AA) + 2\norm{\bDelta}.
    \end{equation}
\end{lemma}
\begin{proof}
    By adding and subtracting $\bDelta$ we have
    \begin{equation}
        \norm{\trunk{\AA+\bDelta}-(\AA+\bDelta)+\bDelta} \leq \norm{\trunk{\AA+\bDelta}-(\AA+\bDelta)} + \norm{\bDelta}.
    \end{equation}
    Now we bound the first term above as
    \begin{align}
        \norm{\trunk{\AA+\bDelta}-(\AA+\bDelta)}
        & = \min_{\rank(\BB) \leq r} \norm{\BB-(\AA+\bDelta)}\nonumber\\
        & \leq \norm{\trunk{\AA} -(\AA+\bDelta)}\nonumber\\
        & \leq \norm{\trunk{\AA}-\AA} + \norm{\bDelta}\nonumber\\
        & \leq \sigma_{r+1}(\AA) + \norm{\bDelta},
    \end{align}
    where in the first step we used Eckart-Young theorem (Lemma~\ref{lemma:young_rank}). Putting it together gives the result. 
\end{proof}
\begin{lemma}[Gaussian Poincaré inequality, {\citet[Thm.~3.20]{boucheron2013concentration}}]\label{lemma:poincare}
    Let $X = (X_1,...,X_n)$ be a vector of i.i.d.\ standard Gaussian random variables (i.e., $X$ is a Gaussian vector with zero mean vector and identity covariance matrix), and $f:\mathbb{R}^d\to\mathbb{R}$ be $1$-Lipschitz. Then 
    \begin{equation}
        \Var(f(X)) \leq 1.
    \end{equation}
\end{lemma}
\begin{lemma}\label{lemma:gaussian_matrix_bound}
    Let $ \AA \in \R^{m\times n} $ with i.i.d.\ entries $\mathcal{N}(0,1)$. Then
    \begin{equation}
        \E\sb{\norm{\AA}^2} \le 5\max\{m, n\}.
    \end{equation}
\end{lemma}
\begin{proof}
We use two standard facts.
\paragraph{(i) Expectation bound.} A classical bound for Gaussian matrices states that
\begin{equation}
    \E\sb{\norm{\AA}} \le \sqrt{m} + \sqrt{n}.
\end{equation}
This can easily be obtained in various ways, \emph{e.g.,} using Gordon's Theorem \citep[Thm. 2.5]{gordon1985some}.

\paragraph{(ii) Variance bound via Gaussian Poincaré.} The map $\AA \mapsto\|\AA\|$ is $1$-Lipschitz with respect to the Frobenius (Euclidean) norm, since for all matrices $\AA,\BB \in \R{ m \times n} $ we have
\begin{equation}
    \big| \fnorm{\AA\|_2-\|\BB\|_2 \big| \leq \norm{\AA-\BB}\le\|\AA-\BB}.
\end{equation}
Hence, by the Gaussian Poincaré inequality in Lemma \ref{lemma:poincare},
\begin{equation}
    \Var\big(\| \AA \|\big) \leq 1.
\end{equation}
Combining these two facts and using $\E [\norm{\AA}^2] = (\E \| \AA \|)^2 + \Var( \| \AA \| )$, we obtain
\begin{equation}
    \E\sb{\norm{\AA}^2} \leq \left( \sqrt{m}+\sqrt{n} \right)^2 + 1.
\end{equation}
Finally, using the elementary inequality $(\sqrt{a}+\sqrt{b})^2 \le 4\max\{a,b\}$ for $a,b \geq 0$, we have
\begin{equation}
    \E\sb{\norm{\AA}^2} \le 4\max\{m, n \} + 1 \le 5\max\{m,n\},
\end{equation}
which proves the lemma.
\end{proof}

\begin{lemma}\label{lemma:general_risk}
    Consider the regression task defined in Sec.~\ref{sec:setting}, and let $\empsol$ be some estimator. Then
    \begin{equation}
        \risk\big(\empsol\big) - \risk \big( \AA^\star \big) = \fnorm{\left( \empsol -\AA^{\star}\right)\CovMat_{\xx \xx}^{\frac{1}{2}}}^2.
    \end{equation}
\end{lemma}
\begin{proof}
The expected excess risk is defined as
\begin{equation}
    \risk\big(\empsol\big) - \risk \big( \AA^\star \big) = \E \left[ \|\yy - \empsol \xx\|^2 - \|\yy - \AA^\star \xx\|^2 \right].
\end{equation}
Using $\yy = \AA^{\star} \xx + \vecnoise$, we compute
\begin{equation}
    \| \yy - \empsol \xx \|^2 = \|\vecnoise - (\empsol- \AA^\star) \xx\|^2 = \|\vecnoise\|^2 - 2 \vecnoise^\top (\empsol - \AA^\star) \xx + \xx^\top (\empsol - \AA^\star)^\top (\empsol - \AA^\star) \xx,
\end{equation}
while
\begin{equation}
    \risk \big( \AA^\star \big) = \| \yy - \AA^\star \xx\|^2 = \|\vecnoise\|^2.
\end{equation}
Thus,
\begin{equation}
    \risk\big(\empsol\big) - \risk \big( \AA^\star \big) = \E\left[ -2 \vecnoise^\top (\empsol - \AA^\star) \xx + \xx^\top (\empsol- \AA^\star)^\top (\empsol - \AA^\star) \xx \right].
\end{equation}
Since $\xx$ and $\vecnoise$ are independent and $\E[\vecnoise] = \zeroVec $, the cross term vanishes, \emph{i.e.},
\begin{equation}
    \E\left[ \vecnoise^\top (\empsol- \AA^\star) \xx \right] = \E\left[ \vecnoise^\top\right] (\empsol- \AA^\star) \xx = 0.
\end{equation}
Hence, the excess risk is,
\begin{align}
    \risk(\empsol) - \risk(\AA^\star)
    & = \E\left[ \xx^\top (\empsol - \AA^\star)^\top (\empsol- \AA^\star) \xx \right]  \nonumber \\
    & = \E \sb{\trace\rb{\xx^\top (\empsol - \AA^\star)^\top (\empsol - \AA^\star) \xx}} \nonumber \\
    & = \E \left[ \trace\rb{ (\empsol - \AA^\star) \xx \xx^\top (\empsol - \AA^\star)^\top } \right] \nonumber \\
    & = \trace\rb{ (\empsol - \AA^\star) \E\big[ \xx \xx^\top \big]  (\empsol - \AA^\star)^\top } \nonumber \\
    & = \trace \rb{ (\empsol - \AA^\star) \CovMat_{\xx \xx} (\empsol - \AA^\star)^\top} \nonumber \\
    & = \fnorm{(\empsol -\AA^{\star})\CovMat_{\xx \xx}^{\frac{1}{2}}}^2.
\end{align}
\end{proof}

\section{Derivation of the explicit solutions}
In this section, we derive the explicit solution found by LoRA and FFT during the fine-tuning. We note that these results are well known and do not claim any novelty in this regard. We include these derivations for the completeness of this work.

\subsection{Explicit solution of full fine-tuning}\label{app_FFT_sol}
We start by deriving the solution found by GD during fine-tuning (FFT).
\begin{lemma}[FFT solution under GD]\label{lemma:FFT solution under GD}
    Let $ \XX \in \R^{d_x \times n} $ and $ \YY \in \R^{d_y \times n } $.  Consider the minimization problem
    \begin{equation}
        \min_{\AA\in\R^{d_y\times d_x}}\sb{\frac{1}{2}\fnorm{\AA \XX - \YY}^2}.
    \end{equation}
    Let $ \{ \AA_{t} \} $ be the iterates of GD with step size $ \eta $, applied to minimize this problem. Assume that $ \XX $ is full rank, {i.e.,} $ \rank(\XX) = \min\{ n,d_{x} \} $. If $ \eta < 2/\lambda_{\max}(\XX^\top \XX ) $, then GD will converge to
    \begin{equation}
        \gdsol = \lim_{t \to \infty} \AA_{t} = \initmodel + \rb{\YY - \initmodel \XX} \XX^\dagger.
    \end{equation}
\end{lemma}
\begin{proof}
    We first note that when $n \geq d_x$, there is a unique minimizer of the objective. Specifically, in this case we have that $ \XX \XX^\dagger = \II $, and therefore $ \gdsol $ is reduced to
    \begin{equation}
        \gdsol = \initmodel + \YY \XX^\dagger - \initmodel \XX \XX^\dagger = \YY \XX^\dagger = \YY \XX^\top \left(\XX \XX^\top\right)^{-1}.
    \end{equation}
    This case is trivial, and its analysis can be found in many sources, \emph{e
    .g.}, \citet{bach2024learning}. Let us consider the underdetermined case where $n < d_x$. Recall that our objective function is
    \begin{equation}
        \ell(\AA) = \frac{1}{2} \fnorm{\AA\XX-\YY}^2,
    \end{equation}
    and therefore its gradient in matrix form is given by
    \begin{equation}
        \nabla \ell(\AA) = \rb{\AA\XX - \YY}\XX^\top.
    \end{equation}
    We first show that GD converges to a solution satisfying $\AA\XX = \YY$, i.e., achieving zero error. Let us write one step of GD.
    \begin{align}
        \AA_{t+1} & = \AA_t - \eta \nabla \ell(\AA_t) \nonumber \\
        & = \AA_t - \eta \rb{\AA_t\XX - \YY}\XX^\top \nonumber \\
        \Rightarrow \quad \AA_{t+1}\XX - \YY & = \AA_{t}\XX - \YY - \eta\rb{\AA_t\XX - \YY}\XX^\top\XX \nonumber \\
        \Rightarrow \quad \AA_{t+1}\XX - \YY & = \rb{\AA_{t}\XX - \YY}\rb{\II - \eta \XX^\top\XX}.
    \end{align}
    Unrolling the recursion for $ T $ steps gives
    \begin{equation}
        \AA_{T}\XX - \YY = \rb{\initmodel\XX - \YY}\rb{\II - \eta \XX^\top\XX}^\top.
    \end{equation}
    Since the step size $\eta \leq {2} / {\lambda_{\max}\rb{\XX^\top\XX}}$ and $\XX^\top\XX$ is full-rank, we have that $ \| \II - \eta \XX^\top\XX  \| < 1 $. This guarantees the convergence of $\AA_{T}\XX$ to $\YY$. Here, we denote the limit $ T \to \infty$ as $\AA_\infty \XX = \YY$. Now, let us write the first step of GD.
    \begin{equation}
        \AA_1 = \initmodel - \eta(\initmodel \XX -\YY)\XX^\top = \initmodel - \CC_0 \XX^\top.
    \end{equation}
    Where $\CC_0:= \eta(\initmodel \XX -\YY)$. Next, 
    \begin{align}
        \AA_2 &= \AA_1 - \eta (\AA_1 \XX - \YY)\XX^\top \nonumber \\
        &= \initmodel - \rb{\CC_0 + \eta (\AA_1 \XX - \YY)}\XX^\top \nonumber \\
        &= \initmodel - \CC_1 \XX^\top.
    \end{align}
    Where $\CC_1:= \CC_0 + \eta (\AA_1 \XX - \YY)$. In principle, it is possible to define $\CC_t$ for any time step and write the general update rule of GD as
    \begin{align}
        \AA_{t+1} &= \initmodel - \CC_t \XX^\top \nonumber \\
        \Rightarrow \quad  \AA_{t+1} \XX &= \initmodel\XX- \CC_t \XX^\top\XX \nonumber \\
        \Rightarrow \quad \AA_{\infty} \XX &= \initmodel\XX- \CC_{\infty} \XX^\top\XX \nonumber \\
        \Rightarrow \quad \YY &= \initmodel\XX- \CC_{\infty} \XX^\top\XX.
    \end{align}
    Now, since $\XX^\top\XX$ is invertible, we have 
    \begin{equation}
      \CC_{\infty} = \initmodel\XX(\XX^\top\XX)^{-1} - \YY(\XX^\top\XX)^{-1}=(\initmodel\XX- \YY)(\XX^\top\XX)^{-1}.
    \end{equation}
    Now recall that $\AA_{\infty} = \initmodel - \CC_{\infty} \XX^\top$. Replacing $\CC_{\infty}$ gives
    \begin{align}
        \gdsol = \AA_{\infty} & = \initmodel - (\initmodel\XX- \YY)(\XX^\top\XX)^{-1}\XX^\top \nonumber \\
        & = \initmodel - (\initmodel\XX- \YY) \XX^\dagger = \initmodel + (\YY - \initmodel\XX) \XX^\dagger = \initmodel + \widehat \bDelta_{\operatorname{FFT}} .
    \end{align}
    Note that the above solution can be simplified in the overdetermined regime where $n > d_x$. In this regime, we have, 
    \begin{equation}
        \widehat \bDelta_{\operatorname{FFT}} = (\YY - \initmodel\XX) \XX^\dagger = (\YY - \initmodel\XX)\XX^\top \rb{\XX\XX^\top}^{-1} = \empcov_{\yy\xx}\empcov_{\xx\xx}^{-1} - \initmodel.
    \end{equation}
    So, 
    \begin{equation}
        \gdsol = \initmodel + \widehat \bDelta_{\operatorname{FFT}} = \initmodel + \empcov_{\yy\xx}\empcov_{\xx\xx}^{-1} - \initmodel = \empcov_{\yy\xx}\empcov_{\xx\xx}^{-1}.
    \end{equation}
    It is well known that in the overdetermined regime, GD converges to the least-squares solution independently of its initialization. 
\end{proof}

\subsection{Explicit solution of rank constraint linear regression and LoRA}\label{app:low_rank_sol}
\begin{lemma}[Low-rank solution]\label{lemma:low_rank_sol}
    Let $ \XX \in \R^{d_x \times n} $ and $ \YY \in \R^{d_y \times n } $.  Consider the rank-constraint minimization problem
    \begin{equation}\label{eq:Rank Constraint Linear Regression}
        \min_{\AA \in \mathbb{R}^{d_y \times d_x} \, : \, \rank(\AA) \leq r } \fnorm{\YY - \AA \XX}^2.
    \end{equation}
    Then
    \begin{equation}\label{eq:optimal rank constraint solution}
        \empsol_r = \trunk{\YY\XX^\dagger\empcov_{\xx\xx}^{\frac{1}{2}}}\rb{\empcov_{\xx\xx}^{\frac{1}{2}}}^\dagger
    \end{equation}
    is a minimizer of \eqref{eq:Rank Constraint Linear Regression}. Furthermore, if
    \begin{equation}
        \sigma_{r}(\YY\XX^\dagger\empcov_{\xx\xx}^{\frac{1}{2}}) > \sigma_{r+1}(\YY\XX^\dagger\empcov_{\xx\xx}^{\frac{1}{2}})
    \end{equation}
    then $ \empsol_r $ has the minimal Frobenius norm among all minimizers of \eqref{eq:Rank Constraint Linear Regression}.
\end{lemma}
\begin{proof}
Let $\empsol=\YY\XX^\dagger$ be a min-norm solution of the unconstrained (OLS) problem. Let us reformulate the objective in \eqref{eq:Rank Constraint Linear Regression} as
\begin{align}
    \fnorm{\YY-\AA \XX}^2 
    & = \fnorm{\YY- \empsol \XX + \rb{\empsol  - \AA} \XX}^2 \nonumber \\
    & = \fnorm{\YY- \empsol \XX}^2 + \fnorm{\rb{\empsol  - \AA} \XX}^2 + 2\trace \rb{\XX^\top \rb{\empsol  - \AA}^\top \rb{\YY- \empsol \XX}}.
\end{align}
Note that the third term above is, in fact, zero since
\begin{align}
    \trace \rb{\XX^\top \rb{\empsol  - \AA}^\top \rb{\YY- \empsol \XX}}
    & = \trace \rb{\rb{\YY- \empsol \XX}\XX^\top \rb{\empsol  - \AA}^\top} \nonumber \\
    & = \trace \rb{\rb{\YY\XX^\top - \empsol \XX \XX^\top }\rb{\empsol  - \AA}^\top} \nonumber \\
    & = \trace \rb{\rb{\YY\XX^\top - \YY\XX^\dagger \XX \XX^\top }\rb{\empsol  - \AA}^\top}=0.
\end{align}
Here, in the last step, we used the fact that $ \XX^\dagger \XX $ is an orthogonal projection matrix onto the row space of $ \XX $ and therefore $ \XX^\top= \XX^\dagger \XX \XX^\top$. We are left with the following constraint optimization problem.
\begin{equation}
   \min_{\AA \in \mathbb{R}^{d_y \times d_x}: \rank(\AA) \leq r} \quad \left\{\fnorm{\YY- \empsol \XX}^2 + \fnorm{\rb{\empsol  - \AA} \XX}^2\right\}.
\end{equation}
Since the first term is independent of $\AA$, the minimizers of this optimization problem are the same as the minimizers of
\begin{equation} \label{eq:rank_constraint_formulation}
    \min_{\AA \in \mathbb{R}^{d_y \times d_x}: \rank(\AA) \leq r} \quad \fnorm{\rb{\empsol  - \AA} \XX}^2 .
\end{equation}
Let us rewrite the reformulated objective from \eqref{eq:rank_constraint_formulation} using trace
\begin{align}
    \fnorm{\rb{\empsol  - \AA} \XX}^2 &= \trace \rb{ \rb{\rb{\empsol  - \AA} \XX} \rb{\rb{\empsol - \AA} \XX}^\top} \nonumber \\
    & = \trace \rb{\rb{\empsol  - \AA} \XX \XX^\top \rb{\empsol  - \AA}^\top} \nonumber \\
    & = \trace \rb{\rb{\empsol  - \AA} \rb{\XX \XX^\top}^{\frac{1}{2}}\rb{\XX \XX^\top}^{\frac{1}{2}} \rb{\empsol  - \AA}^\top} \nonumber \\
    & = \trace \rb{\rb{\empsol  - \AA} \rb{\XX \XX^\top}^{\frac{1}{2}} \rb{\rb{\empsol  - \AA} \rb{\XX \XX^\top}^{\frac{1}{2}}}^{\top}} \nonumber \\
    & = \fnorm{\rb{\empsol  - \AA} \rb{\XX \XX^\top}^{\frac{1}{2}}}^2 \nonumber \\
    & = n\fnorm{\rb{\empsol  - \AA} \empcov_{\xx\xx}^{\frac{1}{2}}}^2.
\end{align}
Therefore, \eqref{eq:rank_constraint_formulation} can be further reduced to
\begin{equation}\label{eq:simplified formulation optimization problem}
    \min_{\rank(\AA)\leq r} \fnorm{\rb{\empsol  - \AA} \empcov_{\xx\xx}^{\frac{1}{2}}}^2.
\end{equation}
By a simple change of variable $\MM = \empsol \empcov_{\xx\xx}^{\frac{1}{2}} $ and $\ZZ = \AA \empcov_{\xx\xx}^{\frac{1}{2}}$, we have that the solutions of the optimization problem in \eqref{eq:simplified formulation optimization problem} are equivalent to the solutions of
\begin{equation}\label{eq:rank_constraint_formulation2}
    \min_{\ZZ \in \R^{d_y \times d_x}} \fnorm{\MM - \ZZ}^2 \qquad \text{s.t.} \qquad \exists \AA \in \R^{d_y \times d_x} : \ZZ = \AA \empcov_{\xx\xx}^{\frac{1}{2}} , \ \rank(\AA) \leq r.
\end{equation}
Let 
\begin{equation}
    \ZZ^\star
    = \empsol_r \empcov_{\xx\xx}^{\frac{1}{2}}
    =  \trunk{\MM} \rb{\empcov_{\xx\xx}^{\frac{1}{2}}}^\dagger \empcov_{\xx\xx}^{\frac{1}{2}}
    = \trunk{\MM} \PP_{\XX},
\end{equation}
where $ \PP_{\XX} $ is a projection matrix onto the column space of the matrix $ \XX $, and we used the fact that $ \rb{\empcov_{\xx\xx}^{\frac{1}{2}}}^\dagger \empcov_{\xx\xx}^{\frac{1}{2}} =  \PP_{\XX} $. We first show that $ \ZZ^\star $is a minimizer of \eqref{eq:rank_constraint_formulation2}. This will show that $ \empsol_r $ is an optimal solution. Then, we prove that $ \empsol_r $ has the minimal Frobenius norm among all minimizers under the assumption that $ \sigma_{r}(\YY\XX^\dagger\empcov_{\xx\xx}^{\frac{1}{2}}) > \sigma_{r+1}(\YY\XX^\dagger\empcov_{\xx\xx}^{\frac{1}{2}}) $.

Note that any $ \ZZ $ that satisfies the constraints in \eqref{eq:rank_constraint_formulation2} has at most rank $r$. This stems from the fact that 
\begin{equation}
    \rank(\ZZ) = \rank\left(\AA \empcov_{\xx\xx}^{\frac{1}{2}} \right) \leq \rank(\AA) \leq r.
\end{equation}
Therefore, by looking at a superset of the feasible set, the minimal value of the optimization problem in \eqref{eq:rank_constraint_formulation2} is bounded from below by
\begin{equation}\label{eq:reduced constraint optimization problem}
    \min_{\rank(\ZZ)\leq r } \fnorm{\MM - \ZZ}^2.
\end{equation}
The optimal value of this problem is given by the Eckart-Young theorem (Lemma~\ref{lemma:young_rank}). Therefore, since $ \ZZ^\star $ satisfies the constraints in \eqref{eq:rank_constraint_formulation2} we have
\begin{equation}\label{eq:constraint reduction}
    \fnorm{\MM - \ZZ^\star}^2 \geq 
    \min_{\substack{\exists \AA \in \R^{d_y \times d_x} \, : \, \ZZ = \AA \empcov_{\xx\xx}^{\frac{1}{2}}\\\rank(\AA) \leq r} } \fnorm{\MM - \ZZ}^2
    \geq \min_{\rank(\ZZ)\leq r } \fnorm{\MM - \ZZ}^2
    = \sum_{i = r+1}^{\min \{ d_x, d_y \} } \sigma_{i}^2 (\MM).
\end{equation}
To show that $ \ZZ^\star $ is optimal, it is sufficient to show that it achieves this bound. It is easy to see that $ \empcov_{\xx\xx}^{\frac{1}{2}} = \empcov_{\xx\xx}^{\frac{1}{2}} \PP_{\XX} $.  Therefore $ \MM = \empsol \empcov_{\xx\xx}^{\frac{1}{2}} = \empsol \empcov_{\xx\xx}^{\frac{1}{2}} \PP_{\XX} = \MM \PP_{\XX} $. Then, plugging~$ \ZZ^\star $ in the objective yields
\begin{align}
    \fnorm{\MM - \ZZ^\star}^2
    & = \fnorm{\MM \PP_{\XX} - \trunk{\MM} \PP_{\XX}}^2 \nonumber \\
    & = \fnorm{ (\MM - \trunk{\MM}) \PP_{\XX}}^2 \nonumber \\
    & \leq \fnorm{\MM - \trunk{\MM}}^2 \nonumber \\
    & = \sum_{i = r+1}^{\min \{ d_x, d_y \} } \sigma_{i}^2 (\MM),
\end{align}
where in the last step we used Eckart-Young theorem.
Hence $ \fnorm{\MM - \ZZ^\star}^2 = \sum_{i = r+1}^{\min \{ d_x, d_y \} } \sigma_{i}^2 (\MM) $ and we get that $ \ZZ^\star $ is a global minimum of \eqref{eq:rank_constraint_formulation2}. Therefore, $ \empsol_r $ is a minimizer of \eqref{eq:Rank Constraint Linear Regression}.

Furthermore, assuming $ \sigma_{r}(\MM) > \sigma_{r+1}(\MM) $, we have that $ \ZZ^\star $ is a unique solution for \eqref{eq:reduced constraint optimization problem}. Because it is unique in the larger feasible set of $ \{ \ZZ : \rank(\ZZ) \leq r \} $, it is also a unique minimizer in every subset that contains it, particularly in \eqref{eq:rank_constraint_formulation2}. Therefore, under this additional assumption, $ \ZZ^\star $ is a unique minimizer for \eqref{eq:rank_constraint_formulation2}.

This means that all minimizers of \eqref{eq:simplified formulation optimization problem}  satisfy $ \ZZ^\star = \AA \empcov_{\xx\xx}^{\frac{1}{2}} $. Given a minimizer $ \AA $, we can present it as pertubation of~$ \empsol_r $, that is, $ \AA = \empsol_r + \QQ $. Then, by optimality, it should satisfy $\ZZ^\star = \AA\empcov_{\xx\xx}^{\frac{1}{2}}$. Then, we have
\begin{equation}
    \empsol_r \empcov_{\xx\xx}^{\frac{1}{2}} = \ZZ^\star = \AA\empcov_{\xx\xx}^{\frac{1}{2}} = \left(\empsol_r + \QQ\right) \empcov_{\xx\xx}^{\frac{1}{2}} \qquad \Rightarrow \qquad \QQ \empcov_{\xx\xx}^{\frac{1}{2}} = \zeroVec.
\end{equation}
Therefore, by optimality and the constraint of \eqref{eq:simplified formulation optimization problem}, we have the following conditions on $ \QQ $:
\begin{equation}
  \QQ \empcov_{\xx\xx}^{\frac{1}{2}} = \zeroVec \qquad \text{and} \qquad \rank\left(\empsol_r + \QQ\right) \leq r.
\end{equation}
Since $ \QQ \empcov_{\xx\xx}^{\frac{1}{2}} = \zeroVec $, we get $  \image(\empcov_{\xx\xx}^{\frac{1}{2}}) \subseteq \kernel (\QQ )$. Note that
\begin{equation}
    \kernel^{\perp}\left(\empsol_r \right) = \kernel^{\perp}\left(\trunk{\MM}\rb{\empcov_{\xx\xx}^{\frac{1}{2}}}^\dagger \right) \subseteq \kernel^{\perp} \left( \rb{\empcov_{\xx\xx}^{\frac{1}{2}}}^\dagger \right) = \image \left( \empcov_{\xx\xx}^{\frac{1}{2}} \right).
\end{equation}
Thus, $ \kernel^{\perp}(\empsol_r) \subseteq \kernel (\QQ ) $, namely, the row space of $ \QQ $ should be orthogonal to the row space of $ \empsol_r $. Therefore, for any solution $ \AA $
\begin{equation}
    \fnorm{\AA}^2 = \fnorm{\empsol_r + \QQ}^2 = \fnorm{\empsol_r}^2 + \fnorm{\QQ}^2 + 2\trace \left(\QQ \empsol_r^\top \right) = \fnorm{\empsol_r}^2 + \fnorm{\QQ}^2 \geq \fnorm{\empsol_r}^2.
\end{equation}
From here, we deduce that $ \empsol_r $ is the min-norm solution (in the Frobenius norm).
\end{proof}

\lemmaLoraSol*
\begin{proof}
The LoRA estimator $ \lorasol $, is given by
\begin{equation}
    \lorasol  =  \initmodel + \widehat\bDelta_{\operatorname{LoRA}}
    \qquad \text{s.t.} \qquad 
    \widehat\bDelta_{\operatorname{LoRA}} = \argmin_{\rank(\bDelta)\le r}
    \sb{\fnorm{\big(\initmodel+\bDelta\big) \XX-\YY}^2}.
\end{equation}
Let $\WW = \YY-\initmodel\XX$, then 
\begin{equation}
    \fnorm{\big(\initmodel+\bDelta\big) \XX-\YY}^2 
    = \fnorm{\bDelta \XX - (\YY - \initmodel \XX)}^2
    = \fnorm{\bDelta \XX - \WW}^2.
\end{equation}
Therefore,
\begin{equation}
    \widehat\bDelta_{\operatorname{LoRA}} = \argmin_{\rank(\bDelta)\le r}
    \sb{\fnorm{\bDelta \XX - \WW}^2}.
\end{equation}
Applying Lemma~\ref{lemma:low_rank_sol}, we have
\begin{equation}
   \widehat\bDelta_{\operatorname{LoRA}} = \trunk{\WW\XX^\dagger\empcov_{\xx\xx}^{\frac{1}{2}}}\rb{\empcov_{\xx\xx}^{\frac{1}{2}}}^\dagger.
\end{equation}
Replacing the definition of $\WW$, 
\begin{equation}
    \widehat\bDelta_{\operatorname{LoRA}} = \trunk{(\YY-\initmodel\XX)\XX^\dagger\empcov_{\xx\xx}^{\frac{1}{2}}}\rb{\empcov_{\xx\xx}^{\frac{1}{2}}}^\dagger = 
    \trunk{\widehat \bDelta_{\operatorname{FFT}} \empcov_{\xx\xx}^{\frac{1}{2}}}\rb{\empcov_{\xx\xx}^{\frac{1}{2}}}^\dagger.
\end{equation}
Using $\lorasol=\initmodel+\widehat\bDelta_{\operatorname{LoRA}}$ we get the solution for \ref{eq:lora_update}. Furthermore, Lemma~\ref{lemma:low_rank_sol} also states that if the truncation is unique, then $ \widehat\bDelta_{\operatorname{LoRA}} $ is the solution with the smallest norm. Therefore, $\lorasol$ is the nearest low-rank adaptation to $ \initmodel $ in Frobenius norm.
 
\end{proof}

\section{Excess risk analysis for full fine-tuning}
In this section, we provide the analysis of the excess risk for FFT.

\subsection{Overdetermined regime}\label{subsection:app_FFT_over}
\GDExcessOver*
\begin{proof}
Lemma~\ref{lemma:general_risk} states that for any estimator $\AA$ we have
\begin{equation}\label{eq:general excess risk simplified}
    \risk(\AA) - \risk (\AA^\star) =  \trace \rb{(\AA - \AA^\star)^\top (\AA - \AA^\star) \CovMat_{\xx \xx}}.
\end{equation}
From Lemma~\ref{lemma:FFT_sol}, the FFT is given by
\begin{equation}
    \gdsol = \initmodel + \rb{\YY - \initmodel \XX} \XX^\dagger.
\end{equation}
In the overdetermined case, when $ n \geq d_x $, $ \XX $ has a full row rank with probability one. In such a case, $\XX^{\dagger}$ has the closed-form expression
\begin{equation}
    \XX^{\dagger} = \XX^\top (\XX \XX^\top)^{-1}.
\end{equation}
Therefore,
\begin{align}
    \gdsol 
    & = \initmodel + \YY\XX^\dagger - \initmodel \XX \XX^\dagger \nonumber \\
    & = \initmodel + \YY\XX^\top (\XX \XX^\top)^{-1} - \initmodel \XX \XX^\top (\XX \XX^\top)^{-1} \nonumber \\
    & = \YY \XX^\top (\XX \XX^\top)^{-1}.
\end{align}
Note that in the overdetermined regime, the fine-tuned model does not depend on the pre-trained values. Namely, this estimator is just OLS. From the data generation model in \eqref{eq:data generation model}, we have that $ \YY = \AA^\star \XX  + \matnoise$. Using this, we get
\begin{align}
    \gdsol & =  \left(\AA^\star \XX  + \matnoise \right)\XX^\top (\XX \XX^\top)^{-1} \nonumber \\
   & = \AA^\star\XX \XX^\top (\XX \XX^\top)^{-1} +  \matnoise\XX^\top (\XX \XX^\top)^{-1} \nonumber \\
   & = \AA^\star +  \matnoise\XX^\top (\XX \XX^\top)^{-1}.
\end{align}
Thus, $ \gdsol-\AA^\star = \matnoise\XX^\top (\XX \XX^\top)^{-1} $. Placing this in \eqref{eq:general excess risk simplified} yields 
\begin{align}\label{eq:overdetermined fft excess risk main steps}
    \E \sb{\risk(\gdsol)} - \risk (\AA^\star)
    & =  \trace \rb{\left(\gdsol - \AA^\star\right)^\top \left(\gdsol - \AA^\star\right) \CovMat_{\xx \xx}}\nonumber \\
    & = \E \left[ \trace\left( (\XX \XX^\top)^{-1} \XX \E[\matnoise^\top \matnoise]  \XX^\top (\XX \XX^\top)^{-1} \CovMat_{\xx \xx} \right) \right] .
\end{align}

In order to compute $\E[\matnoise^\top\matnoise]$ we have,
\begin{equation}
  \E\sb{\rb{\matnoise^\top\matnoise}_{ij}} = \E\sb{\vecnoise^\top_i\vecnoise_j} =
    \begin{cases}
      \E\sb{\trace(\vecnoise^\top_i\vecnoise_i)}=\E\sb{\trace(\vecnoise_i\vecnoise_i^\top)}=\trace\rb{\CovMat_{\vecnoise\vecnoise}} & i=j\\
      0 & i \neq j\\
    \end{cases}       
\end{equation}
Since noise vectors $i$ and $j$ are independent. Then we continue, 
\begin{align}
        \E \sb{\risk(\gdsol)} - \risk (\AA^\star)  = \trace(\CovMat_{\vecnoise\vecnoise}) \E \sb{\trace \rb{(\XX \XX^\top)^{-1} \CovMat_{\xx \xx} }}.
\end{align}
We can further simplify the term $\XX \XX^\top$,
\begin{equation}
    \XX \XX^\top = n \rb{\frac{1}{n} \sum_{i=1}^n \xx_i \xx^\top_i} = n \empcov_{\xx \xx}.
\end{equation}
Plugging this back into \eqref{eq:overdetermined fft excess risk main steps}, we have
\begin{align}
    \E \sb{\risk(\gdsol)} - \risk (\AA^\star)
    & = \trace\left(\CovMat_{\vecnoise\vecnoise}\right) \E \sb{\trace \rb{\rb{n \empcov_{\xx \xx}}^{-1}\CovMat_{\xx \xx}}} \nonumber \\
    & = \frac{1}{n} \trace\left(\CovMat_{\vecnoise\vecnoise}\right) \E  \sb{\trace \rb{\empcov_{\xx \xx}^{-1}\CovMat_{\xx \xx}}}.
\end{align}
Furthermore, note that using Jensen's inequality, we get
\begin{align}
    \E  \sb{\trace \rb{\empcov_{\xx \xx}^{-1}\CovMat_{\xx \xx}}}
    & = \E  \sb{\trace \rb{ \CovMat^{\frac{1}{2}}_{\xx \xx}\empcov_{\xx \xx}^{-1}\CovMat^{\frac{1}{2}}_{\xx \xx}}} \nonumber \\
    & = \E\sb{\trace \rb{ \rb{\CovMat^{-\frac{1}{2}}_{\xx \xx}\empcov_{\xx \xx}\CovMat^{-\frac{1}{2}}_{\xx \xx}}^{-1}}}
    \geq \trace \rb{  \rb{ \E\sb{\CovMat^{-\frac{1}{2}}_{\xx \xx}\empcov_{\xx \xx}\CovMat^{-\frac{1}{2}}_{\xx \xx}}}^{-1}},
\end{align}
where we used the fact that $ f( \CovMat ) =  \trace (\CovMat^{-1} ) $ is convex over the set of PSD matrices. Then,
\begin{equation}
    \E\sb{\CovMat^{-\frac{1}{2}}_{\xx \xx}\empcov_{\xx \xx}\CovMat^{-\frac{1}{2}}_{\xx \xx}}
    = \CovMat^{-\frac{1}{2}}_{\xx \xx}\E\sb{\empcov_{\xx \xx}} \CovMat^{-\frac{1}{2}}_{\xx \xx}
    = \CovMat^{-\frac{1}{2}}_{\xx \xx}\CovMat_{\xx \xx} \CovMat^{-\frac{1}{2}}_{\xx \xx} = \II_{d_x}.
\end{equation}
Therefore,
\begin{equation}
    \E  \sb{\trace \rb{\empcov_{\xx \xx}^{-1}\CovMat_{\xx \xx}}}
    \geq \trace(\II) = d_x.
\end{equation}
Denote $ \bar{\sigma}^2_{\vecnoise} = \E\big[ \| \vecnoise \|^2 \big] / d_y = \trace(\CovMat_{\vecnoise\vecnoise})/ d_y $, then 
\begin{equation}
    \trace \rb{(\AA - \AA^\star)^\top (\AA - \AA^\star) \CovMat_{\xx \xx}} \geq \bar{\sigma}^2_{\vecnoise} \frac{d_y d_x}{n}.
\end{equation}
\end{proof}

\corGaussOver*
\begin{proof}
Recall from Theorem~\ref{theorem:Gd_excessrisk_overdetermined} that
\begin{align}
\E \sb{ \risk\big(\gdsol\big)} - \risk \big( \AA^\star \big)
= \frac{1}{n} \trace\left(\CovMat_{\vecnoise\vecnoise}\right)
\E \sb{\trace \rb{\empcov_{\xx \xx}^{-1}\CovMat_{\xx \xx}}}.
\end{align}
Since $\xx_i \sim \mathcal{N}(\zeroVec,\CovMat_{\xx\xx})$ i.i.d., we can write
$\xx_i = \CovMat_{\xx\xx}^{\frac{1}{2}} \zz_i$ where $\zz_i \sim \mathcal{N}(\zeroVec,\II)$.
Let $\ZZ = [\zz_1,\dots,\zz_n] \in \mathbb{R}^{d_x \times n}$. Then
\[
\empcov_{\xx\xx}
= \frac{1}{n} \XX\XX^\top
= \CovMat_{\xx\xx}^{\frac{1}{2}} \!\left(\frac{1}{n} \ZZ\ZZ^\top\right)\! \CovMat_{\xx\xx}^{\frac{1}{2}}.
\]
Using the cyclicity of the trace, we obtain
\begin{align}
\trace\rb{\empcov_{\xx\xx}^{-1}\CovMat_{\xx\xx}}
&= \trace\left(
\CovMat_{\xx\xx}^{-\frac{1}{2}}
\left(\frac{1}{n} \ZZ\ZZ^\top\right)^{-1}
\CovMat_{\xx\xx}^{-\frac{1}{2}}
\CovMat_{\xx\xx}
\right) \nonumber\\
&= \trace\left(\left(\frac{1}{n} \ZZ\ZZ^\top\right)^{-1}\right).
\end{align}
Since $\ZZ\ZZ^\top \sim \mathcal{W}_{d_x}(\II,n)$ is a Wishart matrix and $n>d_x+1$,
it holds that (see \citet[Sec.~3.8.1]{bach2024learning})
\begin{equation}
    \E \sb{\left(\frac{1}{n} \ZZ\ZZ^\top\right)^{-1}} = \frac{n}{n-d_x-1}\II,
\end{equation}
and therefore
\begin{equation}
    \E \sb{\trace\rb{\empcov_{\xx\xx}^{-1}\CovMat_{\xx\xx}}}
    = \trace \left(\frac{n}{n-d_x-1}\II\right)
    = \frac{n d_x}{n-d_x-1}.
\end{equation}
Substituting this back yields
\begin{equation}
    \E \sb{ \risk\big(\gdsol\big)} - \risk \big( \AA^\star \big)
    = \frac{1}{n}\trace \left(\CovMat_{\vecnoise\vecnoise}\right)
    \frac{n d_x}{n-d_x-1},
\end{equation}
which concludes the proof.
\end{proof}

\subsection{Underdetermined regime}\label{app:FFT_under_proof}
\FullFTLowerUnder*
\begin{proof}
By Lemma~\ref{lemma:general_risk} we get
\begin{equation}
    \risk\big(\gdsol\big)-\risk\big(\AA^\star\big) = \trace \left( \big(\gdsol-\AA^\star\big)^\top \big(\gdsol-\AA^\star\big) \CovMat_{\xx\xx} \right).
\end{equation}
From Lemma~\ref{lemma:FFT_sol}, the FFT is given by
\begin{equation}\label{eq:FFT solution underdetermined}
    \gdsol = \initmodel + \rb{\YY - \initmodel \XX} \XX^\dagger.
\end{equation}
In the underdetermined case, when $ n < d_x $, $ \XX $ has a full column rank (rank $n$) with probability one. In such a case, $\XX^{\dagger}$ has the closed-form expression
\begin{equation}
    \XX^{\dagger} =  (\XX^\top \XX)^{-1}\XX^\top.
\end{equation}
Using $\YY = \AA^\star\XX+\matnoise$ and \eqref{eq:FFT solution underdetermined},
\begin{align}
    \gdsol - \AA^\star
    & = \initmodel + \rb{\YY - \initmodel \XX} \XX^\dagger - \AA^\star \nonumber \\
    & = \initmodel + \rb{\AA^\star\XX+\matnoise - \initmodel \XX} \XX^\dagger - \AA^\star \nonumber \\
    & = \rb{\bDelta^\star \XX + \matnoise}\XX^\dagger -\bDelta^\star \nonumber \\
    & = \bDelta^\star \XX \XX^\dagger + \matnoise \XX^\dagger - \bDelta^\star \nonumber \\
    & = \bDelta^\star \PP_{\XX}-\bDelta^\star + \matnoise \XX^{\dagger} \nonumber \\
    & = -\bDelta^\star(\II-\PP_{\XX}) + \matnoise \XX^{\dagger},
\end{align}
where we used the relation $\bDelta^\star = \AA^\star-\initmodel$ from \eqref{eq:delta_star}, and denoted the orthogonal projection matrix onto the column space of $ \XX $ by $ \PP_{\XX} = \XX \XX^{\dagger}$. Plugging this into the excess-risk formula gives
\begin{align}
    \risk\big(\gdsol\big)-\risk\big(\AA^\star\big)
    & = \trace \rb{\rb{-\bDelta^\star(\II-\PP_{\XX}) + \matnoise \XX^{\dagger}}^\top \rb{-\bDelta^\star(\II-\PP_{\XX}) + \matnoise \XX^{\dagger}}\CovMat_{\xx\xx}} \nonumber \\
    & = \trace \rb{\rb{-\bDelta^\star(\II-\PP_{\XX}) + \matnoise \XX^{\dagger}}\CovMat_{\xx\xx}\rb{-\bDelta^\star(\II-\PP_{\XX}) + \matnoise \XX^{\dagger}}^\top } \nonumber \\
    & = \trace (\bDelta^\star(\II-\PP_{\XX})\CovMat_{\xx\xx}(\II-\PP_{\XX})\bDelta^{\star\top})
    - 2\trace(\matnoise \XX^{\dagger}\CovMat_{\xx\xx}(\II-\PP_{\XX})^{\top} \bDelta^{\star\top})
    \nonumber \\
    & \qquad +  \trace \rb{\matnoise\XX^{\dagger}\CovMat_{\xx\xx}\left(\XX^{\dagger}\right)^{\top} \matnoise^{\top}}.
\end{align}
Let us take a conditional expectation over the training noise matrix $\matnoise$ given the training data $\XX$. Since the noise has zero mean, and it is indepandent of $ \XX $,we get
\begin{equation}
    \E\sb{\trace(\matnoise \XX^{\dagger}\CovMat_{\xx\xx}(\II-\PP_{\XX})^{\!\top} \bDelta^{\star\top}) \; \middle| \; \XX}
    = \trace(\E\sb{\matnoise \; \middle| \;  \XX} \XX^{\dagger}\CovMat_{\xx\xx}(\II-\PP_{\XX})^{\!\top} \bDelta^{\star\top})
    = 0.
\end{equation}
Therefore,
\begin{align}\label{eq:full_FT under two terms}
    \E\left[\risk\big(\gdsol\big) \; \middle| \; \XX \right] - \risk\big(\AA^\star\big)
    & = \trace \left(\bDelta^\star(\II-\PP_{\XX})\CovMat_{\xx\xx}(\II-\PP_{\XX})\bDelta^{\star\top} \right) \nonumber \\
    & \hspace{3cm} + \E \left[\trace \rb{\matnoise\XX^{\dagger}\CovMat_{\xx\xx}\left(\XX^{\dagger}\right)^{\top}  \matnoise^{\top}} \; \middle| \; \XX\right].
\end{align}
Now, for the second term, we have
\begin{align}
    \E \left[\trace \rb{\matnoise\XX^{\dagger}\CovMat_{\xx\xx}\left(\XX^{\dagger}\right)^{\top}  \matnoise^{\top}} \; \middle| \; \XX\right]
    & = \E \left[\trace \rb{ \matnoise^{\top} \matnoise\XX^{\dagger}\CovMat_{\xx\xx}\left(\XX^{\dagger}\right)^{\top} } \; \middle| \; \XX\right]\nonumber \\
    & = \trace \rb{ \E \left[ \matnoise^{\top} \matnoise \; \middle| \; \XX\right] \XX^{\dagger} \CovMat_{\xx\xx} \left(\XX^{\dagger}\right)^{\top} } \nonumber \\
    & = \trace \rb{ \big( \trace(\CovMat_{\vecnoise\vecnoise}) \II \big) \XX^{\dagger} \CovMat_{\xx\xx} \left(\XX^{\dagger}\right)^{\top} } \nonumber \\
    & = \trace(\CovMat_{\vecnoise\vecnoise}) \trace\left(\XX^{\dagger}\CovMat_{\xx\xx}\left(\XX^{\dagger}\right)^{\top}\right),
\end{align}
where we used the fact that $ \E \left[ \matnoise^{\top} \matnoise \; \middle| \; \XX\right] = \trace(\CovMat_{\vecnoise\vecnoise}) \II  $.
Note that
\begin{equation}
    \trace\left( \XX^{\dagger}\CovMat_{\xx\xx} \left(\XX^{\dagger}\right)^{\top} \right)
    =  \trace\left( \CovMat_{\xx\xx} \left(\XX^{\dagger}\right)^{\top} \XX^{\dagger} \right)
    = \trace\left( \CovMat_{\xx\xx} \left(\XX^{\top}\right)^{\dagger} \XX^{\dagger} \right).
\end{equation}
Because $ \XX $ has linearly independent columns with probability one, we have that $ (\XX^{\top})^{\dagger} \XX^{\dagger} = (\XX \XX^{\top})^{\dagger}$. Thus,
\begin{equation}
    \E \left[\trace \rb{\matnoise\XX^{\dagger}\CovMat_{\xx\xx}\left(\XX^{\dagger}\right)^{\top}  \matnoise^{\top}} \; \middle| \; \XX\right]
    = \trace\left( \CovMat_{\xx\xx} \left(\XX \XX^{\top} \right)^{\dagger}\right)
    = \frac{1}{n}\trace\left( \CovMat_{\xx\xx} \empcov_{\xx\xx}^\dagger \right).
\end{equation}
For now, we reduced \eqref{eq:full_FT under two terms} into
\begin{equation}\label{eq:app_overparameterized_bias_variance}
    \E\left[\risk\big(\gdsol\big) \; \middle| \; \XX \right] - \risk\big(\AA^\star\big)
    = \trace \left(\bDelta^\star(\II-\PP_{\XX})\CovMat_{\xx\xx}(\II-\PP_{\XX})\bDelta^{\star\top} \right)
    + \frac{1}{n}\trace(\CovMat_{\vecnoise\vecnoise})\trace\left( \CovMat_{\xx\xx} \empcov_{\xx\xx}^\dagger \right).
\end{equation}
Let us proceed with the first term, 
\begin{align}
    \trace\rb{\bDelta^\star(\II-\PP_{\XX})\CovMat_{\xx\xx}(\II-\PP_{\XX})\bDelta^{\star\top}} 
    & = \trace\rb{\bDelta^\star(\II-\PP_{\XX})\CovMat_{\xx\xx}^{\frac{1}{2}}\CovMat_{\xx\xx}^{\frac{1}{2}}(\II-\PP_{\XX})\bDelta^{\star\top}} \nonumber \\
    & = \trace\rb{\rb{\CovMat_{\xx\xx}^{\frac{1}{2}}(\II-\PP_{\XX})\bDelta^{\star\top}}^\top\CovMat_{\xx\xx}^{\frac{1}{2}}(\II-\PP_{\XX})\bDelta^{\star\top}} \nonumber \\
    & = \fnorm{\CovMat_{\xx\xx}^{\frac{1}{2}}(\II-\PP_{\XX})\bDelta^{\star\top}}^2.
\end{align}
Therefore,
\begin{equation}
    \E \left[\risk\big(\gdsol\big) \right] - \risk\big(\AA^\star\big)
    = \E\sb{ \fnorm{\CovMat_{\xx\xx}^{\frac{1}{2}}(\II-\PP_{\XX})\bDelta^{\star\top}}^2}
    + \frac{1}{n}\trace(\CovMat_{\vecnoise\vecnoise}) \E\sb{\trace\left( \CovMat_{\xx\xx} \empcov_{\xx\xx}^\dagger \right)}.
\end{equation}

\end{proof}

\LoRAExcessUnderProp*

\begin{proof}
Since $\xx_i \sim \mathcal{N}(\zeroVec,\sigma_{\xx}^2 \II)$ i.i.d., define
$\zz_i := \xx_i / \sigma_{\xx}$ so that $\zz_i \sim \mathcal{N}(\zeroVec,\II)$.
Let $\ZZ = [\zz_1,\dots,\zz_n] \in \mathbb{R}^{d_x \times n}$. Then the empirical covariance can be written as
\begin{equation}
    \empcov_{\xx\xx}
    = \frac{1}{n} \XX\XX^\top
    = \sigma_{\xx}^2 \frac{1}{n} \ZZ\ZZ^\top .
\end{equation}
Since $n<d_x$, the matrix $\ZZ\ZZ^\top$ has rank $n$ almost surely, and
\begin{equation}
    \empcov_{\xx\xx}^\dagger
    = \frac{1}{\sigma_{\xx}^2} \left(\frac{1}{n} \ZZ\ZZ^\top\right)^\dagger .
\end{equation}
Using $\CovMat_{\xx\xx}=\sigma_{\xx}^2 \II$ and cyclicity of the trace,
\begin{equation}
    \trace\left(\empcov_{\xx\xx}^\dagger \CovMat_{\xx\xx}\right)
    = \trace\left(
    \frac{1}{\sigma_{\xx}^2}\left(\frac{1}{n} \ZZ\ZZ^\top\right)^\dagger
    \sigma_{\xx}^2 \II
    \right)
    = \trace\left(\left(\frac{1}{n} \ZZ\ZZ^\top\right)^\dagger\right).
\end{equation}
Because $ \ZZ $ has linearly independent columns with probability one, we have that $ (\ZZ^{\top})^{\dagger} \ZZ^{\dagger} = (\ZZ \ZZ^{\top})^{\dagger}$ and $ \ZZ^{\dagger} = (\ZZ^\top \ZZ)^{-1} \ZZ^\top $. Therefore,
\begin{equation}
    \left(\frac{1}{n} \ZZ\ZZ^\top\right)^\dagger
    = n \big(\ZZ^{\top}\big)^{\dagger} \ZZ^{\dagger}
    = n \ZZ \big(\ZZ^\top \ZZ\big)^{-1}  \big(\ZZ^\top \ZZ\big)^{-1} \ZZ^\top.
\end{equation}
Thus,
\begin{equation}
    \trace\left(\left(\frac{1}{n} \ZZ\ZZ^\top\right)^\dagger\right)
    = n \trace\left( \ZZ \big(\ZZ^\top \ZZ\big)^{-1}  \big(\ZZ^\top \ZZ\big)^{-1} \ZZ^\top \right)
    = n \trace\left(  \big(\ZZ^\top \ZZ\big)^{-1} \right)
    = \trace\left(\left( \frac{1}{n} \ZZ^\top \ZZ \right)^{-1} \right),
\end{equation}
where we used the cyclic property of the trace. Since $\ZZ^\top\ZZ \sim \mathcal{W}_{n}(\II,d_x)$ with $n+1<d_x$, it is known that
\begin{equation}
    \E \sb{\left(\frac{1}{n} \ZZ^{\top}\ZZ\right)^{-1}}
    = \frac{n}{d_x-n-1} \II.
\end{equation}
Therefore,
\begin{equation}
    \E \sb{\trace\left(\left(\frac{1}{n} \ZZ\ZZ^\top\right)^\dagger \right)}
    = \trace\left(\E \sb{\left(\frac{1}{n} \ZZ^{\top}\ZZ\right)^{-1}}\right)
    = \frac{n}{d_x-n-1} \trace(\II)
    = \frac{n^2}{d_x-n-1}.
\end{equation}
Substituting back, we obtain
\begin{equation}
    \frac{1}{n} \trace(\CovMat_{\vecnoise\vecnoise})
    \E \sb{\trace\left(\empcov_{\xx\xx}^\dagger \CovMat_{\xx\xx}\right)}
    = \frac{1}{n}\trace(\CovMat_{\vecnoise\vecnoise}) \frac{n^2}{d_x-n-1}
    = \trace(\CovMat_{\vecnoise\vecnoise})\frac{n}{d_x-n-1},
\end{equation}
which concludes the first part of the proof. For the bias term, we get
\begin{align}
    \E\sb{ \fnorm{ \big. \smash{ \bDelta^{\star} \PP^{\perp}_{\XX} \CovMat_{\xx\xx}^{\frac{1}{2}}} }^2}
    & = \fnorm{ \big. \smash{ \bDelta^{\star}}}^2 -  \E\sb{ \fnorm{ \big. \smash{ \bDelta^{\star} \PP_{\XX}} }^2} \nonumber \\
    & = \fnorm{ \big. \smash{ \bDelta^{\star}}}^2 - \E\left[ \trace \left( \bDelta^{\star} \big(\bDelta^{\star}\big)^\top \PP_{\XX} \right)\right] \nonumber \\
    & = \fnorm{ \big. \smash{ \bDelta^{\star}}}^2 -  \trace \left( \bDelta^{\star} \big(\bDelta^{\star}\big)^\top \E[\PP_{\XX}] \right),
\end{align}
where in the second step we used $ \PP_{\XX} \PP_{\XX}^{\top} = \PP_{\XX} $. Using Lemma~\ref{lemma:random_rotation} we have that $ \PP_{\XX} = \frac{n}{d_x} \II $. Thereofre,
\begin{align}
    \E\sb{ \fnorm{ \big. \smash{ \bDelta^{\star} \PP^{\perp}_{\XX} \CovMat_{\xx\xx}^{\frac{1}{2}}} }^2}
    & = \fnorm{ \big. \smash{ \bDelta^{\star}}}^2 -  \trace \left( \bDelta^{\star} \big(\bDelta^{\star}\big)^\top \E[\PP_{\XX}] \right), \nonumber \\
    & = \fnorm{ \big. \smash{ \bDelta^{\star}}}^2 -  \frac{n}{d_x}\trace \left( \bDelta^{\star} \big(\bDelta^{\star}\big)^\top \right), \nonumber \\
    & = \fnorm{\bDelta^{\star}}^2 \frac{d_x - n}{d_x}
\end{align}
\end{proof}

\begin{lemma}\label{lemma:random_rotation}
    Let $\XX \in \R^{d \times n} $ be a random Gaussian matrix with entries $ \mathcal{N}(0,\sigma^2_{x}) $ i.i.d. Denote by $ \PP_{\XX} = \XX (\XX^{\top} \XX)^{-1} \XX^{\top} $ the projection matrix onto the column space of $ \XX $. Then $ \E [ \PP_{\XX} ] = \frac{n}{d} \II $.
\end{lemma}
\begin{proof}
    Let $ \RR \in \R^{d \times d} $ be a orthogonal matrix.  Note that $ \RR \XX $ is distributed identically to $ \XX $. Therefore,
    \begin{align}
        \E [ \PP_{\XX} ]
        = \E [ \PP_{\RR \XX} ]
        & = \E \left[ \RR \XX \big((\RR \XX)^{\top} (\RR \XX) \big)^{-1} (\RR \XX)^{\top} \right] \nonumber \\
        & = \RR \E \left[ \XX \big(\XX^{\top} \XX \big)^{-1} \XX^{\top} \right] \RR^{\top} = \RR \E [ \PP_{\XX} ] \RR^{\top}.
    \end{align}
    Therefore, $ \E [ \PP_{\XX} ] $ must satisfy $ \E [ \PP_{\XX} ] \RR = \RR \E [ \PP_{\XX} ] $ for all orthogonal matrices. Namely, it comute with any orthogonal matrix. Specifically for  $ \RR = \II - 2 \vv \vv^\top $, with $ \| \vv \| = 1 $. It is easy to validate that this is indeed an orthogonal matrix. Therefore,
    \begin{align}
        & (\II - 2 \vv \vv^\top) \E [ \PP_{\XX} ] = \E [ \PP_{\XX} ]  (\II - 2 \vv \vv^\top)  \qquad \forall \| \vv \| = 1  \nonumber \\
        \Leftrightarrow \qquad &  \vv ( \vv^\top \E [ \PP_{\XX} ]) = (\E [ \PP_{\XX} ] \vv) \vv^\top  \qquad \forall \| \vv \| =1 \nonumber \\
        \Rightarrow \qquad &  \vv ( \vv^\top \E [ \PP_{\XX} ] \vv ) = \E [ \PP_{\XX} ] \vv \qquad \forall \| \vv \| = 1.
    \end{align}
    Where in the last step we multiplied by $ \vv $ from the right both sides of the equation. Therefore, we have that $ \E [ \PP_{\XX} ] $ satisfies $ \E [ \PP_{\XX} ] \vv  = \lambda_{\vv} \vv $, with $ \lambda_{\vv} = \vv^\top \E [ \PP_{\XX} ] \vv $, for all $ \| \vv \| = 1 $. This means that any vector is an eigenvector for $\E [ \PP_{\XX} ]$. This holds only for matrices proportinal to the identaty matrix. Thus $ \E [ \PP_{\XX} ] = \lambda \II $ for some positive $ \lambda $. Using the fact that $ \trace(\PP_{\XX}) = n $ for any realization of $\XX$, we get
    \begin{align}
        \lambda d = \trace(\lambda \II) = \trace(\E [\PP_{\XX} ]) = \E [\trace(\PP_{\XX} )] = n.
    \end{align}
    Hence $ \lambda = n/d $.
\end{proof}

\section{Excess risk analysis for LoRA}
In this section, we provide detailed proofs for the LoRA solution in both overdetermined and underdetermined regimes. We start by introducing some useful lemmas used only for proofs related to LoRA.
\begin{lemma}\label{lemma:lowrank_second_form}
    We have shown in Lemma \ref{lemma:lora_sol} that
\begin{equation}
    \widehat \bDelta_{\operatorname{LoRA}} = \trunk{\widehat \bDelta_{\operatorname{FFT}} \empcov_{\xx\xx}^{\frac{1}{2}}}\rb{\empcov_{\xx\xx}^{\frac{1}{2}}}^{\dagger}.
\end{equation}
    Given the linear relation,
    \begin{equation}
        \YY = \AA^\star\XX + \matnoise,
    \end{equation}
    and, 
    \begin{equation}
        \widehat \bDelta_{\operatorname{FFT}} = \rb{\YY-\initmodel\XX}\XX^\dagger.
    \end{equation}
    The low-rank update $\widehat \bDelta_{\operatorname{LoRA}}$ can be written as, 
    \begin{equation}
        \widehat \bDelta_{\operatorname{LoRA}} = \trunk{\bDelta^\star \empcov_{\xx\xx}^{\frac{1}{2}} + \matnoise \XX^\dagger\empcov_{\xx\xx}^{\frac{1}{2}}} \rb{\empcov_{\xx\xx}^\frac{1}{2}}^\dagger. 
    \end{equation}
    In the overdetermined regime, we can further simplify the above equation to,
    \begin{equation}
        \widehat \bDelta_{\operatorname{LoRA}} = \trunk{\bDelta^\star\empcov_{\xx\xx}^{\frac{1}{2}} + \frac{1}{n}\matnoise\XX^\top\empcov_{\xx\xx}^{-\frac{1}{2}}}\empcov_{\xx\xx}^{-\frac{1}{2}}.
    \end{equation}
\end{lemma}
\begin{proof}
    Recall that, 
\begin{equation}
    \widehat \bDelta_{\operatorname{LoRA}} = \trunk{\widehat \bDelta_{\operatorname{FFT}} \empcov_{\xx\xx}^{\frac{1}{2}}}\rb{\empcov_{\xx\xx}^{\frac{1}{2}}}^{\dagger} =
    \trunk{\rb{\YY-\initmodel\XX}\XX^\dagger \empcov_{\xx\xx}^{\frac{1}{2}}}\rb{\empcov_{\xx\xx}^{\frac{1}{2}}}^{\dagger}.
\end{equation}
Given the the linear relation $\YY = \AA^\star \XX + \matnoise$, we can simplify the term inside the truncation further,
\begin{align}
    \trunk{\rb{\YY-\initmodel\XX}\XX^\dagger \empcov_{\xx\xx}^{\frac{1}{2}}}\rb{\empcov_{\xx\xx}^{\frac{1}{2}}}^{\dagger} &= 
    \trunk{\rb{\AA^\star \XX + \matnoise-\initmodel\XX}\XX^\dagger \empcov_{\xx\xx}^{\frac{1}{2}}}\rb{\empcov_{\xx\xx}^{\frac{1}{2}}}^{\dagger} \nonumber \\
    &= \trunk{\rb{\bDelta^\star \XX + \matnoise}\XX^\dagger \empcov_{\xx\xx}^{\frac{1}{2}}}\rb{\empcov_{\xx\xx}^{\frac{1}{2}}}^{\dagger} \nonumber \\
    &= \trunk{\rb{\bDelta^\star \XX + \matnoise}\XX^\dagger \rb{\XX\XX^\top}^{\frac{1}{2}}}\rb{\rb{\XX\XX^\top}^\frac{1}{2}}^{\dagger} .
\end{align}
Now we expand the term inside the truncation,
\begin{align}
    \rb{\bDelta^\star \XX + \matnoise}\XX^\dagger \rb{\XX\XX^\top}^{\frac{1}{2}} &= \bDelta^\star\XX \XX^\dagger\rb{\XX\XX^\top}^{\frac{1}{2}}+\matnoise \XX^\dagger \rb{\XX\XX^\top}^{\frac{1}{2}} \nonumber \\
    &= \bDelta^\star\XX\rb{\XX^\top\XX}^{-1}\XX^\top\rb{\XX\XX^\top}^{\frac{1}{2}}+\matnoise \XX^\dagger \rb{\XX\XX^\top}^{\frac{1}{2}} \nonumber \\
    &= \bDelta^\star \rb{\XX\XX^\top}^{\frac{1}{2}} + \matnoise \XX^\dagger \rb{\XX\XX^\top}^{\frac{1}{2}} \nonumber \\
    &= \sqrt{n}\bDelta^\star \empcov_{\xx\xx}^{\frac{1}{2}} + \sqrt{n}\matnoise \XX^\dagger\empcov_{\xx\xx}^{\frac{1}{2}}.
\end{align}
Where in the third line we used the fact that $\XX\rb{\XX^\top\XX}^{-1}\XX^\top$ is the orthogonal projection onto the column space of $\XX$ and it leaves $(\XX\XX^\top)^{\frac{1}{2}}$ unchanged. Finally, we have this other explicit form for the $\widehat \bDelta_{\operatorname{LoRA}}$,
\begin{align}
    \widehat \bDelta_{\operatorname{LoRA}} &= \trunk{\sqrt{n}\bDelta^\star \empcov_{\xx\xx}^{\frac{1}{2}} + \sqrt{n}\matnoise \XX^\dagger\empcov_{\xx\xx}^{\frac{1}{2}}}\rb{\rb{\XX\XX^\top}^{\frac{1}{2}}}^{\dagger} \nonumber \\
    &= \trunk{\bDelta^\star \empcov_{\xx\xx}^{\frac{1}{2}} + \matnoise \XX^\dagger\empcov_{\xx\xx}^{\frac{1}{2}}} \rb{\empcov_{\xx\xx}^\frac{1}{2}}^\dagger. 
\end{align}
Where we used the property that $(c\DD)^\dagger=c^{-1}\DD^\dagger$ for some scaler $c$ and matrix $\DD$. For the more simplified case of overdetermined regime where $(\XX\XX^\top)$ is assumed to be invertible, we get the claimed result just by using the definition of $\XX^\dagger=\XX^\top(\XX\XX^\top)^{-1}$.
\end{proof}

\begin{lemma}\label{lemma:noise_upperbound}
    Let $\boldsymbol{\Psi} = \CovMat_{\vecnoise\vecnoise}^{-\frac{1}{2}} \matnoise$ be a whitened noise where $\matnoise \in \R^{d_y\times n}$ is the matrix containing $\vecnoise_i \sim \mathcal{N}( \zeroVec, \CovMat_{\vecnoise\vecnoise} ) $ as columns. Let $ \XX \in \R^{d_x \times n} $ be a matrix with rank $ d_x $. Define the projection matrix $\PP_{\XX} = \XX^\top (\XX\XX^\top)^{-1}\XX$, then
    \begin{equation}
        \E \sb{\norm{\boldsymbol{\Psi}\PP_{\XX} }^2 \,\middle|\, \XX } \leq 5 \max\{d_x,d_y\}.
    \end{equation}
\end{lemma}
\begin{proof}
{\bfseries Whitened noise and orthogonal projection.}
Assume the output noise vector $\vecnoise \in \R^{d_{y}}$ is Gaussian with covariance $\CovMat_{\vecnoise} \succ 0$. Define the whitened noise
\begin{equation}
    \boldsymbol{\psi} := \CovMat_{ \vecnoise \vecnoise }^{-\frac{1}{2}} \vecnoise,
    \qquad
    \E[\boldsymbol{\psi}]=\mathbf{0},
    \qquad
    \CovMat_{\boldsymbol{\psi}}=\II_{d_{y}}.
\end{equation}
Let $\boldsymbol{\Psi}\in\R^{d_{y}\times n}$ be the matrix whose $i$-th column is an independent copy of $\boldsymbol{\psi}$. Then the entries of $\boldsymbol{\Psi}$ are i.i.d.\ $\mathcal{N}(0,1)$. Let $\XX\in\R^{d_{x}\times n}$ be the data matrix and let $\PP_{\XX}\in\R^{n\times n}$ be the orthogonal projector onto $\operatorname{col}(\XX)$. Since we know that $\rank(\XX) \leq d_x$, there exists a matrix $\QQ\in\R^{n\times d_x}$ with orthonormal columns ($\QQ^\top\QQ=\II_{d_x}$) such that
\begin{equation}
    \PP_{\XX} = \QQ\QQ^\top .
\end{equation}

{\bfseries Reduction to a smaller Gaussian matrix.}
Write
\begin{equation}
\boldsymbol{\Psi}\PP_{\XX} \;=\; \boldsymbol{\Psi}\QQ\QQ^\top \;=\; \GG\,\QQ^\top,
\qquad \text{where} \qquad \GG := \boldsymbol{\Psi}\QQ \in \R^{d_{y}\times d_x}.
\end{equation}
Right multiplication by an orthogonal matrix preserves singular values, hence
\begin{equation}\label{eq:orth-invariance-sv}
    \norm{\boldsymbol{\Psi}\PP_{\XX}} = \norm{\GG\QQ^\top} = \norm{\GG}.
\end{equation}
Now we show that $\GG$ is Gaussian with i.i.d.\ $\mathcal{N}(0,1)$ entries. Let $\qq_1,\hdots,\qq_{d_x}$ denote the columns of $\QQ$. Then
\begin{equation}
\GG = \big[\boldsymbol{\Psi}\qq_1, \dots, \boldsymbol{\Psi}\qq_{d_x}\big].
\end{equation}
Fix $j\in\{1,\dots,d_x\}$. Since $\boldsymbol{\Psi}$ has i.i.d.\ $\mathcal{N}(0,1)$ entries, $\boldsymbol{\Psi}\qq_j$ is a linear transformation of a standard Gaussian vector and hence is Gaussian. Now we need to compute its mean and covariance. We take a unified approach by looking at the vectorization of $\GG$, denoted by $ \vectorization(\GG) $. First, note that $ \vectorization(\GG) = \vectorization(\boldsymbol{\Psi} \QQ) = (\QQ^{\top}\otimes \II) \, \vectorization(\boldsymbol{\Psi} ) $, where $\otimes$ is the Kronecker product. Then,
\begin{equation}
    \E[\vectorization(\GG) ] = \QQ^{\top}\otimes \II \, \E[\vectorization(\boldsymbol{\Psi} ) ]= \zeroVec,
\end{equation}
and
\begin{align}
    \Cov(\vectorization(\GG))
    & = \E[\vectorization(\GG)\vectorization(\GG)^{\top}] \nonumber \\
    & = \E\left[ \left( \QQ^{\top}\otimes \II \, \vectorization(\boldsymbol{\Psi} ) \right) \left( \QQ^{\top}\otimes \II \, \vectorization(\boldsymbol{\Psi} ) \right)^{\top} \right] \nonumber \\
    & = \E\left[  \QQ^{\top}\otimes \II \, \vectorization(\boldsymbol{\Psi} ) \vectorization(\boldsymbol{\Psi} )^\top \left( \QQ^{\top}\otimes \II \,  \right)^{\top} \right] \nonumber \\
    & =  \QQ^{\top}\otimes \II \, \E\left[ \vectorization(\boldsymbol{\Psi} ) \vectorization(\boldsymbol{\Psi} )^\top \right] \QQ \otimes \II.
\end{align}
Observe that $ \E[ \vectorization(\boldsymbol{\Psi} ) \vectorization(\boldsymbol{\Psi} )^\top ] = \II $, then
\begin{equation}
    \Cov(\vectorization(\GG)) =  \left( \QQ^{\top}\otimes \II \right) \left( \QQ \otimes \II \right)
    = \left( \QQ^{\top}\QQ \right) \otimes  \II = \II \otimes \II = \II .
\end{equation}
Equivalently, $\GG$ is a $d_{y}\times d_x$ standard Gaussian matrix.

{\bfseries Bounding the conditional second moment of the operator norm.}
By \eqref{eq:orth-invariance-sv},
\begin{equation}
    \E \left[\norm{\boldsymbol{\Psi}\PP_{\XX} }^2 \,\middle|\, \XX \right]
    = \E\left[\norm{\GG}^2\right].
\end{equation}
Since $\GG$ is a matrix with i.i.d.\ standard Gaussian entries, we can use Lemma~\ref{lemma:gaussian_matrix_bound} and we have
\begin{equation}
    \E\sb{\norm{\GG}^2} \leq 5\max\{d_x,d_y\}
\end{equation}
\end{proof}

\begin{lemma}\label{lemma:lowrank_decomposition}
Let $ \AA,\BB \in \R^{d_x \times d_y} $. Assume that $ \rank(\AA) \leq r $, then for any PSD matrix $ \CovMat_{\xx\xx}^{\frac{1}{2}} $
\begin{equation}\label{eq:three_terms}
    \fnorm{(\AA - \BB) \CovMat_{\xx\xx}^{\frac{1}{2}} }^2
    \leq 8r \rb{\norm{\rb{\AA-\BB}\CovMat^\frac{1}{2}_{\xx \xx}}^2 + \norm{\rb{\BB- \BB_r}\CovMat^\frac{1}{2}_{\xx \xx}}^2} + 2\fnorm{\rb{\BB-\BB_r} \CovMat^\frac{1}{2}_{\xx \xx}}^2,
\end{equation}
where $ \BB_r = \trunk{\BB} $.
\end{lemma}

\begin{proof}
Denote $ \BB_r = \trunk{\BB} $. We write: 
\begin{align}
    \fnorm{\rb{\AA - \BB}\CovMat^\frac{1}{2}_{\xx \xx}}^2
    &= \fnorm{\rb{\AA - \BB_r +\BB_r - \BB}\CovMat^\frac{1}{2}_{\xx \xx}}^2 \nonumber \\
    &\leq 2\fnorm{\rb{\AA - \BB_r}\CovMat^\frac{1}{2}_{\xx \xx}}^2 + 2\fnorm{\rb{\BB_r-\BB}\CovMat^\frac{1}{2}_{\xx \xx}}^2 . \label{eq:rank1}
\end{align}
Now we bound the first term: 
\begin{align}
    2\fnorm{\rb{\AA - \BB_r}\CovMat^\frac{1}{2}_{\xx \xx}}^2 
    &\leq 4r \norm{\rb{\AA - \BB_r}\CovMat^\frac{1}{2}_{\xx \xx}}^2\nonumber\\
    &= 4r \norm{\rb{\AA-\BB+\BB - \BB_r}\CovMat^\frac{1}{2}_{\xx \xx}}^2\nonumber\\
    &\leq 8r \norm{\rb{\AA-\BB}\CovMat^\frac{1}{2}_{\xx \xx}}^2 + 8r\norm{\rb{\BB - \BB_r}\CovMat^\frac{1}{2}_{\xx \xx}}^2.
\end{align}
By substituting this result into \eqref{eq:rank1} we have:
\begin{align}
    \fnorm{\rb{\AA - \BB}\CovMat^\frac{1}{2}_{\xx \xx}}^2
    \leq & 8r \rb{\norm{\rb{\AA-\BB}\CovMat^\frac{1}{2}_{\xx \xx}}^2 + \norm{\rb{\BB- \BB_r}\CovMat^\frac{1}{2}_{\xx \xx}}^2} \nonumber\\
    & 
    + 2\fnorm{\rb{\BB - \BB_r}\CovMat^\frac{1}{2}_{\xx \xx}}^2.
\end{align}
\end{proof}

\subsection{Overdetermined regime}\label{app:Overdetermined Regime}
\LoRAExcessOver*

\begin{proof}
Recall from Lemma \ref{lemma:general_risk} for any estimator $\AA$,
\begin{equation}
    \risk(\AA) - \risk (\AA^\star) = \fnorm{\rb{\AA-\AA^\star}\CovMat_{\xx\xx}^\frac{1}{2}}^2.
\end{equation}
Let $\AA=\lorasol=\initmodel+\widehat \bDelta_{\operatorname{LoRA}}$ where $\rank(\widehat \bDelta_{\operatorname{LoRA}})=r$. Then, 
\begin{align}\label{eq:app_lora_overdetermined_maineq}
    \risk(\lorasol) - \risk (\AA^\star) 
    = \fnorm{\rb{\AA-\AA^\star}\CovMat_{\xx\xx}^\frac{1}{2}}^2 &
    = \fnorm{ \rb{\initmodel+\widehat \bDelta_{\operatorname{LoRA}}-\AA^\star}\CovMat_{\xx\xx}^\frac{1}{2}}^2 \nonumber \\
    & = \fnorm{\rb{\widehat \bDelta_{\operatorname{LoRA}}-\bDelta^\star}\CovMat_{\xx\xx}^\frac{1}{2}}^2.
\end{align}
Using Lemma \ref{lemma:lowrank_decomposition} with $\bDelta^\star_r = \trunk{\bDelta^\star}$ we have, 
\begin{align}
    \fnorm{\rb{\widehat \bDelta_{\operatorname{LoRA}}-\bDelta^\star}\CovMat_{\xx\xx}^\frac{1}{2}}^2 
    & \leq 8r \rb{\norm{\rb{\widehat \bDelta_{\operatorname{LoRA}}-\bDelta^\star}\CovMat_{\xx\xx}^\frac{1}{2}}^2 + \norm{\rb{\bDelta^\star-\bDelta^\star_r}\CovMat_{\xx\xx}^\frac{1}{2}}^2} \nonumber \\
    & \hspace{3cm} + 2\fnorm{\rb{\bDelta^\star-\bDelta^\star_r}\CovMat_{\xx\xx}^\frac{1}{2}}^2.
\end{align}
Note that in the above inequality, the second and third terms are the same, just in different norms. We start by bounding the first term by replacing the explicit form of $\widehat \bDelta_{\operatorname{LoRA}}$ from Lemma \ref{lemma:lowrank_second_form} as follows: 
\begin{align}
    \norm{\rb{\widehat \bDelta_{\operatorname{LoRA}}-\bDelta^\star}\CovMat^\frac{1}{2}_{\xx \xx}}^2 
    & = \norm{\rb{\trunk{\bDelta^\star  \empcov_{\xx \xx}^{\frac{1}{2}} + \frac{1}{n}\matnoise\XX^\top\empcov_{\xx \xx}^{-\frac{1}{2}}} \empcov_{\xx \xx}^{-\frac{1}{2}}- \bDelta^\star}\CovMat^\frac{1}{2}_{\xx \xx}}^2\nonumber \\
    & = \norm{\rb{\trunk{\bDelta^\star  \empcov_{\xx \xx}^{\frac{1}{2}} + \frac{1}{n}\matnoise\XX^\top\empcov_{\xx \xx}^{-\frac{1}{2}}} - \bDelta^\star \empcov_{\xx \xx}^{\frac{1}{2}}}\empcov_{\xx \xx}^{-\frac{1}{2}}\CovMat^\frac{1}{2}_{\xx \xx}}^2\nonumber \\
    & \leq \norm{\trunk{\bDelta^\star  \empcov_{\xx \xx}^{\frac{1}{2}} + \frac{1}{n}\matnoise\XX^\top\empcov_{\xx \xx}^{-\frac{1}{2}}}- \bDelta^\star \empcov_{\xx \xx}^{\frac{1}{2}}}^2 \norm{\empcov_{\xx \xx}^{-\frac{1}{2}}\CovMat^\frac{1}{2}_{\xx \xx}}^2 \nonumber\\
    &\leq \rb{4 \norm{\frac{1}{n}\matnoise\XX^\top\empcov_{\xx \xx}^{-\frac{1}{2}}}^2  + 2\sigma^2_{r+1}\rb{\bDelta^\star  \empcov_{\xx \xx}^{\frac{1}{2}}}}\norm{\empcov_{\xx \xx}^{-\frac{1}{2}}\CovMat^\frac{1}{2}_{\xx \xx}}^2  \nonumber \\
    &= \rb{\frac{4}{n} \norm{ \frac{1}{\sqrt{n}} \matnoise\XX^\top\empcov_{\xx \xx}^{-\frac{1}{2}} }^2 + 2\sigma^2_{r+1}\rb{\bDelta^\star  \empcov_{\xx \xx}^{\frac{1}{2}}}}\norm{\empcov_{\xx \xx}^{-\frac{1}{2}}\CovMat^\frac{1}{2}_{\xx \xx}}^2.\label{eq:rank2}
\end{align}
Where in the fourth line we used Lemma \ref{lemma:perturbation_bound}. Note that
\begin{align}
    \norm{ \frac{1}{\sqrt{n}} \matnoise\XX^\top\empcov_{\xx \xx}^{-\frac{1}{2}} }^2
    & = \lambda_{\max} \left( \frac{1}{n} \matnoise\XX^\top\empcov_{\xx \xx}^{-1} \XX\matnoise^\top  \right) \nonumber \\
    & = \lambda_{\max} \left( \matnoise\XX^\top \left(\XX  \XX^\top\right)^{-1} \XX \matnoise^\top  \right) \nonumber \\
    & = \lambda_{\max} \left( \matnoise \PP_{\XX} \matnoise^\top  \right).
\end{align}
Here $ \PP_{\XX} = \XX^\top \left(\XX  \XX^\top\right)^{-1} \XX $ is an orthogonal projection matrix onto the row space of $\XX$. If we assume Gaussian weight noise, then we can give an exact expression which does not depend on $\XX$ (in any case, we should continue developing this term). Moreover,
\begin{equation}
    \norm{\empcov_{\xx \xx}^{-\frac{1}{2}}\CovMat^\frac{1}{2}_{\xx \xx}}^2 = \lambda_{\max} \left(\CovMat^\frac{1}{2}_{\xx \xx} \empcov_{\xx \xx}^{-1}\CovMat^\frac{1}{2}_{\xx \xx} \right) =\lambda_{\max}\left(\CovMat_{\xx \xx} \empcov_{\xx \xx}^{-1} \right),
\end{equation}
where in the second step, we use the fact that $ \CovMat^\frac{1}{2}_{\xx \xx} \empcov_{\xx \xx}^{-1}\CovMat^\frac{1}{2}_{\xx \xx} $ and $ \CovMat_{\xx \xx} \empcov_{\xx \xx}^{-1} $ are similar matrices. Overall,
\begin{equation}
    8r \norm{\rb{\widehat \bDelta_{\operatorname{LoRA}}-\bDelta^\star}\CovMat^\frac{1}{2}_{\xx \xx}}^2
    \leq \rb{32\frac{r}{n} \lambda_{\max} \rb{\matnoise \PP_{\XX} \matnoise^\top} + 16r \sigma^2_{r+1}\rb{\bDelta^\star \empcov^{\frac{1}{2}}_{\xx\xx}}}\lambda_{\max}\left(\CovMat_{\xx \xx} \empcov_{\xx \xx}^{-1} \right).
\end{equation}
Now we proceed with upper bounding the term $\lambda_{\max} \left( \matnoise \PP_{\XX} \matnoise^\top  \right)$.
\begin{align}
    \lambda_{\max} \left( \matnoise \PP_{\XX} \matnoise^\top  \right)
    = \lambda_{\max} \left( \matnoise \PP_{\XX} \PP_{\XX}^{\top} \matnoise^\top  \right)= \norm{\matnoise \PP_{\XX}}^2  = \norm{\CovMat_{\vecnoise\vecnoise}^{\frac{1}{2}} \CovMat_{\vecnoise\vecnoise}^{-\frac{1}{2}}  \matnoise \PP_{\XX} }^2 
     = \lambda_{\max}\left( \CovMat_{\vecnoise\vecnoise} \right) \norm{ \CovMat_{\vecnoise\vecnoise}^{-\frac{1}{2}}  \matnoise \PP_{\XX} }^2.
\end{align}
Note that $ \boldsymbol{\Psi} =  \CovMat_{\vecnoise\vecnoise}^{-\frac{1}{2}}  \matnoise $ is white noise ($\CovMat_{\boldsymbol{\psi}} = \II$). Now we proceed by taking expectation with respect to the randomness of the training dataset $\XX,\matnoise$:
\begin{align}\label{eq:rank3}
    & \E\sb{8r \norm{\rb{\widehat \bDelta_{\operatorname{LoRA}}-\bDelta^\star}\CovMat^\frac{1}{2}_{\xx \xx}}^2 } \\
    & \quad \leq 32\frac{r}{n}\E\sb{ \lambda_{\max} \rb{\matnoise \PP_{\XX} \matnoise^\top} \lambda_{\max}\rb{\CovMat_{\xx \xx} \empcov_{\xx \xx}^{-1}}}+ 16r \E\sb{\sigma^2_{r+1}\rb{\AA^\star \empcov^{\frac{1}{2}}_{\xx\xx}}\lambda_{\max}\rb{\CovMat_{\xx \xx} \empcov_{\xx \xx}^{-1}}} \nonumber \\
    & \quad\leq 32\frac{r}{n}\lambda_{\max}\left( \CovMat_{\vecnoise\vecnoise} \right) \E\sb{ \| \boldsymbol{\Psi} \PP_{\XX}  \|^2 \lambda_{\max}\left(\CovMat_{\xx \xx} \empcov_{\xx \xx}^{-1} \right)}+ 16r \E\sb{\sigma^2_{r+1}\rb{\AA^\star \empcov^{\frac{1}{2}}_{\xx\xx}}\lambda_{\max}\rb{\CovMat_{\xx \xx} \empcov_{\xx \xx}^{-1}}} \nonumber \\
    & \quad = 32\frac{r}{n}\lambda_{\max}\left( \CovMat_{\vecnoise\vecnoise} \right) \E\sb{ \E\sb{\norm{\boldsymbol{\Psi} \PP_{\XX}}^2 \middle| \XX}\lambda_{\max}\left(\CovMat_{\xx \xx} \empcov_{\xx \xx}^{-1} \right)} + 16r \E\sb{\sigma^2_{r+1}\rb{\AA^\star \empcov^{\frac{1}{2}}_{\xx\xx}}\lambda_{\max}\rb{\CovMat_{\xx \xx} \empcov_{\xx \xx}^{-1}}}. \nonumber 
\end{align}
Now we need to carefully bound the term $\E\sb{\norm{\boldsymbol{\Psi} \PP_{\XX}}^2 \middle| \XX}$. Using Lemma \ref{lemma:noise_upperbound} we have, 
\begin{equation}
    \E\sb{\norm{\boldsymbol{\Psi}\PP_{\XX} }^2 \,\middle|\, \XX } \leq 5 \max\{d_x,d_y\}.
\end{equation}
Putting back these results into \eqref{eq:rank3} gives us:
\begin{align}
    \E\sb{8r \norm{\rb{\widehat \bDelta_{\operatorname{LoRA}}-\bDelta^\star}\CovMat^\frac{1}{2}_{\xx \xx}}^2 } 
    & \leq 160\frac{r\cdot\max\{d_x,d_y\}}{n}\E\sb{\lambda_{\max}\rb{\CovMat_{\xx \xx} \empcov_{\xx \xx}^{-1}}}\lambda_{\max}\left( \CovMat_{\vecnoise\vecnoise} \right) \nonumber \\
    & \quad + 16r \E\sb{\sigma^2_{r+1}\rb{\bDelta^\star \empcov^{\frac{1}{2}}_{\xx\xx}}\lambda_{\max}\rb{\CovMat_{\xx \xx} \empcov_{\xx \xx}^{-1}}}. 
\end{align}
Now it remains to bound the term $\norm{\rb{\bDelta^\star-\bDelta^\star_r}\CovMat^\frac{1}{2}_{\xx \xx}}^2$ and $\fnorm{\rb{\bDelta^\star-\bDelta^\star_r}\CovMat^\frac{1}{2}_{\xx \xx}}^2$ from~\eqref{eq:three_terms}. Recall that $\bDelta^\star_r := \trunk{\bDelta^\star}$.
\begin{align}
    \norm{\rb{\bDelta^\star-\bDelta^\star_r}\CovMat^\frac{1}{2}_{\xx \xx}}^2 
    & = \norm{\rb{\bDelta^\star-\trunk{\bDelta^\star}}\CovMat_{\xx \xx}^{\frac{1}{2}}}^2 \nonumber \\
    & \leq \norm{\bDelta^\star-\trunk{\bDelta^\star}}^2 \lambda_{\max} \rb{\CovMat_{\xx \xx}} \nonumber \\
    & =  \sigma^2_{r+1}\rb{\bDelta^\star}\lambda_{\max} \rb{\CovMat_{\xx \xx}}.
\end{align}
Where in the last line we used Lemma \ref{lemma:young_rank}. Next, we have:
\begin{align}
    \fnorm{\rb{\bDelta^\star- \bDelta^\star_r}\CovMat^\frac{1}{2}_{\xx \xx}}^2 &= \fnorm{\rb{\bDelta^\star-\trunk{\bDelta^\star}}\CovMat^\frac{1}{2}_{\xx \xx}}^2 \leq \lambda_{\max}\rb{\CovMat_{\xx \xx}} \fnorm{\bDelta^\star-\trunk{\bDelta^\star}}^2 \nonumber \\
    & =\lambda_{\max}\rb{\CovMat_{\xx \xx}}\rb{\sum_{i=r+1}^{\min\{d_x,d_y\}} \sigma^2_{i}\rb{\bDelta^\star}}.
\end{align}
Putting back all results into \eqref{eq:app_lora_overdetermined_maineq} gives us:
\begin{align}
    & \E\sb{\risk(\lorasol)-\risk(\AA^\star)} \nonumber \\
    & \leq 160\frac{r\max\{d_x,d_y\}}{n}\E\sb{\lambda_{\max}\rb{\CovMat_{\xx \xx} \empcov_{\xx \xx}^{-1}}}\lambda_{\max}\rb{\CovMat_{\vecnoise\vecnoise}}+ 8r\sigma^2_{r+1}\rb{\bDelta^\star}\lambda_{\max} \rb{\CovMat_{\xx \xx}} \nonumber \\
    & \quad +2\lambda_{\max}\rb{\CovMat_{\xx \xx}}\rb{\sum_{i=r+1}^{\min\{d_x,d_y\}} \sigma^2_{i}\rb{\bDelta^\star}} + 16r \E\sb{\sigma^2_{r+1}\rb{\bDelta^\star \empcov^{\frac{1}{2}}_{\xx\xx}}\lambda_{\max}\rb{\CovMat_{\xx \xx} \empcov_{\xx \xx}^{-1}}}\nonumber\\ \nonumber\\
    &\leq 160\frac{r\max\{d_x,d_y\}}{n}\E\sb{\lambda_{\max}\rb{\CovMat_{\xx \xx} \empcov_{\xx \xx}^{-1}}}\lambda_{\max}\rb{\CovMat_{\vecnoise\vecnoise}} + 8r\sigma^2_{r+1}\rb{\bDelta^\star}\lambda_{\max} \rb{\CovMat_{\xx \xx}} \nonumber \\
    & \quad +2\lambda_{\max}\rb{\CovMat_{\xx \xx}}\rb{\sum_{i=r+1}^{\min\{d_x,d_y\}} \sigma^2_{i}\rb{\bDelta^\star}} + 16r\sigma^2_{r+1}\rb{\bDelta^\star} \E\sb{ \lambda_{\max}\rb{\empcov_{\xx\xx}}\lambda_{\max}\rb{\CovMat_{\xx \xx} \empcov_{\xx \xx}^{-1}}} \nonumber\\ \nonumber \\
    &= 160\frac{r\max\{d_x,d_y\}}{n}\E\sb{\lambda_{\max}\rb{\CovMat_{\xx \xx} \empcov_{\xx \xx}^{-1}}}\lambda_{\max}\rb{\CovMat_{\vecnoise\vecnoise}} \nonumber \\
    &\quad + 8r\sigma^2_{r+1}\rb{\bDelta^\star}\rb{\lambda_{\max}\rb{\CovMat_{\xx \xx}}+2 \E\sb{ \lambda_{\max}\rb{\empcov_{\xx\xx}}\lambda_{\max}\rb{\CovMat_{\xx \xx} \empcov_{\xx \xx}^{-1}}}}\nonumber \\
    & \quad +2\lambda_{\max}\rb{\CovMat_{\xx \xx}}\rb{\sum_{i=r+1}^{\min\{d_x,d_y\}} \sigma^2_{i}\rb{\bDelta^\star}}.
\end{align}
Where in the second line we used the fact that $\sigma_{i}\rb{\AA\BB} \leq \sigma_i(\AA)\lambda_{\max}(\BB)$ for some matrix $\AA$ and a PSD matrix $\BB$.
\end{proof}

\LoRAvarGaussOver*

\begin{proof}
From the variance term in \eqref{eq:theorem:lora_excessrisk_overdetermined}, we need to control
\begin{equation}
    \E\left[ \lambda_{\max}\left( \CovMat_{\xx\xx}\empcov_{\xx\xx}^{-1} \right) \right].
\end{equation}
Since $\xx_i \sim \mathcal{N}(\zeroVec,\CovMat_{\xx\xx})$, we may write $\xx_i = \CovMat_{\xx\xx}^{\frac{1}{2}}\zz_i$ with $\zz_i \sim \mathcal{N}(\zeroVec,\II)$.
Then
\begin{equation}
    \empcov_{\xx\xx} = \frac{1}{n}\sum_{i=1}^n \xx_i \xx_i^\top 
    = \CovMat_{\xx\xx}^{\frac{1}{2}} \left(\frac{1}{n}\sum_{i=1}^n \zz_i \zz_i^\top\right) \CovMat_{\xx\xx}^{\frac{1}{2}},
\end{equation}
and hence
\begin{equation}
    \CovMat_{\xx\xx} \empcov_{\xx\xx}^{-1}
    = \CovMat_{\xx\xx}^{\frac{1}{2}}
    \left(\frac{1}{n}\sum_{i=1}^n \zz_i \zz_i^\top\right)^{-1}
    \CovMat_{\xx\xx}^{-\frac{1}{2}}.
\end{equation}
Therefore,
\begin{equation}
    \lambda_{\max}\left( \CovMat_{\xx\xx} \empcov_{\xx\xx}^{-1} \right)
    = \lambda_{\max} \left( \left(\frac{1}{n}\sum_{i=1}^n \zz_i \zz_i^\top\right)^{-1} \right)
    = \frac{1}{\lambda_{\min} \left(\frac{1}{n}\sum_{i=1}^n \zz_i \zz_i^\top\right)}.
\end{equation}
Let $\ZZ=[\zz_1,\dots,\zz_n] \in \R^{d_x\times n}$. Then $\frac{1}{n}\ZZ\ZZ^\top$ is a Wishart matrix. From \citep{silverstein1985smallest}, when $n>d_x+1$, the smallest eigenvalue satisfies
\begin{equation}
    \lambda_{\min}\left(\frac{1}{n}\ZZ\ZZ^\top\right)
    \xrightarrow{\text{a.s.}}
    \bigg(1-\sqrt{c}\bigg)^2
    \quad\text{as } d_x,n\to\infty \text{ with } d_x/n \to c<1.
\end{equation}
Consequently,
\begin{equation}
    \E\left[ \lambda_{\max}\left( \CovMat_{\xx\xx}\empcov_{\xx\xx}^{-1} \right) \right]
    \asymp
    \bigg(1-\sqrt{\frac{d_x}{n}}\bigg)^{-2}.
\end{equation}
Substituting this into the variance term yields
\begin{equation}
    \frac{r\max\{d_x,d_y\}}{n}
    \lambda_{\max}\big( \CovMat_{\vecnoise\vecnoise} \big)
    \E\left[ \lambda_{\max}\left( \CovMat_{\xx\xx}\empcov_{\xx\xx}^{-1} \right) \right]  \asymp r\max\{d_x,d_y\}
    \big(\sqrt{n}-\sqrt{d_x}\big)^{-2} \lambda_{\max}\big(\CovMat_{\vecnoise\vecnoise}\big),
\end{equation}
which proves the claim.
\end{proof}

\subsection{Underdetermined regime}\label{subsection:app_lora_under}

\LoRAExcessUnder*

\begin{proof}
Recall from Lemma \ref{lemma:general_risk} for any estimator $\AA$,
\begin{equation}
    \risk(\AA) - \risk (\AA^\star) = \trace\rb{\rb{\AA-\AA^\star}\rb{\AA-\AA^\star}\CovMat_{\xx\xx}} = \fnorm{\rb{\AA-\AA^\star}\CovMat_{\xx\xx}^\frac{1}{2}}^2.
\end{equation}
Let $\AA=\lorasol=\initmodel+\widehat \bDelta_{\operatorname{LoRA}}$ where $\rank(\widehat \bDelta_{\operatorname{LoRA}})=r$. Then, 
\begin{align}
    \risk(\lorasol) - \risk (\AA^\star)=
    \fnorm{\rb{\AA-\AA^\star}\CovMat_{\xx\xx}^\frac{1}{2}}^2 & = \fnorm{\rb{\initmodel+\widehat \bDelta_{\operatorname{LoRA}}-\AA^\star}\CovMat_{\xx\xx}^\frac{1}{2}}^2 \nonumber \\
    & = \fnorm{\rb{\widehat \bDelta_{\operatorname{LoRA}}-\bDelta^\star}\CovMat_{\xx\xx}^\frac{1}{2}}^2.
\end{align}
We write,
\begin{align}
    \fnorm{\rb{\widehat \bDelta_{\operatorname{LoRA}}-\bDelta^\star}\CovMat_{\xx\xx}^\frac{1}{2}}^2 &= \fnorm{\rb{\widehat \bDelta_{\operatorname{LoRA}} - \bDelta^\star\PP_{\XX} + \bDelta^\star\PP_{\XX} - \bDelta^\star}\CovMat_{\xx\xx}^\frac{1}{2}}^2 \nonumber \\
    &\leq 2\fnorm{\rb{\widehat \bDelta_{\operatorname{LoRA}} - \bDelta^\star\PP_{\XX}}\CovMat_{\xx\xx}^\frac{1}{2}}^2
    +  2\fnorm{\rb{\bDelta^\star\PP_{\XX} - \bDelta^\star}\CovMat_{\xx\xx}^\frac{1}{2}}^2 \nonumber \\
    &= 2\fnorm{\rb{\widehat \bDelta_{\operatorname{LoRA}} - \bDelta^\star\PP_{\XX}}\CovMat_{\xx\xx}^\frac{1}{2}}^2
    +  2\fnorm{\bDelta^\star\PP_{\XX}^{\perp}\CovMat_{\xx\xx}^\frac{1}{2}}^2.
\end{align}
We now try to bound the first term. Using Lemma \ref{lemma:lowrank_decomposition} with $\AA=\widehat \bDelta_{\operatorname{LoRA}}$ and $\BB =\bDelta^\star\PP_{\XX}$ where we know $\rank(\widehat \bDelta_{\operatorname{LoRA}})=r$ we have, 
\begin{align}
    & \fnorm{{\rb{\widehat \bDelta_{\operatorname{LoRA}}-\bDelta^\star\PP_{\XX}}\CovMat_{\xx\xx}^\frac{1}{2}}}^2
    \nonumber \\
    & \quad \leq 8r\rb{\underbrace{\norm{\rb{\widehat \bDelta_{\operatorname{LoRA}}-\bDelta^\star\PP_{\XX}}\CovMat_{\xx\xx}^\frac{1}{2}}^2}_{\circled{1}} + \underbrace{\norm{\rb{\bDelta^\star\PP_{\XX}-\trunk{\bDelta^\star\PP_{\XX}}}\CovMat_{\xx\xx}^\frac{1}{2}}^2}_{\circled{2}}} \nonumber \\
    & \qquad + 2\underbrace{\fnorm{\rb{\bDelta^\star\PP_{\XX}-\trunk{\bDelta^\star\PP_{\XX}}}\CovMat_{\xx\xx}^\frac{1}{2}}^2}_{\circled{3}}.
\end{align}
We start by bounding the term $\circled{1}$,
\begin{equation}
    \circled{1} = \norm{\rb{\trunk{\bDelta^\star \empcov_{\xx\xx}^{\frac{1}{2}} + \matnoise \XX^\dagger\empcov_{\xx\xx}^{\frac{1}{2}}} \rb{\empcov_{\xx\xx}^\frac{1}{2}}^\dagger-\bDelta^\star\PP_{\XX}}\CovMat^\frac{1}{2}_{\xx \xx}}^2.
\end{equation}
Note that $\PP_{\XX} =  \empcov_{\xx\xx}^\frac{1}{2}\rb{\empcov_{\xx\xx}^\frac{1}{2}}^\dagger$, therefore
\begin{align}
    \circled{1} 
    & = \norm{\rb{\trunk{\bDelta^\star \empcov_{\xx\xx}^\frac{1}{2} + \matnoise \XX^\dagger \empcov_{\xx\xx}^\frac{1}{2}}-\bDelta^\star\empcov_{\xx\xx}^\frac{1}{2}}\rb{\empcov_{\xx\xx}^\frac{1}{2}}^\dagger\CovMat^\frac{1}{2}_{\xx \xx}}^2 \nonumber \\
    & \leq \rb{4\norm{\matnoise \XX^\dagger \empcov_{\xx\xx}^\frac{1}{2}}^2 + 2\sigma^2_{r+1}\rb{\bDelta^\star \empcov_{\xx\xx}^\frac{1}{2}}}\norm{\rb{\empcov_{\xx\xx}^\frac{1}{2}}^\dagger\CovMat^\frac{1}{2}_{\xx \xx}}^2 \nonumber \\
    & = \rb{4\norm{\matnoise \XX^\dagger \empcov_{\xx\xx}^\frac{1}{2}}^2 + 2\sigma^2_{r+1}\rb{\bDelta^\star \empcov_{\xx\xx}^\frac{1}{2}}}\lambda_{\max}\rb{\CovMat_{\xx \xx}\empcov_{\xx\xx}^\dagger}.
\end{align}
Where in the last line we used Lemma \ref{lemma:perturbation_bound} and the fact that $\norm{\AA\BB} \leq \norm{\AA}\norm{\BB}$. Next, 
\begin{equation}
    \norm{\rb{\empcov_{\xx\xx}^\frac{1}{2}}^\dagger\CovMat^\frac{1}{2}_{\xx \xx}}^2 = \lambda_{\max}\rb{\CovMat^\frac{1}{2}_{\xx \xx}\empcov_{\xx\xx}^\dagger \CovMat^\frac{1}{2}_{\xx \xx}} = \lambda_{\max}\rb{\CovMat_{\xx \xx}\empcov_{\xx\xx}^\dagger}.
\end{equation}
This holds since $\lambda_{\max}\rb{\CC\DD}=\lambda_{\max}\rb{\DD\CC}$ for $\CC=\CovMat^\frac{1}{2}_{\xx \xx}\empcov_{\xx\xx}^\dagger$ and $\DD=\CovMat^\frac{1}{2}_{\xx \xx}$. We  proceed with the upper bound of $\circled{1}$,
\begin{align}
    \norm{\matnoise \XX^\dagger \empcov_{\xx\xx}^\frac{1}{2}}^2 &= \lambda_{\max}\rb{\matnoise \XX^\dagger \empcov_{\xx\xx}^\frac{1}{2} \empcov_{\xx\xx}^\frac{1}{2} \rb{\XX^\dagger}^\top \matnoise^\top} \nonumber \\
    &= \lambda_{\max}\rb{\frac{1}{n}\matnoise \rb{\XX^\top \XX}^{-1}\XX^\top \rb{\XX\XX^\top} \XX \rb{\XX^\top\XX}^{-1}\matnoise^\top}\nonumber \\
    &= \lambda_{\max}\rb{\frac{1}{n}\matnoise\matnoise^\top} = \lambda_{\max}\rb{\empcov_{\vecnoise\vecnoise}}.
\end{align}
Where $\empcov_{\vecnoise\vecnoise}=\frac{1}{n}\matnoise \matnoise^\top$ is the empirical covariance matrix of the noise. Putting back everything together we can bound $\circled{1}$ as follows, 
\begin{align}
    \circled{1} &\leq 32r\rb{2\lambda_{\max}\rb{\empcov_{\vecnoise\vecnoise}} + \sigma^2_{r+1}\rb{\bDelta^\star \empcov_{\xx\xx}^{\frac{1}{2}}}}\lambda_{\max}\rb{\CovMat_{\xx \xx}\empcov_{\xx\xx}^\dagger}.
\end{align}
For the terms $\circled{2}, \circled{3}$ we use Lemma \ref{lemma:young_rank} and we have,
\begin{align}
    & \circled{2} = \norm{\rb{\bDelta^\star\PP_{\XX}-\trunk{\bDelta^\star\PP_{\XX}}}\CovMat_{\xx\xx}^\frac{1}{2}}^2 \leq \lambda_{\max}\rb{\CovMat_{\xx\xx}}\sigma_{r+1}^2\rb{\bDelta^\star}, \\
    & \circled{3} = \fnorm{\rb{\bDelta^\star\PP_{\XX}-\trunk{\bDelta^\star\PP_{\XX}}}\CovMat_{\xx\xx}^\frac{1}{2}}^2
    \leq \lambda_{\max}\rb{\CovMat_{\xx\xx}}\sum_{i=r+1}^{\min\{d_x,d_y\}}\sigma_{i}^2\rb{\bDelta^\star\PP_{\XX}} \nonumber \\
    & \hspace{6.4cm}\leq  \lambda_{\max}\rb{\CovMat_{\xx\xx}}\sum_{i=r+1}^{n}\sigma_{i}^2\rb{\bDelta^\star}.
\end{align}
In the last line, we used the fact that $\bDelta^\star\PP_{\XX}$ projects $\bDelta^\star$ onto the column space of $\XX$, which is an $n$-dimensional subspace since, in the underdetermined regime, $\XX$ has full column rank. Now we take the conditional expectation with respect to the randomness of noise while fixing the data, and get the bound
\begin{align}
    \E\sb{\risk(\lorasol) - \risk (\AA^\star) \Big| \XX}
    & \leq 2\fnorm{\bDelta^\star\PP_{\XX}^{\perp}\CovMat_{\xx\xx}^\frac{1}{2}}^2 \nonumber\\
    & \quad + 32r\rb{2 \E\sb{\lambda_{\max}\rb{\empcov_{\vecnoise\vecnoise}}\Big|\XX} + \sigma^2_{r+1}\rb{\bDelta^\star \empcov_{\xx\xx}^{\frac{1}{2}}}}\lambda_{\max}\rb{\CovMat_{\xx \xx}\empcov_{\xx\xx}^\dagger} \nonumber \\
    & \quad  + 8r\lambda_{\max}\rb{\CovMat_{\xx\xx}}\sigma_{r+1}^2\rb{\bDelta^\star} + 2\lambda_{\max}\rb{\CovMat_{\xx\xx}}\sum_{i=r+1}^{n}\sigma_{i}^2\rb{\bDelta^\star}.
\end{align}
In order to bound the term with the empirical covariance of the noise, we first whiten the noise and use Lemma \ref{lemma:gaussian_matrix_bound}.
\begin{align}
    \E\sb{\lambda_{\max}\rb{\empcov_{\vecnoise\vecnoise}}\Big|\XX} &= \frac{1}{n}\E\sb{\norm{\matnoise}^2\Big|\XX} = \frac{1}{n}\E\sb{\norm{\CovMat_{\vecnoise\vecnoise}^{\frac{1}{2}}\CovMat_{\vecnoise\vecnoise}^{-\frac{1}{2}}\matnoise}^2\Big|\XX} \nonumber \\
    &\leq \lambda_{\max}\rb{\CovMat_{\vecnoise\vecnoise}}\frac{1}{n}\E\sb{\norm{\CovMat_{\vecnoise\vecnoise}^{-\frac{1}{2}}\matnoise}^2\Big|\XX}  \nonumber \\
    &\leq \lambda_{\max}\rb{\CovMat_{\vecnoise\vecnoise}} \frac{5\max\{d_y,n\}}{n} .
\end{align}
Where in the last line we used the fact that $\CovMat_{\vecnoise\vecnoise}^{-\frac{1}{2}}\matnoise$ is a standard gaussian matrix. We also have,
\begin{equation}
    \sigma_{r+1}^2\rb{\bDelta^\star \empcov_{\xx\xx}^{\frac{1}{2}}}
    \le \sigma_{r+1}^2\rb{\bDelta^\star}\,\lambda_{\max}\rb{\empcov_{\xx\xx}}.
\end{equation}
Finally, take full expectation over the data and the noise, 
\begin{align}
    \E\sb{\risk(\lorasol) - \risk (\AA^\star)} 
    & 2\fnorm{\bDelta^\star\PP_{\XX}^{\perp}\CovMat_{\xx\xx}^\frac{1}{2}}^2  \leq 320 \frac{r\max\{d_y,n\}}{n} \lambda_{\max}(\CovMat_{\vecnoise\vecnoise})
    \E\sb{\lambda_{\max} \rb{\CovMat_{\xx\xx}\empcov_{\xx\xx}^\dagger}} \nonumber \\
    & \quad + 32 r \sigma_{r+1}^2\rb{\bDelta^\star}
    \E\sb{\lambda_{\max}\rb{\empcov_{\xx\xx}}
    \lambda_{\max}\rb{\CovMat_{\xx\xx} \empcov_{\xx\xx}^\dagger}} 
    \nonumber \\
    & \quad + 8r\lambda_{\max}\rb{\CovMat_{\xx\xx}}\sigma_{r+1}^2\rb{\bDelta^\star}  + 2\lambda_{\max}\rb{\CovMat_{\xx\xx}}\sum_{i=r+1}^{n}\sigma_{i}^2\rb{\bDelta^\star} \nonumber \\
    & = 2\fnorm{\bDelta^\star\PP_{\XX}^{\perp}\CovMat_{\xx\xx}^\frac{1}{2}}^2  + 320 \frac{r\max\{d_y,n\}}{n}\,\lambda_{\max}(\CovMat_{\vecnoise\vecnoise})\E\sb{\lambda_{\max}\rb{\CovMat_{\xx\xx}\empcov_{\xx\xx}^\dagger}} \nonumber \\
    & \quad + 32r\sigma_{r+1}^2\rb{\bDelta^\star}\E\sb{\lambda_{\max}\rb{\empcov_{\xx\xx}}\lambda_{\max}\rb{\CovMat_{\xx\xx}\empcov_{\xx\xx}^\dagger}} \nonumber \\
    &\quad +2\rb{ 4r\sigma_{r+1}^2\rb{\bDelta^\star} + \sum_{i=r+1}^{n}\sigma_{i}^2\rb{\bDelta^\star}}\lambda_{\max}\rb{\CovMat_{\xx \xx} }. 
\end{align}
\end{proof}

\LoRAvarGaussUnder*

\begin{proof}
From the variance term in \eqref{eq:LoRA_excessrisk_underdetermined}, we need to control
\begin{equation}
    \E\left[\lambda_{\max}\left(\CovMat_{\xx\xx}\empcov_{\xx\xx}^{\dagger}\right)\right].
\end{equation}

Since $\xx_i \sim \mathcal{N}(\zeroVec,\sigma_{\xx}^2\II)$, we have $\CovMat_{\xx\xx}=\sigma_{\xx}^2\II$ and
\begin{equation}
    \empcov_{\xx\xx}=\frac{1}{n}\XX\XX^\top,
    \qquad\text{where} \qquad \XX=[\xx_1,\ldots,\xx_n]\in\R^{d_x\times n}.
\end{equation}
Hence,
\begin{equation}
    \CovMat_{\xx\xx}\empcov_{\xx\xx}^{\dagger}
    = \sigma_{\xx}^2 \left(\frac{1}{n}\XX\XX^\top\right)^{\dagger}.
\end{equation}
Note that
\begin{equation}
    \lambda_{\max}\left(\CovMat_{\xx\xx}\empcov_{\xx\xx}^{\dagger}\right)
    = \sigma_{\xx}^2\,\lambda_{\max}\left(\left(\frac{1}{n}\XX\XX^\top\right)^{\dagger}\right)
    = \sigma_{\xx}^2\,\frac{n}{\lambda_{\min}(\XX^\top\XX)}.
\end{equation}
Let $\ZZ=\frac{1}{\sigma_{\xx}}\XX$, so that $\ZZ$ has i.i.d.\ standard Gaussian entries.
Then $\XX\XX^\top=\sigma_{\xx}^2\ZZ\ZZ^\top$ and
\begin{equation}  
    \lambda_{\max}\left(\CovMat_{\xx\xx}\empcov_{\xx\xx}^{\dagger}\right)
= \frac{n}{\lambda_{\min}(\ZZ^\top\ZZ)}.
\end{equation}
By classical results on the smallest nonzero eigenvalue of Wishart matrices (see, e.g., \citet{silverstein1985smallest}), when $n<d_x$ and $d_x,n\to\infty$ with $n/d_x\to c<1$,
\begin{equation}
    \frac{1}{d_x} \lambda_{\min}\left( \ZZ^\top\ZZ \right) =
    \lambda_{\min}\left( \frac{1}{d_x} \ZZ^\top\ZZ \right)
    \xrightarrow[]{\text{a.s.}}
    \big(1-\sqrt{c}\big)^2.
\end{equation}
Therefore,
\begin{equation}
    \E\left[\lambda_{\max}\left(\CovMat_{\xx\xx}\empcov_{\xx\xx}^{\dagger}\right)\right]
    \asymp
    \frac{n}{\big(\sqrt{d_x}-\sqrt{n}\big)^2}.
\end{equation}
Substituting this into the variance term yields
\begin{equation}
    \frac{r\max\{n,d_y\}}{n} \lambda_{\max} \big( \CovMat_{\vecnoise\vecnoise} \big)
    \E\sb{\lambda_{\max}\big(\CovMat_{\xx \xx} \empcov_{\xx \xx}^{\dagger} \big) }
    \asymp
    r \max\{n,d_y\} \big( \sqrt{d_x} - \sqrt{n}\big)^{-2} \lambda_{\max}\big(\CovMat_{\vecnoise\vecnoise}\big),
\end{equation}
which completes the proof.
\end{proof}

\subsection{Inevitable error of low-rank adaptations}\label{app:Inevitability error}
\begin{proposition}
   Let $ \empsol \in \R^{d_y \times d_x} $ be a low rank adaptation of $ \initmodel $, i.e., taking the form $ \empsol = \initmodel + \widehat{\bDelta} $ with $ \rank( \widehat{\bDelta} ) \leq r $.  Then
   \begin{equation}
       \E\big[\risk\big(\empsol\big)\big] - \risk\big(\AA^\star\big) \geq \lambda_{\min}(\CovMat_{\rvx\rvx}) \sum_{i=r+1}^{\min\{d_x,d_y\}} \sigma^2_{i} \big(\bDelta^\star\big)
   \end{equation}
   \begin{proof}
       From Lemma~\ref{lemma:general_risk} we have
       \begin{align}
           \E\big[\risk\big(\empsol\big)\big] - \risk \big( \AA^\star \big)
           & = \fnorm{\big(\empsol -\AA^{\star} \big)\CovMat_{\xx \xx}^{\frac{1}{2}}}^2 \nonumber \\
           & = \fnorm{\big(\widehat{\bDelta} -\bDelta^{\star} \big)\CovMat_{\xx \xx}^{\frac{1}{2}}}^2 \nonumber \\
           & \geq \lambda_{\min} \left( \CovMat_{\xx \xx} \right) \fnorm{\widehat{\bDelta} -\bDelta^{\star} }^2 .
       \end{align}
       Since $ \rank(\widehat{\bDelta} ) \leq r $ we have
       \begin{equation}
           \fnorm{\widehat{\bDelta} -\bDelta^{\star} }^2 
           \geq \min_{ \rank(\bDelta) \leq r } \fnorm{\bDelta -\bDelta^{\star} }^2 = \sum_{i=r+1}^{\min\{d_x,d_y\}} \sigma^2_{i} \big(\bDelta^\star\big),
       \end{equation}
       where in the last step we used Eckart-Young theorem (Lemma~\ref{lemma:young_rank}).
   \end{proof}
\end{proposition}

\clearpage
\section{Experiments}\label{app:Experiments}

\subsection{Models and tasks}
We use the Qwen2.5 family at three scales $\{0.5,1.5,3\}$B. For evaluation, we use two datasets: BoolQ and CommonsenseQA. All evaluation is zero-shot via
\texttt{lm-evaluation-harness} \footnote{https://github.com/EleutherAI/lm-evaluation-harness}.

\subsection{Input/output format}
Training data is in Alpaca form
$\{\texttt{instruction},\texttt{input},\texttt{output}\}$.
The loss is cross-entropy on the answer tokens only; the prompt span is
masked with $-100$.
For BoolQ:
\begin{quote}\small
\texttt{instruction:} ``Answer the following yes/no question based on the passage.''\\
\texttt{input:} ``Passage: \dots\textbackslash n\textbackslash nQuestion: \dots''\\
\texttt{output:} ``Yes'' \textbar{} ``No''
\end{quote}
For CSQA:
\begin{quote}\small
\texttt{instruction:} ``Answer the following multiple-choice question. Choose the correct option letter only (A, B, C, D, or E).''\\
\texttt{input:} ``Question: \dots\textbackslash n\textbackslash nOptions:\textbackslash nA. \dots\textbackslash nB. \dots\textbackslash nE. \dots''\\
\texttt{output:} ``A''\,\textbar{}\,\dots\,\textbar{}\,``E''
\end{quote}

\subsection{Training hyperparameters common to all runs}
All runs use bf16 precision with SDPA attention and gradient checkpointing, a sequence-length cutoff of 512, and train for 3 epochs with an effective batch size of 32. We use a cosine learning-rate schedule with linear warmup and clip gradients at $\|g\|_2\le 1$. Weight decay is $0.01$ on all parameters except biases and LayerNorm/RMSNorm scales, for which it is set to $0$. Per-device batch sizes and gradient-accumulation counts depend on model size and are reported in Tab.~\ref{tab:bs}. Each job runs on a single A100 GPU. We hold out $20\%$ of the training data for validation ($\texttt{val\_size}=0.2$); the $n$ reported throughout refers to the training-set size \emph{before} this split.

\begin{table}[ht]
\centering\small
\caption{Per-device batch ($\mathrm{bs}_{\text{dev}}$) $\times$ gradient
accumulation ($\mathrm{ga}$). Effective batch size is always $32$.}
\label{tab:bs}
\begin{tabular}{lccc}
\toprule
 & 0.5B & 1.5B & 3B \\
\midrule
Full FT (AdamW or SGD)   & $8\!\times\!4$ & $2\!\times\!16$ & $1\!\times\!32$ \\
LoRA (any rank)          & $8\!\times\!4$ & $8\!\times\!4$  & $4\!\times\!8$ \\
\bottomrule
\end{tabular}
\end{table}

\subsection{Full fine-tuning}
Both full-FT optimizers update every parameter in the model.

\paragraph{AdamW.} We use the HuggingFace defaults ($\beta_1=0.9$, $\beta_2=0.999$, $\varepsilon=10^{-8}$) with a warmup ratio of $0.02$. A learning rate of $5\mathrm{e}{-5}$ was optimal at every cell (Tab.~\ref{tab:lr}): we scanned $\{2,5,10\}\times 10^{-5}$ across model sizes and $5\mathrm{e}{-5}$ won in each case.

\paragraph{SGD with momentum.} We use momentum $0.9$ and a warmup ratio of $0.20$. The optimal learning rate is strongly cell-dependent (Tab.~\ref{tab:lr}) and was tuned per cell over $\{1,3,5,8\}\times 10^{-3}\cup\{1,1.5,2,3\}\times 10^{-2}$. For 3B SGD we further raise the per-device batch size from $1$ to $4$ (correspondingly dropping $\mathrm{ga}$ from $32$ to $8$); this reduces per-step gradient variance and eliminates the seed-level collapses we observed at $\mathrm{lr}\ge 5\mathrm{e}{-3}$.

\begin{table}[ht]
\centering\small
\caption{Per-cell best learning rates and warmup ratios. Full FT AdamW
uses $5\mathrm{e}{-5}$ everywhere. SGD requires a much larger and
size-dependent LR. LoRA values are reproduced for reference; LoRA at
$r\!=\!128$ uses a halved LR.}
\label{tab:lr}
\begin{tabular}{lcccc}
\toprule
 & Full FT & Full FT & LoRA & LoRA \\
Cell & AdamW & SGD & $r\!\in\!\{1,8,32,64\}$ & $r\!=\!128$ \\
\midrule
0.5B BoolQ & $5\mathrm{e}{-5}$ & $1\mathrm{e}{-2}$ & $2\mathrm{e}{-4}$ & $1\mathrm{e}{-4}$ \\
0.5B CSQA  & $5\mathrm{e}{-5}$ & $8\mathrm{e}{-3}$ & $2\mathrm{e}{-4}$ & $1\mathrm{e}{-4}$ \\
1.5B BoolQ & $5\mathrm{e}{-5}$ & $1.5\mathrm{e}{-2}$ & $2\mathrm{e}{-4}$ & $1\mathrm{e}{-4}$ \\
1.5B CSQA  & $5\mathrm{e}{-5}$ & $1.5\mathrm{e}{-2}$ & $2\mathrm{e}{-4}$ & $1\mathrm{e}{-4}$ \\
3B   BoolQ & $5\mathrm{e}{-5}$ & $2\mathrm{e}{-3}$  & $2\mathrm{e}{-4}$ & $1\mathrm{e}{-4}$ \\
3B   CSQA  & $5\mathrm{e}{-5}$ & $2\mathrm{e}{-3}$  & $2\mathrm{e}{-4}$ & $1\mathrm{e}{-4}$ \\
\midrule
warmup ratio & $0.02$ & $0.20$ & $0.02$ & $0.02$ \\
\bottomrule
\end{tabular}
\end{table}

\subsection{LoRA configuration}
All LoRA runs use AdamW with the same defaults as above and a warmup ratio of $0.02$. We sweep over ranks $r\in\{1,8,32,64,128\}$ and fix the scaling to $\alpha=2r$, so the effective scale $\alpha/r=2$ is constant and the per-step adapter contribution $\tfrac{\alpha}{r}BA$ has the same magnitude across ranks. Adapter dropout is disabled.

Adapters are placed on \texttt{all-linear} targets---every linear projection in every transformer block, i.e.\ the attention projections $\{W_q,W_k,W_v,W_o\}$ together with the MLP projections $\{W_{\text{gate}},W_{\text{up}},W_{\text{down}}\}$. The base weights are frozen and only the low-rank factors $A,B$ are trained. We follow the HuggingFace PEFT defaults: $A$ is drawn from $\mathcal{N}(0,1/r)$ and $B$ is initialized to zero.

For $r\in\{1,8,32,64\}$ we use a learning rate of $2\mathrm{e}{-4}$; at $r=128$ we lower it to $1\mathrm{e}{-4}$. For each (cell, rank) configuration we report $3$ training seeds, used for both the sample-size sweeps and the final accuracy means.

\subsection{Model averaging (SGD only)}\label{app:Model averaging}
For SGD we report a uniform average of fully-trained checkpoints,
\begin{equation*}
    W_{\text{avg}} = \frac{1}{N}\sum_{i=1}^{N} W_{\text{seed}_i},
\end{equation*}
computed as an arithmetic mean in fp32 and recast to bf16. We train $N=4$ independent seeds per cell at the cell's best learning rate. We do \emph{not} average AdamW seeds, instead, we report best-of-$N$ for AdamW comparisons.

\subsection{Data augmentation}\label{app:Data augmentation}
We paraphrase BoolQ using GPT-5.4-mini. For each original (passage, question) pair, the model is asked to rewrite both the passage and the question while preserving the gold yes/no label as well as all named entities and numerical values. Every candidate paraphrase is then \emph{validated} by re-asking the same model the rewritten question against the rewritten passage at temperature $0$; pairs whose re-answer disagrees with the original gold label are discarded. This yields $8{,}510$ kept paraphrases out of $9{,}427$ originals---a retention rate of $90.3\%$. The augmented BoolQ training set is the concatenation of the originals and the validated paraphrases, for a total of $17{,}937$ examples.

\subsection{Runs on the augmented dataset}\label{app:augmented-runs}
The augmented BoolQ runs reuse the training configuration above without modification; the only difference is the training-set size.

\subsection{Label-noise injection}
We consider two flipping protocols, both using a fixed noise seed so that the corrupted subsets are nested across noise levels---the examples flipped at $p=5\%$ are a strict subset of those flipped at $p=10\%$, and so on.

\begin{itemize}\setlength{\itemsep}{0pt}
\item \textbf{BoolQ:} flip the binary label on $p\%$ of the training examples.
\item \textbf{CSQA:} for each example selected for flipping, replace the correct answer letter with one of the four incorrect letters drawn uniformly at random.
\end{itemize}

In both cases the resulting label distribution remains close to balanced. We sweep $p\in\{5,10,20,30\}\%$, and otherwise keep the per-cell hyperparameters of Tab.~\ref{tab:lr} unchanged---only the data file is swapped. We train $3$ seeds per (cell, method, noise level) and report the mean together with the min--max band.

\subsection{Compute}
All training jobs run on a single NVIDIA A100 (40\,GB or 80\,GB). The 3B full-FT runs require the 80\,GB variant; everything else fits comfortably on 40\,GB. Per-seed wall-clock training time on the original training set is roughly $13$\,min at 0.5B, $36$\,min at 1.5B, and $80$\,min at 3B; on augmented BoolQ ($n=17{,}937$) these times roughly double.

\subsection{High-performing reference model}\label{app:reference model}
To obtain the high-performing reference model for each dataset, we first generated additional samples 
that match the same format as described above in Sec.~\ref{app:Data augmentation}. Then, we applied two training techniques on the extended training set: (1) FFT with AdamW and (2) training a few times using SGD with different seeds, and averaging the resulting weights to get a better model, as explained in Sec.~\ref{app:Model averaging}. The reference model was chosen as the one with the better test performance.

\clearpage
\subsection{Additional LLM experiments}\label{subsection:app_additional_llm_experiments}
In this section, we provide the plots for the LLM experiment form Sec.~\ref{sec:experiments}, including all model sizes ($0.5$B, $1.5$B, and $3$B). The behavior for all models is the same as described in the main text.

\begin{figure}[ht]
    \centering
    \begin{subfigure}[t]{0.5\linewidth}
        \centering
        \includegraphics[width=\linewidth]{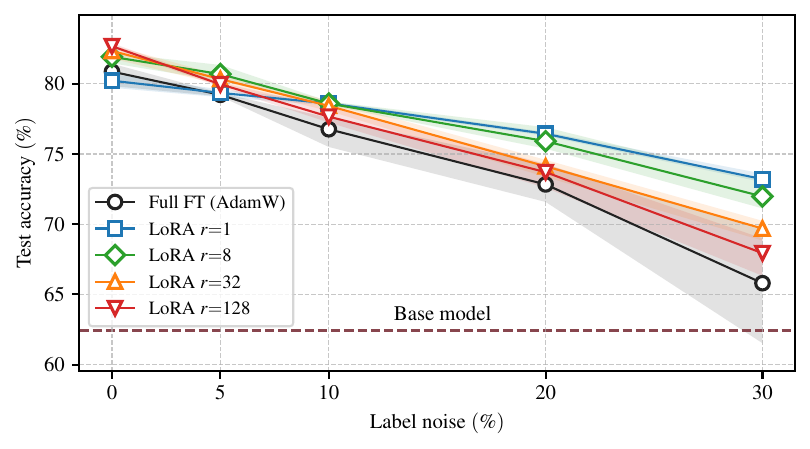}
        \caption{Qwen$2.5$-$0.5$B on BoolQ}
        \label{fig:Boolq_noise_sweep_0.5B}
    \end{subfigure}%
    \begin{subfigure}[t]{0.5\linewidth}
        \centering
        \includegraphics[width=\linewidth]{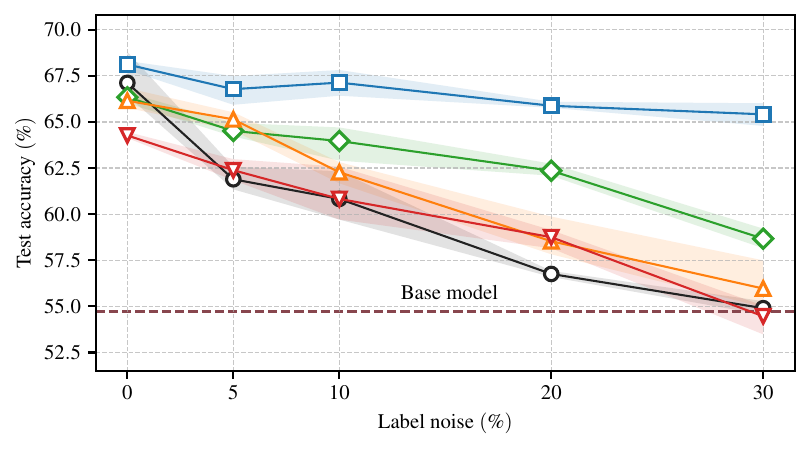}
        \caption{Qwen$2.5$-$0.5$B on CommonsenseQA}
        \label{fig:commonsense_noise_sweep_0.5B}
    \end{subfigure}
    \caption{{\bf Effect of label noise in LLMs fine-tuning.}
    We fine-tuned Qwen$2.5$ models using LoRA with various ranks across different levels of label noise. For each configuration, we trained $3$ times using random seeds. Panels~\protect\subref{fig:Boolq_noise_sweep_0.5B} and~\protect\subref{fig:commonsense_noise_sweep_0.5B} show the mean and the range of the results for the $0.5$B model fine-tuned on BoolQ and CommonsenseQA, respectively. Here, in the strong-noise regime, LoRA outperforms FFT as predicted by our theory.}
    \label{fig:LLM_noise_sweeps_0.5B}
\end{figure}

\begin{figure}[ht]%
    \begin{subfigure}[t]{0.5\linewidth}%
    \!\!\!
        \includegraphics[width=\linewidth]{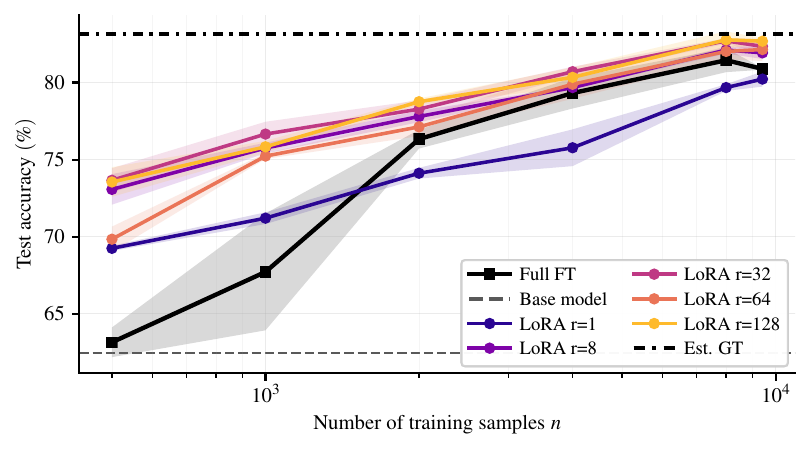}\\%
        \includegraphics[width=\linewidth]{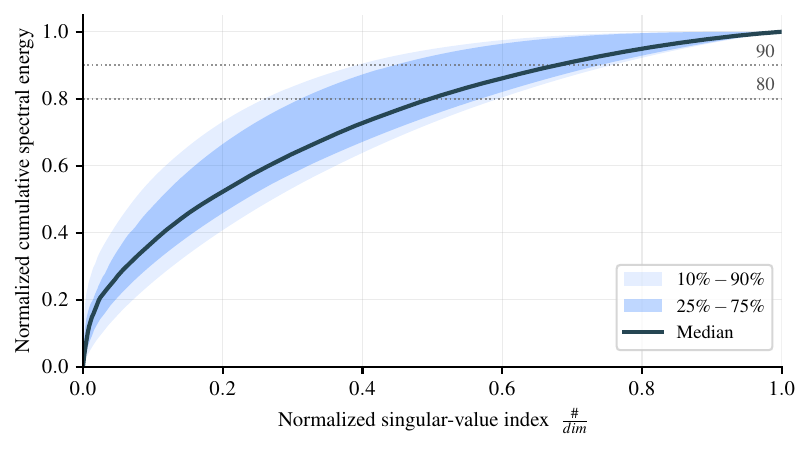}
        \caption{Qwen$2.5$-$0.5$B on BoolQ}
        \label{fig:Boolq_samples_sweep_0.5B}
    \end{subfigure}%
    \begin{subfigure}[t]{0.5\linewidth}%
        \includegraphics[width=\linewidth]{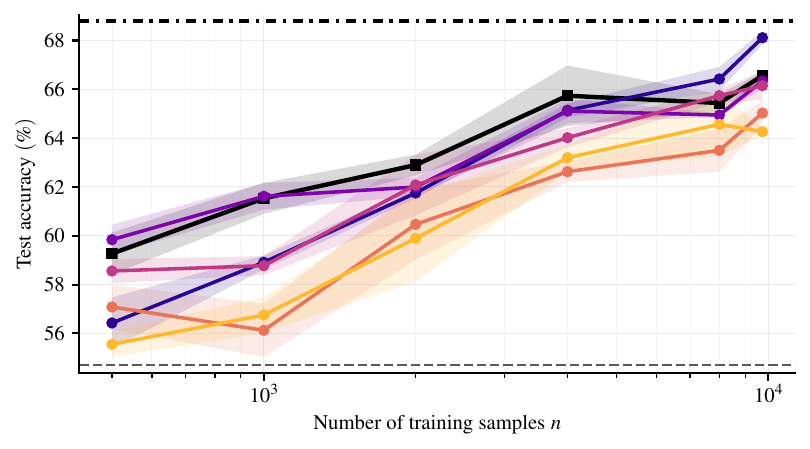}\\%
        \includegraphics[width=\linewidth]{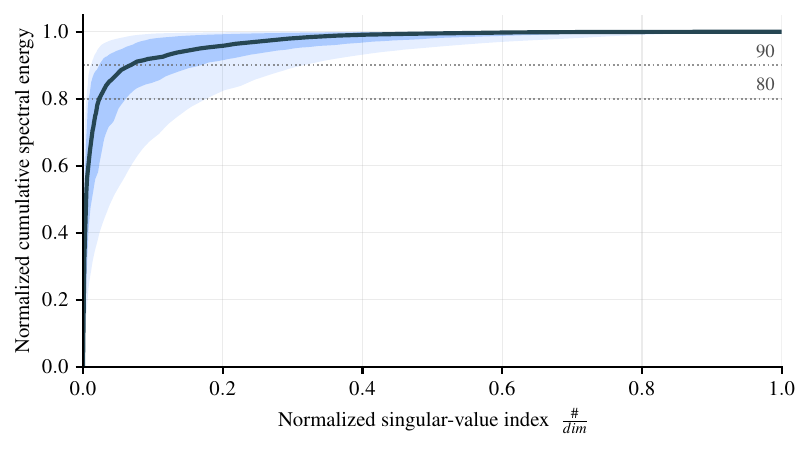}
        \caption{Qwen$2.5$-$0.5$B on CommonsenseQA}
        \label{fig:commonsense_samples_sweep_0.5B}
    \end{subfigure}
    \caption{{\bf Effect of sample size in LLMs fine-tuning.}
    We fine-tuned Qwen$2.5$ models using LoRA with various ranks across different sample sizes. For each configuration, we trained $3$ times using random seeds. The top of Panels~\protect\subref{fig:Boolq_samples_sweep_0.5B} and~\protect\subref{fig:commonsense_samples_sweep_0.5B} shows the mean and the range of the results for the $0.5$B model fine-tuned on BoolQ and CommonsenseQA, respectively. For each task, we estimated $\bDelta^\star$ and computed its singular values per layer.
    The bottom part shows the statistics of the spectral cumulative energy of $\bDelta^\star$ across layers. When the effective rank of the adaptation is low (CommonsenseQA), LoRA outperforms FFT, as predicted by our theory.}
    \label{fig:LLM_samples_sweeps_0.5B}
\end{figure}

\clearpage
\begin{figure}[t]
    \centering
    \begin{subfigure}[t]{0.5\linewidth}
        \centering
        \includegraphics[width=\linewidth]{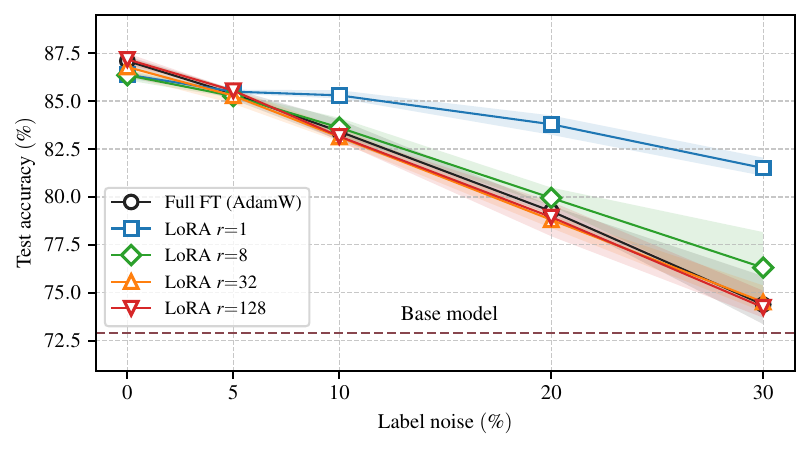}
        \caption{Qwen$2.5$-$1.5$B on BoolQ}
        \label{fig:Boolq_noise_sweep_1.5B}
    \end{subfigure}%
    \begin{subfigure}[t]{0.5\linewidth}
        \centering
        \includegraphics[width=\linewidth]{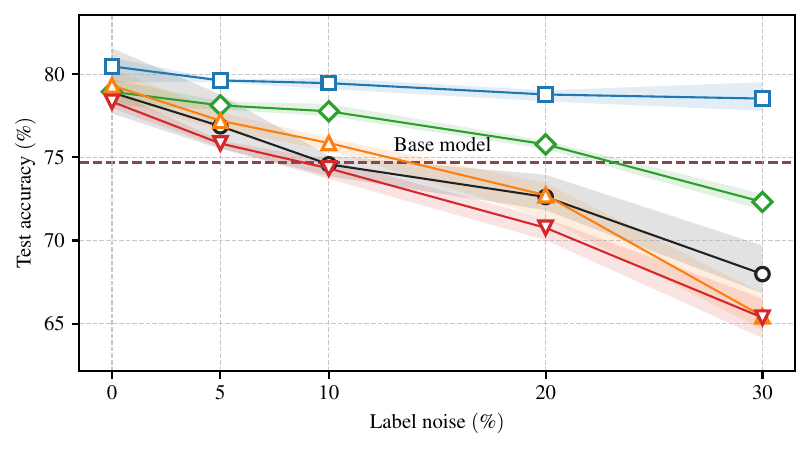}
        \caption{Qwen$2.5$-$1.5$B on CommonsenseQA}
        \label{fig:commonsense_noise_sweep_1.5B}
    \end{subfigure}
    \caption{{\bf Effect of label noise in LLMs fine-tuning.}
    We fine-tuned Qwen$2.5$ models using LoRA with various ranks across different levels of label noise. For each configuration, we trained $3$ times using random seeds. Panels~\protect\subref{fig:Boolq_noise_sweep_1.5B} and~\protect\subref{fig:commonsense_noise_sweep_1.5B} show the mean and the range of the results for the $1.5$B model fine-tuned on BoolQ and CommonsenseQA, respectively. Here, in the strong-noise regime, LoRA outperforms FFT as predicted by our theory.}
    \label{fig:LLM_noise_sweeps_1.5B}
\end{figure}

\begin{figure}[t]%
    \begin{subfigure}[t]{0.5\linewidth}%
    \!\!\!
        \includegraphics[width=\linewidth]{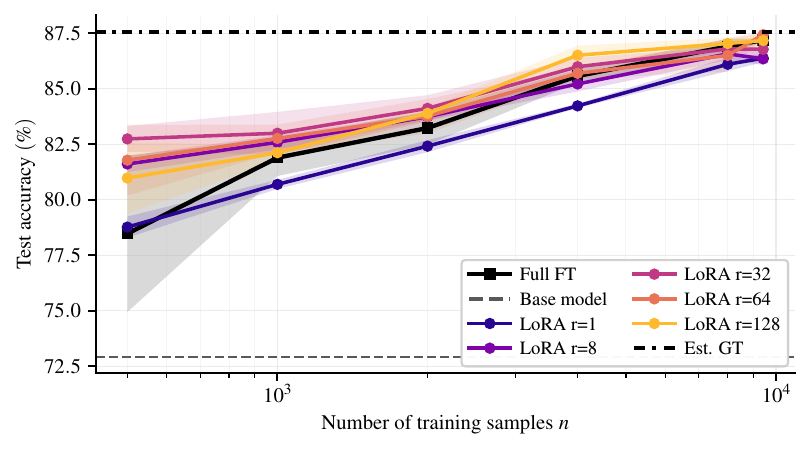}\\%
        \includegraphics[width=\linewidth]{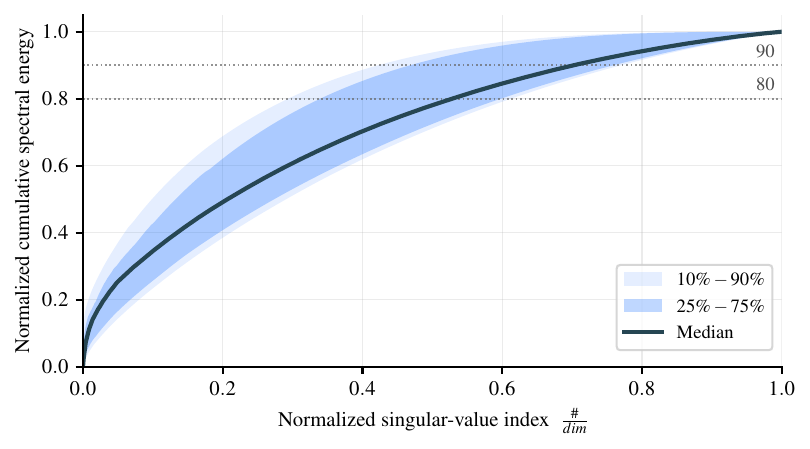}
        \caption{Qwen$2.5$-$1.5$B on BoolQ}
        \label{fig:Boolq_samples_sweep_1.5B}
    \end{subfigure}%
    \begin{subfigure}[t]{0.5\linewidth}%
        \includegraphics[width=\linewidth]{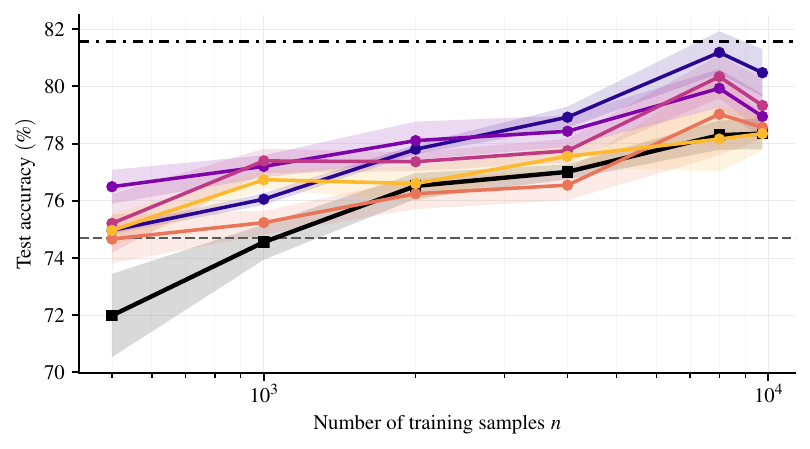}\\%
        \includegraphics[width=\linewidth]{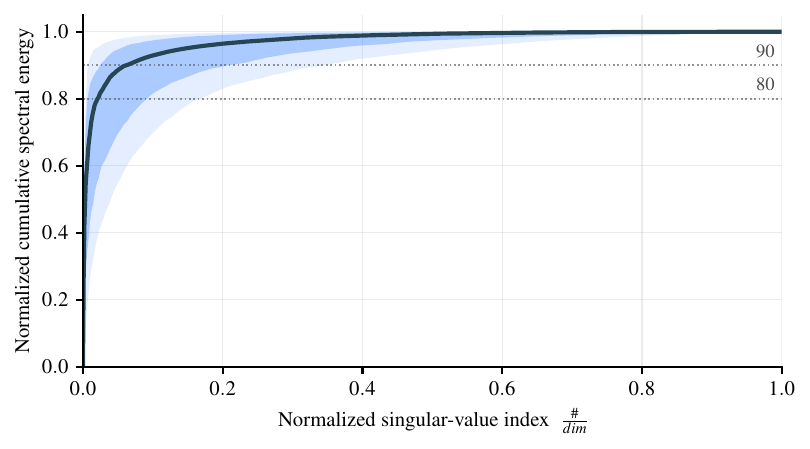}
        \caption{Qwen$2.5$-$1.5$B on CommonsenseQA}
        \label{fig:commonsense_samples_sweep_1.5B}
    \end{subfigure}
    \caption{{\bf Effect of sample size in LLMs fine-tuning.}
    We fine-tuned Qwen$2.5$ models using LoRA with various ranks across different sample sizes. For each configuration, we trained $3$ times using random seeds. The top of Panels~\protect\subref{fig:Boolq_samples_sweep_1.5B} and~\protect\subref{fig:commonsense_samples_sweep_1.5B} shows the mean and the range of the results for the $1.5$B model fine-tuned on BoolQ and CommonsenseQA, respectively. For each task, we estimated $\bDelta^\star$ and computed its singular values per layer.
    The bottom part shows the statistics of the spectral cumulative energy of $\bDelta^\star$ across layers. When the effective rank of the adaptation is low (CommonsenseQA), LoRA outperforms FFT, as predicted by our theory.}
    \label{fig:LLM_samples_sweeps_1.5B}
\end{figure}

\clearpage
\begin{figure}[t]
    \centering
    \begin{subfigure}[t]{0.5\linewidth}
        \centering
        \includegraphics[width=\linewidth]{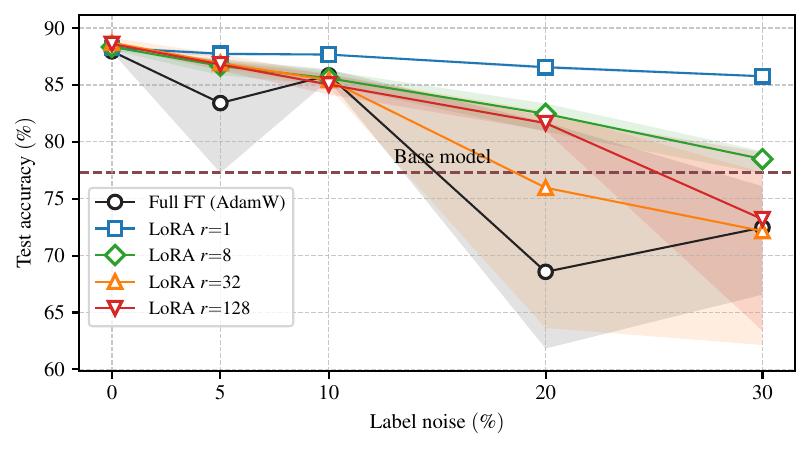}
        \caption{Qwen$2.5$-$3$B on BoolQ}
        \label{fig:Boolq_noise_sweep_3B}
    \end{subfigure}%
    \begin{subfigure}[t]{0.5\linewidth}
        \centering
        \includegraphics[width=\linewidth]{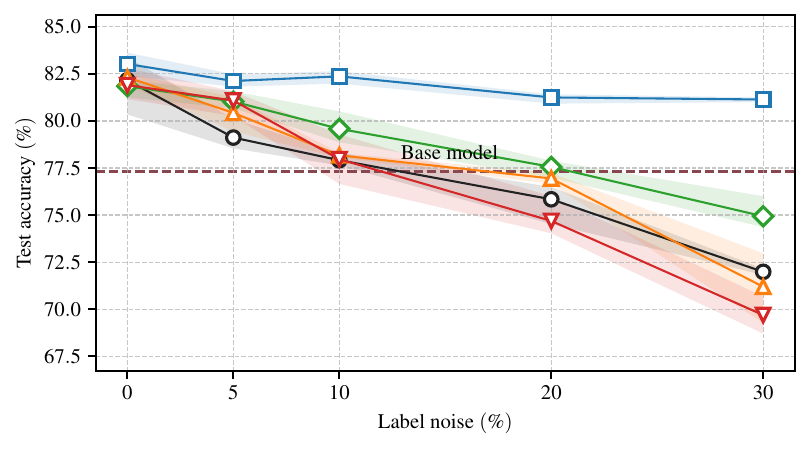}
        \caption{Qwen$2.5$-$3$B on CommonsenseQA}
        \label{fig:commonsense_noise_sweep_3B}
    \end{subfigure}
    \caption{{\bf Effect of label noise in LLMs fine-tuning.}
    We fine-tuned Qwen$2.5$ models using LoRA with various ranks across different levels of label noise. For each configuration, we trained $3$ times using random seeds. Panels~\protect\subref{fig:Boolq_noise_sweep_3B} and~\protect\subref{fig:commonsense_noise_sweep_3B} show the mean and the range of the results for the $3$B model fine-tuned on BoolQ and CommonsenseQA, respectively. Here, in the strong-noise regime, LoRA outperforms FFT as predicted by our theory.}
    \label{fig:LLM_noise_sweeps_3B}
\end{figure}

\begin{figure}[t]%
    \begin{subfigure}[t]{0.5\linewidth}%
    \!\!\!
        \includegraphics[width=\linewidth]{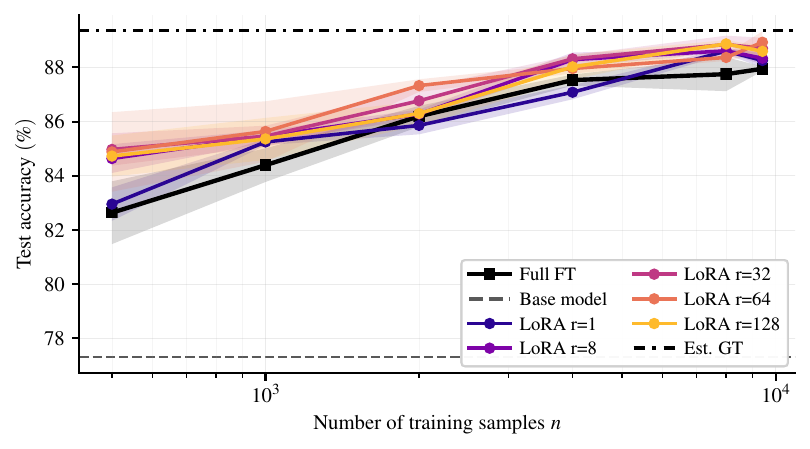}\\%
        \includegraphics[width=\linewidth]{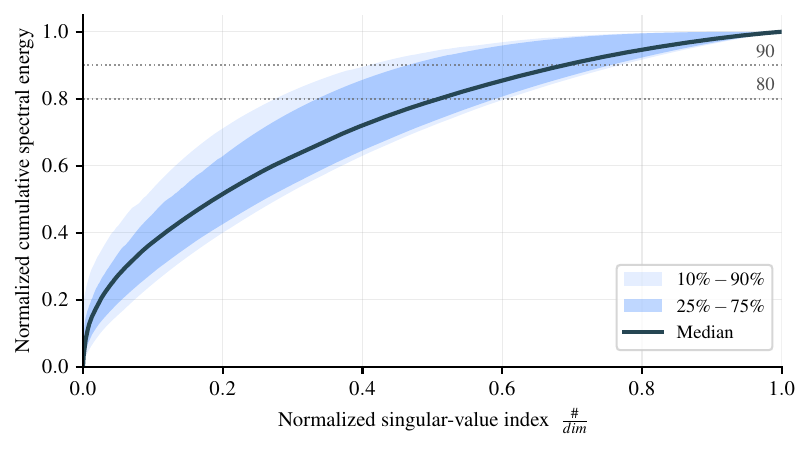}
        \caption{Qwen$2.5$-$3$B on BoolQ}
        \label{fig:Boolq_samples_sweep_3B}
    \end{subfigure}%
    \begin{subfigure}[t]{0.5\linewidth}%
        \includegraphics[width=\linewidth]{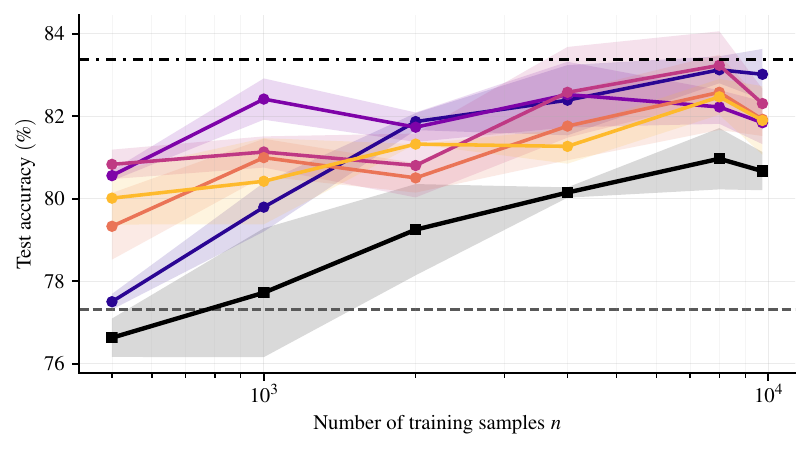}\\%
        \includegraphics[width=\linewidth]{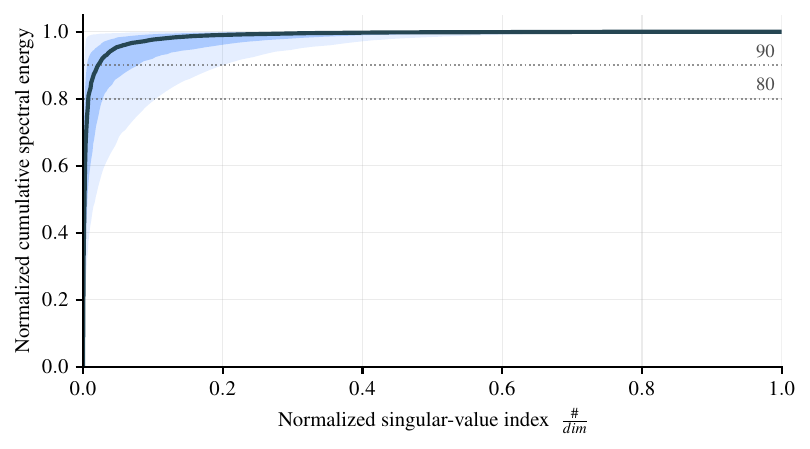}
        \caption{Qwen$2.5$-$3$B on CommonsenseQA}
        \label{fig:commonsense_samples_sweep_3B}
    \end{subfigure}
    \caption{{\bf Effect of sample size in LLMs fine-tuning.}
    We fine-tuned Qwen$2.5$ models using LoRA with various ranks across different sample sizes. For each configuration, we trained $3$ times using random seeds. The top of Panels~\protect\subref{fig:Boolq_samples_sweep_3B} and~\protect\subref{fig:commonsense_samples_sweep_3B} shows the mean and the range of the results for the $3$B model fine-tuned on BoolQ and CommonsenseQA, respectively. For each task, we estimated $\bDelta^\star$ and computed its singular values per layer.
    The bottom part shows the statistics of the spectral cumulative energy of $\bDelta^\star$ across layers. When the effective rank of the adaptation is low (CommonsenseQA), LoRA outperforms FFT, as predicted by our theory.}
    \label{fig:LLM_samples_sweeps_3B}
\end{figure}

\clearpage

\subsection{Additional linear regression experiments}\label{subsection:app_additional_experiments}
In this section, we present additional linear regression experiments to further illustrate the behavior of LoRA and FFT. We first study the effect of the rank of $\bDelta^\star$ on fine-tuning performance. Next, we repeat the experiment of Figure~\ref{fig:lora_sweeps} in the low-noise and overdetermined regime (corresponding to Figure \ref{fig:samples_sweep} and Figure \ref{fig:noise_sweep}) and observe a similar qualitative pattern. Finally, we show that even when $\bDelta^\star$ is full-rank, the decay of its singular values plays a key role in determining whether LoRA or FFT performs better, highlighting the importance of the spectral tail of $\bDelta^\star$, which also appears explicitly in our theoretical bounds.

Figure~\ref{fig:lora_sweeps_appendix} illustrates the behavior of FFT and LoRA in complementary experimental settings: the left panel shows the effect of varying dataset size in a low noise regime, while the right panel examines the effect of noise magnitude in the overdetermined regime.

\paragraph{Effect of dataset size in low-noise regime. } The left panel (Figure \ref{fig:samples_sweep_low_noise}) displays excess risk as a function of the input dimension-to-sample ratio $d_x/n$ in a low noise setting, with fixed dimensions, noise level, and $\rank(\bDelta^\star)=4$. The gray dashed vertical line at $d_x/n=1$ marks the interpolation threshold, separating the overdetermined regime (left, where $n > d_x$) from the underdetermined regime (right, where $n < d_x$).

In the overdetermined regime, LoRA with rank matching the true rank achieves the best performance: LoRA with $r=4$ (matching $r^{\star}=4$) exhibits the lowest excess risk, followed by LoRA with $r=8$, and then FFT. This demonstrates that in low noise conditions, proper rank selection allows LoRA to outperform FFT even when sufficient samples are available. The superior performance of rank-matched LoRA arises because it correctly captures the low-rank structure of $\bDelta^{\star}$ without introducing unnecessary degrees of freedom. In contrast, FFT updates all directions, including irrelevant ones, leading to higher variance even in the overdetermined regime with low noise. For smaller ranks $r=1$ and $r=2$, which are insufficient to capture the true rank, the excess risk remains nearly constant and substantially higher throughout the overdetermined regime. This flat behavior reflects the dominance of the bias term induced by the rank constraint, as the spectral tail of $\bDelta^{\star}$ is significant and cannot be captured by such low ranks.

As the dimension-to-sample ratio increases toward the interpolation threshold, all methods experience increasing excess risk. At the threshold $d_x/n=1$, FFT exhibits a sharp increase in risk as it transitions into the underdetermined regime. LoRA methods show more graceful degradation across the threshold, with the transition behavior depending on the rank. In the deeply underdetermined regime (large $d_x/n$), all methods eventually converge to similar excess risk levels determined by the bias of the initial model $\initmodel$.

\begin{figure}[t]
    \centering
    \begin{subfigure}[t]{0.49\linewidth}
        \centering
        \includegraphics[width=\linewidth]{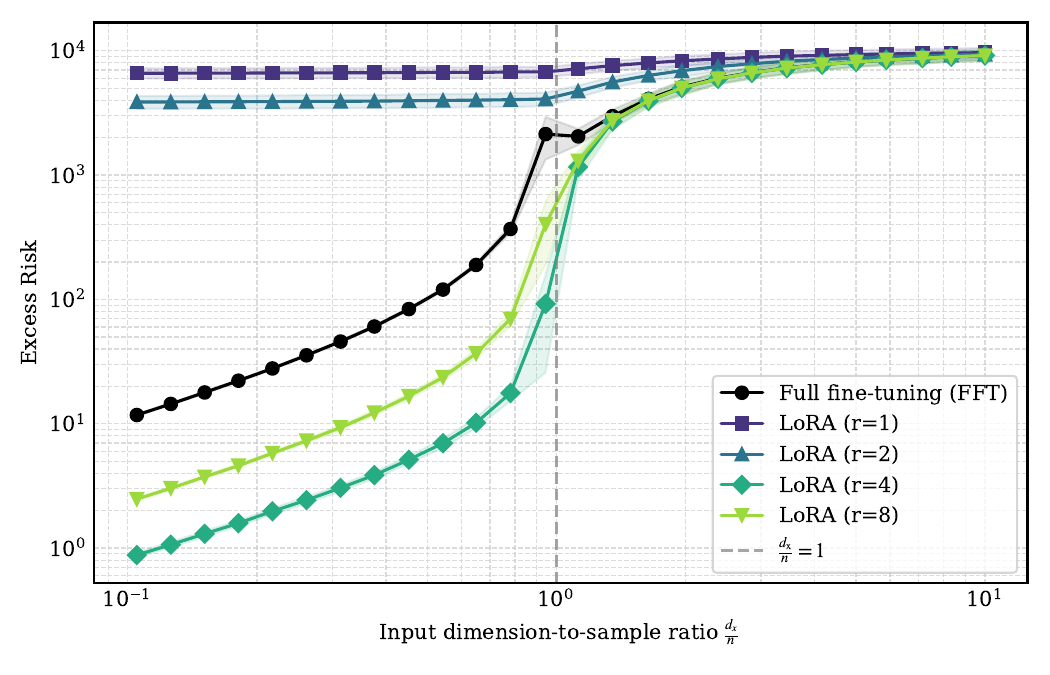}
        \caption{Excess risk vs. the dimension-to-sample ratio}
        \label{fig:samples_sweep_low_noise}
    \end{subfigure}
    \hfill
    \begin{subfigure}[t]{0.49\linewidth}
        \centering
        \includegraphics[width=\linewidth]{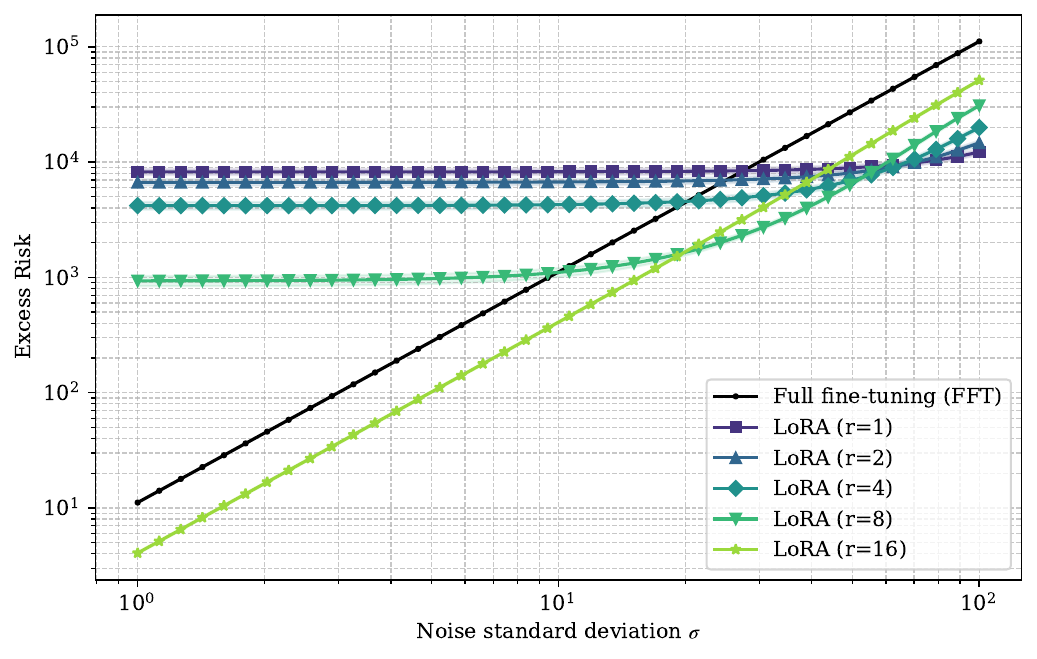}
        \caption{Excess risk vs. noise level}
        \label{fig:noise_sweep_overdetermined}
    \end{subfigure}
    \caption{{\bf Linear regression experiments.}
    Panel~\protect\subref{fig:samples_sweep_low_noise} presents the excess risk of FFT and LoRA under varying sample size $n$ (decreasing from left to right), fixed dimensions $d_x = d_y = 100$, noise magnitude $\sigma=1$, and true task-shift rank $\mathrm{rank}(\bDelta^\star) = 4$.
    Panel~\protect\subref{fig:noise_sweep_overdetermined} plots the excess risk as a function of noise level $\sigma\in[1,100]$ with $d_x=d_y=100$, $n=1000$, and $\rank(\bDelta^\star)=10$.
    Results are averaged over $100$ random seeds; shaded regions indicate $\pm1$ standard deviation.}
    \label{fig:lora_sweeps_appendix}
\end{figure}

\paragraph{Effect of noise magnitude in overdetermined regime. }Figure~\ref{fig:noise_sweep_overdetermined} examines the effect of noise magnitude in the overdetermined regime with fixed dimensions and $\rank(\bDelta^\star)$. In this setting, when noise is low, FFT and high-rank LoRA (such as $r=16$) achieve similar performance and outperform low-rank alternatives. However, as noise increases, FFT becomes increasingly sensitive, and its excess risk grows rapidly. In contrast, low-rank LoRA methods ($r=1$, $r=2$, and $r=4$) demonstrate remarkable stability, with their excess risk remaining nearly constant across a wide range of noise levels. This robustness arises because the low-rank constraint prevents the estimator from fitting noise in irrelevant directions. At sufficiently high noise levels, low-rank LoRA outperforms FFT, highlighting the regularization benefit of the rank constraint even in the overdetermined regime.

\paragraph{Effect of the rank of true map on excess risk.}
Figure~\ref{fig:delta_star_rank_sweep} shows how the excess risk depends on the rank of $\bDelta^\star$ in the underdetermined regime. When $\bDelta^\star$ is low-rank, LoRA with a suitable choice of rank $r$ consistently achieves lower excess risk. As the rank of $\bDelta^\star$ increases and approaches the full-rank setting, FFT begins to outperform LoRA with small $r$, since low-rank updates do not provide enough flexibility to capture all relevant directions. Nevertheless, even when $\bDelta^\star$ is full-rank, choosing a sufficiently large $r$ allows LoRA to achieve performance close to that of FFT. The exact behavior also depends on the noise level. Overall, these results indicate that the choice of fine-tuning method and LoRA rank should account for the intrinsic rank of the task shift, the data regime, and the noise level.

\begin{figure}[t]
\centerline{\includegraphics[width=0.5\linewidth]{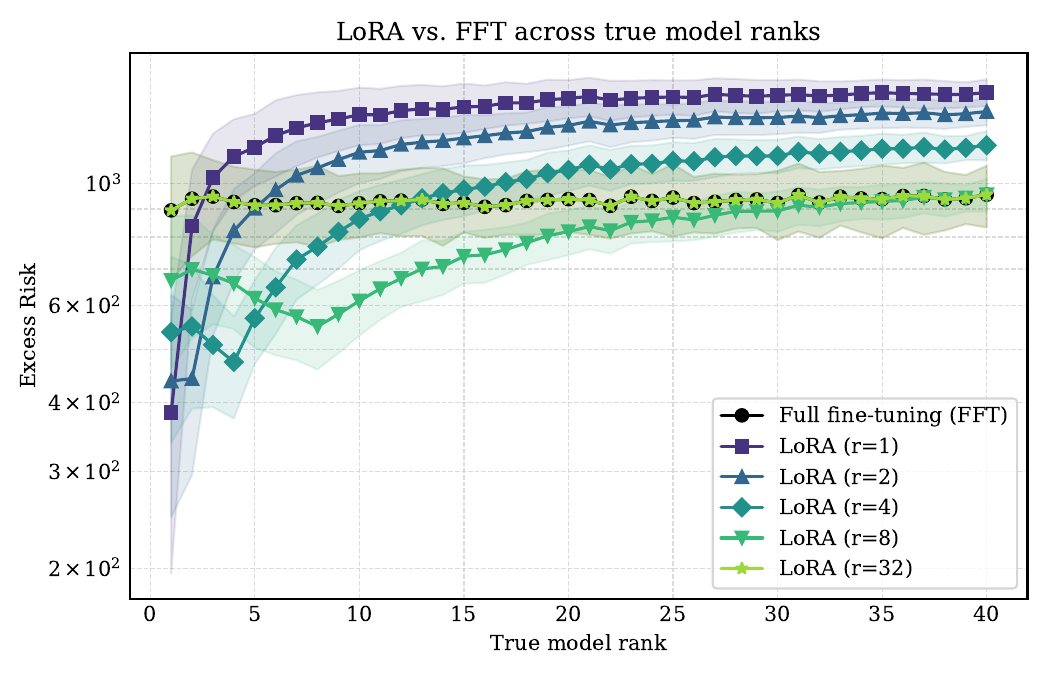}}
\caption{This figure illustrates the role of the central quantity $\rank(\bDelta^\star)$ in our analysis in underdetermined regime. We vary the rank of $\bDelta^\star$ from $1$ to $40$ (full rank) with a fixed number of training data $n=30$, fixed dimension $d_x=d_y=40$, and $\sigma=2$. Results are reported across $100$ random seeds, and shaded areas show the standard deviation.}
\label{fig:delta_star_rank_sweep}
\end{figure}

\paragraph{Full-rank $\bDelta^\star$ with varying spectral decay.}

We study how the singular value structure of $\bDelta^\star$ affects the relative performance of FFT and LoRA in the overdetermined regime.
We fix $d_x = d_y = 40$ and construct $\bDelta^\star = \UU \CovMat \VV^\top$, where $\UU$ and $\VV$ are orthonormal and the singular values follow
\begin{equation}
    \sigma_i = 5 \exp(-\lambda i),
\end{equation}
with decay rate $\lambda \in [0, 2.5]$.
When $\lambda = 0$, the spectrum is flat, and all directions are equally important; as $\lambda$ increases, the spectrum becomes increasingly concentrated in the top singular directions.

Figure~\ref{fig:decay_rate_sweep} reports the excess risk as a function of $\lambda$ for $n=200$ samples and noise level $\sigma_{\vecnoise}=0.5$.
When the spectrum is flat ($\lambda=0$), FFT significantly outperforms LoRA, since the task effectively requires modifying many directions and low-rank updates are too restrictive.
As $\lambda$ increases, the effective rank of $\bDelta^\star$ decreases, and LoRA with moderate rank begins to outperform FFT.
In this regime, constraining the update to a low-dimensional subspace helps avoid fitting noise in weak directions while still capturing the dominant structure.

Overall, this experiment shows that the benefit of LoRA is governed not by the formal rank of $\bDelta^\star$, but by its spectral decay.
When the update is spread across many directions, full fine-tuning is preferable; when it is concentrated in a few directions, low-rank adaptation yields better generalization. This is indeed aligned with our excess risk bounds for LoRA, which has an error term depending on the tail of $\bDelta^\star$. When this tail decays quickly, this error becomes negligible, while a flat tail could cause a large error.

\begin{figure}[t]
    \centering
    \includegraphics[width=0.5\linewidth]{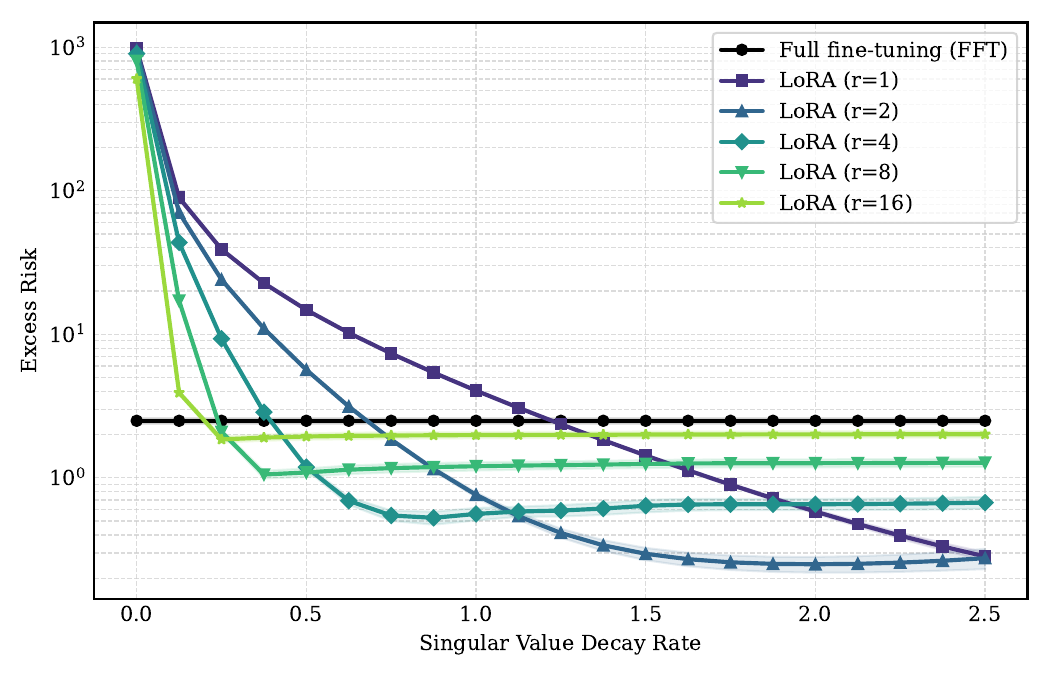}
    \caption{%
        \textbf{Effect of spectral decay on FFT vs.\ LoRA.}
        Excess risk as a function of the singular value decay rate $\lambda$ for a \textbf{full-rank} $\bDelta^\star$ with $\sigma_i = 5\exp(-\lambda i)$.
        A flat spectrum ($\lambda=0$) favors FFT, while increasing spectral concentration leads to superior performance of LoRA with moderate rank.
        Settings: $d_x=d_y=40$, $n=200$, $\sigma_{\vecnoise}=0.5$, averaged over 100 runs.
    }
    \label{fig:decay_rate_sweep}
\end{figure}

\end{document}